\documentclass{article}

\usepackage{ml4cps}

\usepackage[utf8]{inputenc} 
\usepackage[T1]{fontenc}    
\usepackage{lmodern}        
\usepackage{amsmath}        
\usepackage{amsfonts}       
\usepackage{nicefrac}       
\usepackage{microtype}      
\usepackage{hyperref}       
\usepackage{url}            
\usepackage{booktabs}       
\usepackage{graphicx}       
\usepackage{enumitem}       
\usepackage{natbib}         
\usepackage{doi}            
\usepackage{lipsum}         
\usepackage{cleveref}       
\usepackage{multirow}
\usepackage{tikz}
\usepackage{xcolor}
\usepackage{booktabs} 
\usepackage{fancyhdr}
\usepackage{listings}
\usepackage{titlesec}
\usepackage{tabularx}
\usepackage{multirow}
\usepackage{graphicx}
\usepackage{mathrsfs}
\usepackage{amsbsy}
\usepackage{physics}
\usepackage{threeparttable}
\usepackage{subcaption} 

\usepackage{amsthm, amsmath, amssymb}

\newtheorem{theorem}{Theorem}[section]

\titleformat{\section}[block]{\normalfont\Large\bfseries}{\thesection}{1em}{}
\titlespacing*{\section}{0pt}{*3}{*1}

\definecolor{lightgray}{rgb}{0.95, 0.95, 0.95} 
\definecolor{gray}{rgb}{0.4, 0.4, 0.4}         
\definecolor{blue}{rgb}{0, 0, 0.8}             
\definecolor{purple}{rgb}{0.58, 0, 0.82}       

\lstdefinestyle{mystyle}{
    backgroundcolor=\color{lightgray},   
    commentstyle=\color{gray},           
    keywordstyle=\color{blue},           
    numberstyle=\tiny\color{gray},       
    stringstyle=\color{purple},          
    basicstyle=\ttfamily\footnotesize,     
    breakatwhitespace=false,             
    breaklines=true,                     
    captionpos=b,                        
    keepspaces=true,                     
    numbers=left,                        
    numbersep=5pt,                       
    showspaces=false,                    
    showstringspaces=false,              
    showtabs=false,                      
    tabsize=2                            
}

\title{Hybrid Quantum-Classical PINNs for Scientific Computing: A Multi-GPU Open-Source Framework}

\date{}

\newif\ifblinded
\blindedfalse   

\newif\ifuniqueAffiliation
\uniqueAffiliationtrue

\ifblinded
    \author{}
\else
    \ifuniqueAffiliation
    \author{
        Ziv Chen\footnotemark[1]\thanks{The Andrew and Erna Viterbi Faculty of Electrical \& Computer Engineering, Technion Israel Institute of Technology, Haifa, Israel.}\thanks{Helen Diller Quantum Center, Technion Israel Institute of Technology, Haifa, Israel.} \\
        \texttt{ziv.chen@campus.technion.ac.il} \\
        \And
        Hemanth Chandravamsi\footnotemark[1]\footnotemark[2]\thanks{Corresponding author.} \\
        \texttt{hemanth@campus.technion.ac.il} \\
        \And
        Shimon Pisnoy\thanks{Equal contribution.}\thanks{Faculty of Mechanical Engineering, Technion Israel Institute of Technology, Haifa, Israel.} \\
        \texttt{shimonpi@campus.technion.ac.il} \\
        \And
        Aaron Goldgewert\thanks{Faculty of Electrical and Computer Engineering, Cornell University, Ithaca, NY, USA.} \\
        \texttt{ajg322@cornell.edu} \\
        \And
        Gal Shaviner\footnotemark[2] \\
        \texttt{gal.shaviner@campus.technion.ac.il} \\
        \And
        Boris Shrager\footnotemark[2] \\
        \texttt{boris.sh@campus.technion.ac.il} \\
        \And
        Steven H. Frankel\footnotemark[2]\footnotemark[5] \\
        \texttt{frankel@me.technion.ac.il}
    }
    \else
    \usepackage{authblk}
    
    \newbox{\orcid}\sbox{\orcid}{\includegraphics[scale=0.06]{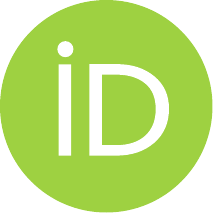}} 
    \author[1]{%
        \href{https://orcid.org/0000-0000-0000-0000}{\usebox{\orcid}\hspace{1mm}First author}%
    }
    \author[1,2]{%
        \href{https://orcid.org/0000-0000-0000-0000}{\usebox{\orcid}\hspace{1mm}Second author}%
    }
    \affil[1]{Affiliation, Address}
    \affil[2]{Affiliation, Address}
    \fi
\fi

\ifblinded
\hypersetup{
pdftitle={contribution title},
pdfsubject={contribution subject},
pdfauthor={Anonymous},
pdfkeywords={PINNs, Activation function, Optimization},
}
\else
\hypersetup{
pdftitle={contribution title},
pdfsubject={contribution subject},
pdfauthor={First author, Second author},
pdfkeywords={PINNs, Activation function, Optimization},
}
\fi

\begin{document}
\setlength{\parskip}{0pt}
\setlength{\parindent}{1em}

\maketitle
    
\begin{abstract}

We present QPINNACLE, an open-source computational framework for physics-informed neural networks (PINNs) that integrates modern training strategies, multi-GPU acceleration, and hybrid quantum-classical architectures within a unified modular workflow. The framework enables systematic evaluation of PINN performance across benchmark problems including 1D hyperbolic conservation laws, incompressible flows, and electromagnetic wave propagation. It supports a range of architectural and training enhancements, including Fourier feature embeddings, random weight factorization, strict boundary condition enforcement, adaptive loss balancing, curriculum training, and second-order optimization strategies, with extensibility to additional methods. We provide a comprehensive benchmark study quantifying the impact of these methods on convergence, accuracy, and computational cost, and analyze distributed data parallel scaling in terms of runtime and memory efficiency. In addition, we extend the framework to hybrid quantum-classical PINNs and derive a formal estimate for circuit-evaluation complexity under parameter-shift differentiation. Results highlight the sensitivity of PINNs to architectural and training choices, confirm their high computational cost relative to classical solvers, and identify regimes where hybrid quantum models offer improved parameter efficiency. QPINNACLE provides a foundation for benchmarking physics-informed learning methods and guiding future developments through quantitative assessment of their trade-offs.

\end{abstract}








\clearpage
\tableofcontents

\section{Introduction}

\subsection{Physics-informed neural networks}

Many computational physics workloads reduce to forward or inverse problems constrained by differential equations. Classical discretization methods offer reliability and mature error analysis, but they also impose methodological overhead in those workflows that require differentiable mappings, inclusion of sparse measurements alongside governing laws, or separate standalone simulations when geometries or problem parameters change. Physics-informed neural networks (PINNs) provide an alternative computational strategy: a neural network $u_{\theta}(\mathbf{x},t)$ approximates the solution field and is trained by penalizing violations of the governing equations and constraints at collocation points \cite{lagaris1998artificial,raissi2019pinns}.

In PINNs, the solution is approximated by a parametric function $u_{\theta}(\mathbf{x},t)$. Residuals of the governing equations and constraints are evaluated at collocation points using automatic differentiation, and training proceeds by minimizing a composite objective that penalizes violations of the PDE, boundary conditions, and optional observational data. This construction delivers several attractive properties. First, PINNs are mesh-free in the sense that the cost function is evaluated at a set of scattered collocation points, without relying on a predefined spatial discretization or connectivity typical of grid-based methods. Second, the learned surrogate is differentiable with respect to inputs and parameters, which supports inverse problems \cite{cai2021physics}, system identification \cite{wang2025physics}, and gradient-based design loops \cite{sun2020surrogate} without the need to implement discrete adjoints. Third, data assimilation can be integrated by appending measurement-mismatch terms to the same training objective, enabling hybrid regimes in which both observations and physics constrain the state estimation \cite{zaki2025data,karniadakis2021piml}. These advantages explain the broad adoption of PINNs across transport, wave propagation, electromagnetics, and fluid mechanics. 


\subsection{Limitations and failure modes}

Although the underlying problem is mathematically well-posed, training PINNs to achieve high accuracy remains computationally demanding even with current state-of-the-art approaches \cite{wang2025simulating}. Moreover, the nonconvex optimization landscape implies that convergence to a global minimum is generally not guaranteed, and error bounds are a priori non-deterministic. A primary challenge in training PINNs is that the optimization seeks a high-dimensional continuous function that satisfies multiple constraints over an extended spatiotemporal domain. The use of soft constraint enforcement through PDE residual penalties leads to poorly conditioned optimization problems \cite{krishnapriyan2021failuremodes}, particularly for nonlinear systems and problems with broader Fourier scales. Moreover, the training algorithm is a constrained optimization problem that is commonly reformulated as an unconstrained penalty objective. In practice, the governing equations and boundary conditions are enforced through weighted residual terms added to the loss function rather than as exact constraints, which introduces sensitivity to penalty weights and degrades the optimization process \cite{wang2021understanding}. The resulting loss landscape is often ill-conditioned, with competing gradients arising from residual, boundary, and data terms. Empirically, this manifests as training stagnation, convergence to inaccurate solutions \cite{du2023state}, or strong dependence on hyperparameter choices \cite{wang2021understanding,wang2022ntk}.

Several mechanisms contribute to these limitations. Neural networks commonly exhibit spectral bias, whereby low-frequency components are learned at a faster rate than high-frequency components. This behavior degrades performance for problems featuring sharp gradients and multiscale structures \cite{rahaman2019spectralbias, chandravamsi2025spectral}. Boundary and interface constraints add further stiffness: even when the PDE residual can be reduced, the boundary losses may dominate parameter updates or remain underfit, depending on weighting \cite{wang2021understanding}. Time-dependent systems introduce additional imbalance because early-time errors can propagate forward and dominate the optimization dynamics if not controlled \cite{wang2024causal}. Finally, the computational cost and memory requirements are substantially larger than those of traditional numerical methods \cite{du2023state, shaviner2025pinns}. Each training step evaluates the differential operator via automatic differentiation, often requiring first- and second-order derivatives and multiple passes through the network. In practice, for many PDEs, the floating point operation count required to reach a given accuracy is far larger than that of a classical solver, even before accounting for the cost of hyperparameter searches and sensitivity to random initialization \cite{wang2021understanding,krishnapriyan2021failuremodes}. Furthermore, PINNs do not inherently encode temporal causality, since the solution is learned over the entire space-time domain simultaneously. This leads to a rapid loss of accuracy when extrapolating to the space-time domain beyond the training window \cite{bonfanti2024generalization}.



\subsection{Methods to improve convergence and accuracy}

A growing body of work has examined the optimization challenges inherent to PINN training and proposed a range of mitigation strategies to improve convergence and accuracy. A comprehensive overview of these developments is provided in \cite{toscano2025pinns}. The most effective strategies can be grouped by the mechanisms they address.

\begin{enumerate}[leftmargin=2em]
\item \textbf{Architectural modifications.}
\begin{itemize}
\item \textit{Fourier feature embeddings.} Mapping inputs to sinusoidal bases before the first hidden layer broadens the spectral support of the network activations across the network layers \cite{chandravamsi2025spectral}, improving the approximation of high-wavenumber components and mitigating spectral bias, which is essential for multiscale problems \cite{rahimi2007random,tancik2020fourierfeatures}.
\item \textit{Residual and adaptive network architectures.} Residual connections combined with learnable adaptive weighting of residual blocks improve gradient flow and reduce optimization stiffness, as demonstrated by PirateNets \cite{wang2024piratenets, wang2025simulating} for multiscale physics-informed learning.
\item \textit{Model and domain decomposition and multi-network formulations.} Decomposing the physical domain into subdomains and training coupled neural networks with interface constraints reduces optimization stiffness and improves scalability for high Reynolds number flows \cite{jagtap2020xpinns, roy2024adaptive}.
\item \textit{Hard boundary constraints.} Dirichlet and periodic constraints can be enforced exactly by constructing the network output to satisfy boundary or initial conditions by design, reducing the number of competing loss terms \cite{berrone2023enforcing, lu2021physics, dong2021}.
\end{itemize}

\item \textbf{Initialization.}
\begin{itemize}
\item \textit{Weight reparameterizations.} Reparameterizing network weights to decouple magnitude and direction improves conditioning and gradient propagation, an idea formalized through weight normalization \cite{salimans2016weightnorm} and later adapted to PINNs via random weight factorization to mitigate gradient pathologies \cite{wang2021understanding}.
\end{itemize}

\item \textbf{Loss balancing strategies.}
\begin{itemize}
\item \textit{Residual-based weighting.} These methods adapt loss weights using residual magnitudes or residual statistics to emphasize poorly satisfied constraints during training, including residual-based attention and related pointwise weighting schemes \cite{wang2021understanding,mcclenny2023self}.
\item \textit{Gradient-based weighting.} These methods adapt loss weights using gradient magnitudes or relative training rates with respect to network parameters, including gradient normalization \cite{chen2018gradnorm} and learning-rate-annealing-based strategies \cite{wang2021understanding}.
\end{itemize}

\item \textbf{Sampling and curricula.}
\begin{itemize}
\item \textit{Residual-based adaptive refinement.} Iteratively adding or redistributing collocation points based on residual magnitude focuses computation where the solution is prone to high fitting error \cite{wu2023sampling, mao2023physics}.
\item \textit{Curriculum strategies.} Training can be stabilized by progressively increasing problem difficulty, for example, through scheduled learning rate decay, gradual ramping of loss coefficients, or continuation in parameters such as the Reynolds number, which guides optimization toward physically consistent solution regimes \cite{bengio2009curriculum,krishnapriyan2021failuremodes}.
\item \textit{Causality-aware training for time-dependent problems.} Temporal domain segmentation and reweighting those time segments to respect temporal causality reduces backward contamination from future-time residuals and can stabilize long-horizon dynamics \cite{penwarden2023causal, wang2024causal}.
\end{itemize}

\item \textbf{Optimization algorithms.}
\begin{itemize}
\item \textit{First-order optimizers.} Adaptive stochastic methods such as Adam remain standard for early training due to robustness to noisy gradients and large parameter spaces \cite{kingma2014adam}.
\item \textit{Quasi-Newton and second-order methods.} L-BFGS is widely used as a refinement stage once the iterate enters a favorable region of parameter space \cite{liu1989lbfgs}. More recent curvature-approximating methods based on Nystrom ideas offer additional options for large-scale settings \cite{singh2021nysnewton,rathore2024pinnlosslandscape}.
\end{itemize}
\end{enumerate}

Neither individual methods nor their combinations guarantee convergence or accuracy. In practice, convergence typically requires a coherent combination: a representation that can express the target field, a constraint handling strategy that avoids unnecessary penalty stiffness, a loss weighting scheme that prevents gradient domination, and an optimizer schedule that transitions from exploration to rapid local convergence.

\subsection{Quantum PINNs}

Recent advances in quantum computing have drawn growing interest due to its potential to accelerate scientific machine learning. Quantum PINNs (QPINNs) extend the classical PINNs paradigm by embedding parametrized quantum circuits (PQCs) within the neural network ansatz, either partially or fully replacing classical components. Existing approaches can be broadly categorized into three main classes. The first approach is the direct application of classical PINNs to quantum-mechanical equations, where the network architecture is classical and only the underlying physical system is quantum~\cite{norambuena2024physics}. This is not the focus of this work. The second approach employs predominantly quantum QPINNs with at most two classical layers, one as the input layer and the other as the output layer~(section 4 in \cite{trahan2024quantum}). The third is quantum-classical hybrid PINNs, where a major part of the network is classical~(section 5 in \cite{trahan2024quantum}, \cite{chen2025quantum}). By leveraging this approach, previous studies have demonstrated significant reductions in learnable parameters and improvements in accuracy (ranging from 19\% (for both~\cite{chen2025quantum}) to 58\% and 59\% (respectively~\cite{trahan2024quantum})), depending on the specific case. However, these gains come with nontrivial practical and computational challenges.

In particular, QPINNs are challenging to simulate on quantum hardware and do not permit storage of intermediate states or conventional backpropagation, because measurement collapses the quantum state. Other methods are needed to differentiate the quantum layers to obtain derivatives for parameter updates and the derivatives of the output fields. To differentiate a PQC on actual quantum hardware (QPU), the current widespread method is \textit{parameter-shift rule}~\cite{mitarai2018quantum, schuld2019evaluating}, which is an analytic computation of the derivatives (opposed to finite difference methods, which are an approximation). The parameter-shift rule methodology requires multiple circuit evaluations per derivative. Even for a gradient evaluation at a \textbf{single collocation point} in a \textbf{single epoch}, \textbf{excluding shot count}, the number of circuit evaluations scales with derivative order and the number of trainable parameters. This overhead motivates a careful quantification of how circuit evaluations scale with the order of required derivatives and the number of trainable parameters. To quantify this overhead, we derive a formal estimate for the circuit-evaluation complexity of QPINNs under parameter-shift differentiation. The result shows that the number of circuit evaluations grows exponentially with the order of derivatives required by the PDE and linearly with the number of trainable parameters. The full derivation is presented in Section~\ref{subsec:qpinns}. 


\subsection{Multi-GPU strategies for PINNs}

The computational algorithm of PINNs is dominated by repeated forward evaluations of the network and gradients of the cost with respect to all model parameters via automatic differentiation. While this workload is well-suited to accelerators, as the problem size increases (number of parameters, dimensionality, and collocation point size), the stress on memory bandwidth and device memory due to the storage of intermediate activations and higher-order derivative graphs increases. The memory footprint and computational cost, therefore, grow with the number of parameters, the number of collocation points, and the order of derivatives required by the PDE. For large collocation sets or coupled systems, a single device can become a bottleneck in both memory capacity and wall-clock time. In such settings, distributing the computation across multiple GPUs is often necessary to fit the problem in memory and to maintain practical training times.

The simplest scaling route to multi-GPU execution of PINNs is data parallelism via colocated points, also referred to as \textit{distributed data parallelism (DDP)}. Since the residual and boundary losses decompose as sums over points, each GPU can process a disjoint mini-batch of collocation points and compute the corresponding local gradients. After completing the forward and backward passes on each device, the gradients are synchronized across GPUs through an all-reduce operation, and the averaged gradients are then used to update the replicated model parameters consistently on all devices. Distributed data parallel implementations in modern deep learning systems have been shown to provide this pattern with relatively little code overhead \cite{paszke2019pytorch}. The approach scales the effective batch size with the number of devices and increases the maximum feasible number of collocation points by distributing memory usage across GPUs. A key limitation of this approach is that it does not increase the model's representation capacity, since each GPU maintains a full replica of the network. This can become restrictive on training when very large models or highly detailed solutions are being optimized for.

In practice, the efficiency of DDP primarily hinges on controlling the communication overhead. PINN training often uses large networks and expensive backward passes, so gradient synchronization can be overlapped with backpropagation. Performance also depends on how collocation points are generated and batched. Spatially structured batches can improve cache locality and reduce variance in residual magnitudes, while adaptive sampling can concentrate compute where it matters, reducing the total number of points needed for a given accuracy \cite{lu2021deepxde,wu2023sampling}. Memory overload can be reduced through \textit{mixed precision} and careful management of automatic differentiation graphs, especially for second derivatives \cite{micikevicius2018mixedprecision}. Mixed precision is not universally safe for PINNs because residual terms can be sensitive to roundoff and scaling, but it can be effective when combined with nondimensionalization and loss balancing \cite{micikevicius2017mixed, xue2022novel}. Additional engineering tactics include gradient checkpointing, micro-batching, and separating forward passes for different loss components when memory is the limiting factor.

A complementary scaling route is \textit{model and domain parallelism}. Domain decomposition methods split the domain into subregions, train separate networks for each subregion, and enforce interface continuity via additional constraints \cite{jagtap2020xpinns}. The tradeoff is an expanded set of interface losses and coordination costs. On the other hand, modern large language models rely on hybrid parallelism strategies that combine data parallelism, tensor model parallelism~\cite{shoeybi2019megatron}, and pipeline parallelism~\cite{huang2019gpipe, narayanan2021efficient}. Tensor parallelism partitions weight matrices across devices so that a single layer is computed collectively. Pipeline parallelism distributes consecutive layers across devices and overlaps computation with communication. More recently, mixture-of-experts architectures activate only a subset of parameters per input, enabling sparse model scaling with manageable communication costs~\cite{fedus2022switch, lepikhin2020gshard}. These strategies are effective for transformer architectures because dense linear layers with regular structure and uniform memory access patterns dominate the computation. In contrast, PINNs require repeated higher-order automatic differentiation, irregular derivative graphs, and tight coupling between outputs and spatial derivatives, which limits the benefits of tensor and pipeline parallelism. As a result, while model parallel strategies from LLM training are conceptually transferable, their efficiency gains are less straightforward in the PINN setting and often require substantial customization.

\subsection{This work: QPINNACLE}
Existing frameworks such as DeepXDE \cite{lu2021deepxde}, Modulus \cite{hennigh2021nvidia}, and SciANN \cite{haghighat2021sciann} offer PINN implementations with varying levels of flexibility and scalability. However, hybrid quantum-classical extensions remain limited \cite{hao2024pinnacle}. The present work features an in-house-developed PyTorch-based library, \textit{TorQ - Tensor Operations for Research of Quantum systems}~\cite{torq}, which adopts a hybrid quantum-classical strategy and demonstrates its advantages on a simulated quantum hardware running on a classical computing machine. The framework also integrates convergence-enhancing techniques to the hybrid quantum-classical architectures within a unified modular workflow. Distributed data parallelism is used to analyze scaling behavior under memory and communication constraints~\cite{paszke2019pytorch}, and the results quantify both predictive accuracy and computational limitations. The main contributions of this work are:
\begin{itemize}
\item A modular open-source framework for PINNs that integrates convergence-enhancing strategies, multi-GPU execution, and hybrid quantum-classical models.
\item A formal complexity analysis of quantum PINNs based on parameter-shift differentiation.
\item Reproducible benchmark study across representative PDEs, quantifying the effect of architectural and training choices on convergence and accuracy.
\item Analysis of distributed data parallel PINN training, including scaling behavior in runtime and memory usage.
\end{itemize}

\section{Physics Informed Neural Networks: overview and convergence enhancing methods}

This section offers a \textit{from-scratch} overview of the key components in building a PINNs model. It begins with the formulation of PDEs, followed by non-dimensionalization, the design of the network architecture, the definition of loss functions, and a standard training methodology for a vanilla PINNs model. Subsequently, an overview of various convergence-enhancing techniques, primarily based on the methods outlined by Wang et al. \cite{wang2023expertsguide}, is presented to improve PINN training performance.

\subsection{Problem definition, PDE formulation, and non-dimensionalization} \label{PDEformulation}
To fully define a physical problem, it is necessary to specify the spatio-temporal domain \((\mathbf{x}, t)\), the variables of interest, the governing equations, as well as the initial and boundary conditions. Let $\textbf{x} \in \Omega$ represent the spatial domain and $t \in [0, T]$ the temporal domain. The generic form of a partial differential equation (PDE), along with the initial and boundary conditions employed in Physics-Informed Neural Networks (PINNs), can be mathematically expressed as follows:

\begin{equation}
    \text{PDE:} \quad \underbrace{\mathcal{L}_t\left[u(\mathbf{x},t)\right]}_{\text{time derivatives}} + \underbrace{\mathcal{L}_\mathbf{x}\left[u(\mathbf{x},t)\right]}_{\text{spatial derivatives}} = \underbrace{f\left(\mathbf{x},t\right)}_{\text{source term}}, \quad \mathbf{x} \in \Omega, \, t \in [0, T],
\end{equation}

\begin{equation}
\begin{aligned}
    \text{Initial conditions (IC):}& \quad u(\mathbf{x}, t = 0) = g(x), \quad &\mathbf{x} \in \Omega, \\
    \text{Boundary conditions (BC):}& \quad B\left[u(\mathbf{x}_b, t)\right] = 0, \quad & \mathbf{x}_b \in \partial \Omega,  t \in [0, T].
\end{aligned}
\end{equation}

Here, $\mathcal{L}_t[\cdot]$ and $\mathcal{L}_\mathbf{x}[\cdot]$ denote a linear/nonlinear differential operator in time and space, respectively. $B[\cdot]$ represents the boundary conditions, such as Dirichlet, Neumann, or Robin. 

\textit{Non-dimensionalization.} The equations and conditions often involve variables with vastly different scales, for example, spatial coordinates \(\mathbf{x}\) may represent distances in kilometers, while time \(t\) can range from milliseconds to several seconds, leading to large input values. Such disparity in scales can complicate the training process of the neural network by introducing numerical instability or causing gradients involved in the back-propagation to vanish or explode during optimization.

To address this, the domain and governing equations are typically non-dimensionalized. Non-dimensionalizing the equations helps standardize the problem by reducing the scale of the variables to a more manageable scale that the network can handle. This process enhances the network's ability to learn and generalize by focusing on the relative importance of different terms, making the problem more manageable. Moreover, non-dimensionalization improves numerical stability during training, helping the network avoid issues associated with handling very large or very small numbers. This often leads to faster convergence.

\subsection{Network architecture}

\begin{figure}[t!]
    \centering
    \includegraphics[width=0.4\linewidth]{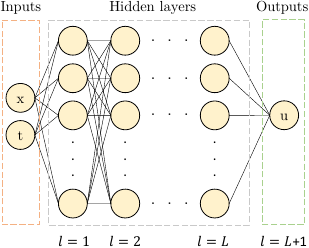}
    \caption{Schematic of a multi-layer perceptron mapping two inputs ($x,t$) to one output $u$ (solution to PDE), with its model parameters.}
    \label{fig:MLP}
\end{figure}


To establish a functional relationship between the inputs \((\mathbf{x}, t)\) and outputs $u$, a fully connected network can be employed. Selecting the rigwell-suited torchitecture is critical for solving complex PDEs effectively. A cornerstone of this architectural framework is the Multi-Layer Perceptron (MLP), known for its universal approximation capabilities and therefore suitable for approximating the latent functions that characterize solutions to PDEs. The network output depends on the inputs \((\mathbf{x}, t)\) and model parameters $\theta$ that are weights and biases involved in the network. Therefore, the functional relationship between the inputs and outputs is written as follows:

\begin{equation}
    \text{MLP:} \quad u = f(\mathbf{x},t, \theta)
\end{equation}

Let us consider an MLP with one input layer, $L$ hidden layers, and one output layer. Each layer in the MLP computes values from the previous one using a set formula:

\begin{equation}
    \text{First `$L$' layers:} \quad f^{(l)}(x) = W^{(l)} \cdot g^{(l-1)}(x) + b^{(l)}, \quad l = 1, 2, \dots, L,
\end{equation}

\begin{equation}
\begin{aligned}
    \text{Applying the activation function:} \quad g^{(l)}(x) = \sigma \left( f^{(l)}(x) \right), \quad l = 1, 2, \dots, L,
\end{aligned}
\end{equation}

\begin{equation}
    \text{Final layer (no activation):} \quad f_\theta^{(l)}(x) = W^{(L+1)} \cdot g^{(L)}(x) + b^{(L+1)}.
\end{equation}

Here, $W^{(l)}$
and $b^{(l)}$ represents the weight matrix and bias vector for the $l$-th layer, respectively, and $\sigma$ denotes an element-wise activation function. To ensure a continuously differentiable neural representation,  a hyperbolic tangent (Tanh) activation function is usually recommended in the literature.

\subsection{Loss function} \label{sec:losses}
The MLP's default random parameters are unlikely to satisfy the governing equations for obvious reasons. The question at hand is how to update the random model parameters to satisfy the governing equations, initial conditions, and boundary conditions. Firstly, to mathematically estimate how far the current network is from producing a correct solution, the error, or the `loss', is estimated. To illustrate the definition of the loss function, let us consider Burgers' equation on the domain \(x \in [0,1]\) and \(t \in [0,1]\), with the governing equations, initial conditions, and boundary conditions written as follows:

\begin{align}
    \text{Governing Equation:} & \quad \frac{\partial u}{\partial t} + u \frac{\partial u}{\partial x} = 0,  \quad x \in [0,1], t \in [0,1]\\
    \text{Boundary Conditions:} & \quad u(x=0,t) = u(x=1,t), \\
    \text{Initial Condition:} & \quad u(x,0) = u_0(x).
\end{align}

In PINNs, the physical laws and constraints define the loss function. A commonly used function for estimating loss is the `mean squared error' (MSE). Using the current model parameters, \textit{forward pass} is performed to obtain the model output for various input combinations of $x$ and $t$. Using MSE, various loss components involved in training PINNs are computed as follows:

\begin{align}
    \mathcal{L}_{pde}(\theta) &= \frac{1}{N_p} \sum_{i=1}^{N_p} \left[ \left( \frac{\partial u_i(x_i,t_i)}{\partial t} \right)_{\text{autoGrad}} + u_i(x_i,t_i) \left( \frac{\partial u_i(x_i,t_i)}{\partial x} \right)_{\text{autoGrad}} \right]^2,\\
    \mathcal{L}_{bc}(\theta) &= \frac{1}{N_{bc}} \sum_{i=1}^{N_{bc}} \left( u(x=0,t) - u(x=1,t) \right)^2, \\
    \mathcal{L}_{ic}(\theta) &= \frac{1}{N_{ic}} \sum_{i=1}^{N_{ic}} \left( u_{ic,i} - u_0(x_i) \right)^2.
\end{align}

where: 
\(\theta\) represents the parameters of the neural network.
\(N_p\) is the number of collocation points used to evaluate the PDE residual,
\(N_{bc}\) is the number of boundary condition points,
\(N_{ic}\) is the number of initial condition points,
\(u_i(x_i,t_i)\) represents the predicted values at different points in the domain,
\(u_{bc,i}\) represents the values at boundary condition points, and
\(u_{ic,i}\) represents the values at initial condition points. The total loss is computed by summing up the individual losses:

\begin{equation}
    \mathcal{L}_{total}(\theta) = \mathcal{L}_{pde}(\theta) + \mathcal{L}_{bc}(\theta) + \mathcal{L}_{ic}(\theta)
\end{equation}

The derivatives involved in estimating the residuals above are computed using the autograd functions available in PyTorch and other machine learning frameworks. By minimizing \(\mathcal{L}_{total}(\theta)\), the network parameters are updated to satisfy best the governing equations, initial conditions, and boundary conditions.

The next step is to determine how to update the model parameters to minimize the total loss $\mathcal{L}_{total}(\theta)$. The loss function typically defines a nonconvex optimization landscape, and training amounts to locating a suitable minimum. In PINNs, this is commonly addressed using gradient-based optimization methods such as gradient descent, Adam, and L-BFGS. These methods compute gradients of the loss with respect to the model parameters $\theta$ and update the parameters in directions that reduce the loss. Mathematically, for a gradient-based optimizer, the parameters are updated iteratively as follows:

\[
\theta^{(k+1)} = \theta^{(k)} - \eta \nabla_{\theta} \mathcal{L}_{total}(\theta)
\]

where:
\begin{itemize}
    \item \(\theta^{(k)}\) represents the model parameters at the \(k\)-th epoch,
    \item \(\eta\) is the learning rate or step size, which is a user-specified hyperparameter,
    \item \(\nabla_{\theta} \mathcal{L}_{total}(\theta)\) is the gradient of the loss function with respect to the parameters, computed using automatic differentiation (autograd).
\end{itemize}

Each optimizer employs a distinct strategy for selecting update directions and step sizes based on accumulated gradient information. Adaptive first-order methods, such as Adam, rescale parameter-wise updates using estimates of the first and second moments of gradients, thereby improving robustness during the early stages of training. In contrast, quasi-Newton methods such as L-BFGS construct low-rank approximations of the second-order curvature information, enabling faster local convergence once the optimization enters a favorable region of the parameter space. When applied appropriately, these optimization schemes enable progressive refinement of the network parameters toward solutions that satisfy the governing equations, boundary conditions, and initial conditions within prescribed accuracy levels.

\subsection{PINN training model}

\begin{figure}[t!]
    \centering
    \includegraphics[width=\linewidth]{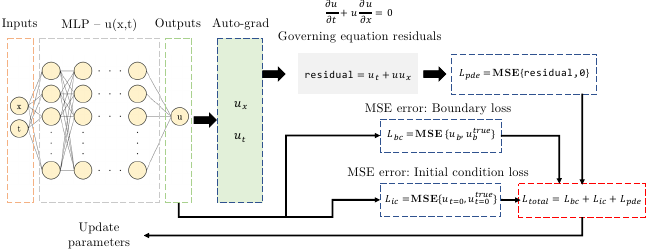}
    \caption{Schematic of a basic PINNs model for solving the Burgers' equation. Here, $u_b$ and $u_b^{\text{true}}$ denote the network-estimated and true boundary values, respectively, while $u_{t=0}$ and $u_{t=0}^{\text{true}}$ represent the network-estimated and true initial values at $t = 0$. The function MSE represents the mean square error.}
    \label{fig:vanilla}
\end{figure}

By combining an MLP-based network architecture, an MSE-based loss function, and a gradient-based optimizer, a basic PINNs model is fully specified. A schematic of this vanilla PINNs setup is shown in Fig.~\ref{fig:vanilla}. Although a fully connected neural network can approximate a function in the limit of sufficiently many parameters and the underlying optimization problem is mathematically well-posed, such PINN models are notoriously difficult to train and often exhibit poor convergence on moderately simple problems, such as solving a 1D advection equation. This difficulty arises from limitations in capturing the intricate features of high-dimensional, strongly nonlinear PDEs, as well as from the increasingly challenging nonconvex optimization landscape. As problem complexity grows, effective training requires more advanced strategies to address these optimization challenges. The next section reviews mathematical techniques from the literature that aim to improve convergence and training stability.

\subsection{Quantum neural network (QNN) architecture} \label{sec:qnn_arch}
\begin{figure}[t!]
    \centering
    \includegraphics[width=\linewidth]{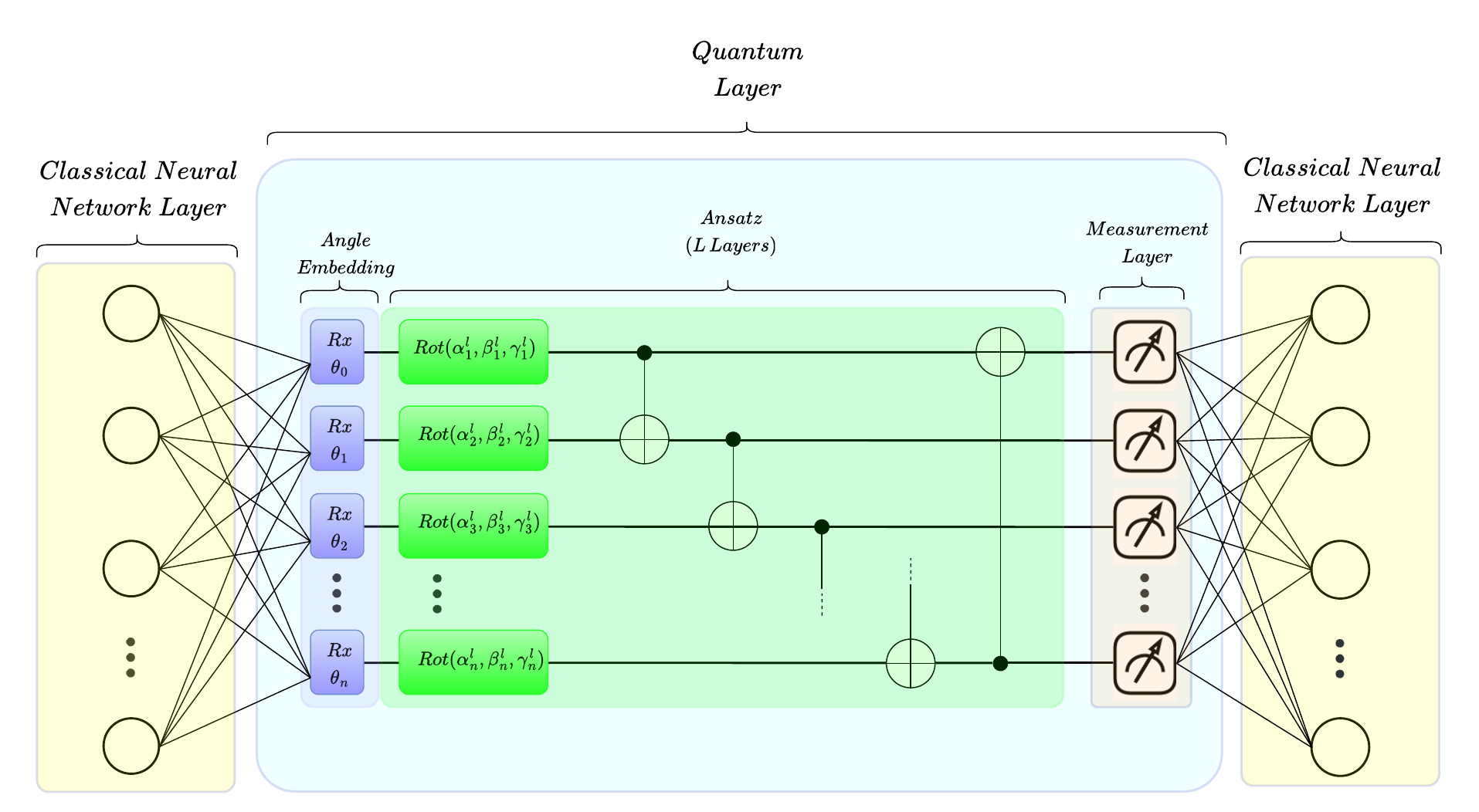}
    \caption{Schematic of the parametrized quantum circuit (PQC) integration between two classical neural network layers. The preceding and subsequent layers adjust dimensions for the quantum layer's qubits (each qubit acting as a neuron).}
    \label{fig:pqc_arch}
\end{figure}
As their name suggests, quantum neural networks (QNNs) treat a parametrized quantum circuit (PQC) as a neural network layer. As seen in Fig.~\ref{fig:pqc_arch}, a PQC can be integrated between two neural network (NN) layers, and specifically PINN layers in this case, by using the following components:
\begin{itemize}
    \item \textbf{Data encoding:} To be an integral part of the NN, the results of the previous layer need a way to be propagated into the PQC downstream, so a data encoding (many times called data embedding) method is used. There are several such methods, and the most common are:
    \begin{itemize}
        \item \textbf{Angle embedding:} This method maps classical inputs to quantum states by assigning each qubit a rotation angle derived from the preceding layer’s activations. This approach is straightforward to implement and integrates naturally with neural network pipelines. However, the encoding capacity scales linearly with the number of qubits, $\mathcal{O}(n)$, which limits representational efficiency compared to alternative schemes. In addition, the periodicity of rotation gates with period $2\pi$ requires careful input scaling to avoid aliasing effects. Despite these limitations, angle embedding remains the most commonly used encoding strategy in QPINNs~\cite{mitarai2018quantum}.
        \item \textbf{Amplitude embedding:} By using amplitude embedding, the quantum circuit encodes the input data to the amplitude of each possible quantum state, thus gaining an exponential scaling ($\mathcal{O}(2^n)$) in the qubit amount. The main issue with this method in the general case is that it requires arbitrary state preparation, which is inefficient and requires $\mathcal{O}(2^n)$ quantum gates, unless the encoded data has a structure that can be exploited for a more efficient encoding~\cite{lloyd2013quantum}.
    \end{itemize}
    \item \textbf{Ansatz:} An ansatz is the structure of the computational part of the circuit, and can be thought of as a family of circuits. In a PQC, the ansatz contains parameterized gates that depend on the learned parameters, along with fixed gates in many cases. An ansatz's shape can depend on the number of qubits and the number of ansatz layers that 'multiply' the same ansatz shape, but with different parameters; both of those can be considered hyperparameters. 
    \item \textbf{Measurement:} To obtain classical results and their propagation to the next classical layer, measurement of quantum observables~\cite{nielsen2010quantum} is conducted at the end of the circuit. The quantum state before measurement ($\ket{\psi}$) is of an exponential size in the number of qubits $n$, i.e., $\ket{\psi} \in \mathcal{H}^{2^n}$, where $\mathcal{H}$ is a Hilbert space. However, it should be noted that to obtain the exact state $\ket{\psi}$, full state tomography~\cite{james2001measurement} is needed, which requires $\mathcal{O}(2^n)$ circuit runs with different observable measurements, thereby diminishing any possible quantum speedup. Therefore, only observables that would capture the needed information from the quantum state should be used.
\end{itemize}

The PQC is differentiable with respect to both its trainable parameters and its embedded input angles, enabling end-to-end training with the same objective used for the classical NN.

\subsection{Enhancing PINNs optimization}

To address challenges of convergence in PINNs, several advanced techniques summarized by Wang et al. \cite{wang2023expertsguide} are implemented in QPINNACLE. Table~\ref{tab:pinn_convergence_methods} provides a summary of various techniques from the literature to improve the training performance of PINNs.


\begin{table}[h!]
\centering
\caption{Overview of convergence enhancement strategies implemented in QPINNACLE, compiled from the expert guide on training physics-informed neural networks by Wang et al. \cite{wang2023expertsguide}. The first four approaches modify the network architecture, whereas the remaining methods act on the loss formulation or the training procedure.}
\label{tab:pinn_convergence_methods}
\renewcommand{\arraystretch}{1.2}
\setlength{\tabcolsep}{6pt}

\begin{tabularx}{\textwidth}{l X l}
\hline
\textbf{Method} & \textbf{Description} & \textbf{Reference(s)} \\
\hline

Random Fourier features 
& Mitigates spectral bias of the network to learn high-wavenumber features of the network output(s). The accuracy and convergence are faster. 
& \begin{tabular}[t]{@{}l@{}}
Rahimi et al.~\cite{rahimi2007random} \\
Tancik et al.~\cite{tancik2020fourierfeatures}
\end{tabular} \\

Random weight factorization 
& Initializes network weights with random values extracted from a scaled normal distribution to keep the activation and gradient variance balanced across layers. This prevents both very small and very large gradients, which can disrupt learning in deep networks. 
& Wang et al. \cite{wang2022random} \\

Strict boundary conditions 
& Alters the network architecture to enforce strict boundary conditions, eliminating the need to compute boundary-condition losses. 
& Dong et al. \cite{dong2021} \\

Periodic activation function 
& Uses a sinusoidal activation to help learn high-frequency features in the solution faster, similar to the Random Fourier Features method. 
&  \begin{tabular}[t]{@{}l@{}}
Sitzmann et al. \cite{sitzmann2020implicit} \\
Liu et al.~\cite{liu2024finer}
\end{tabular} \\

Loss balancing 
& Dynamically adjusts the weights of different loss components based on their sensitivity to network parameters, preventing training bias and ensuring balanced learning across all loss components. 
& \begin{tabular}[t]{@{}l@{}}
Wang et al. \cite{wang2021understanding} \\
Xiang et al. \cite{xiang2022self}
\end{tabular} \\

Combined use of optimizers 
& After an initial training phase with a first-order optimizer such as Adam, which uses only gradient information, the optimization algorithm is switched to a second-order optimizer like L-BFGS or NysNewton, which approximates the curvature of the loss function, to accelerate convergence. 
& \begin{tabular}[t]{@{}l@{}}
Gu et al. \cite{gu2023nysnewton} \\
Noceda and Wright \cite{nocedal2006numerical}
\end{tabular} \\

Curriculum training 
& Gradually increases the complexity of training to ensure accurate convergence. For example, in fluid flow problems, the solution for a high Reynolds number is achieved by first training the model on lower Reynolds numbers, progressively increasing the Reynolds number to the target value. 
& Krishnapriyan et al. \cite{krishnapriyan2021failuremodes} \\

Temporal causality 
& Divides the temporal domain into multiple segments to assign different weights to corresponding losses, allowing the network to learn from past states while preserving the integrity of temporal dynamics. 
&  \begin{tabular}[t]{@{}l@{}}
Wang et al. \cite{wang2024respecting} \\
Penwarden et al. \cite{penwarden2023unified}
\end{tabular} \\


\hline
\end{tabularx}
\end{table}

\subsubsection{Random Fourier features} \label{sec:RFF}
Random Fourier feature (RFF) embedding is a technique that allows MLPs to represent functions with higher-frequency content \cite{rahimi2007random, rahaman2019spectralbias}. It aims to address the spectral bias observed in PINNs by mapping input variables to a randomized Fourier feature space, thereby enhancing the representation of high-frequency components in the target solution.
\begin{equation}
    \gamma(x) = \left[ \cos(Bx), \, \sin(Bx) \right],
\end{equation}

where $B$ is a matrix sampled according to a Gaussian distribution $\mathcal{N}(0,\sigma^2)$, enhancing the network's ability to approximate complex and sharp solution features. Where $\sigma$ is a hyperparameter chosen based on the expected intricate features in the solution, it typically lies in the range of [1,10] for applications considered in this manuscript. For a simple use case with two inputs $(x,t)$, the mapping is depicted pictorially in Figure \ref{fig:rff}.

\begin{figure}[h!]
    \centering
    \includegraphics[width=0.8\linewidth]{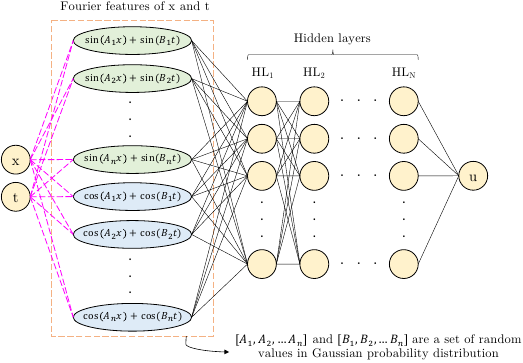}
    \caption{Schematic of the training model incorporating a \textit{Random Fourier Features (RFF)} layer between the input layer and the first hidden layer, $\text{HL}_1$. The RFF layer has a width twice that of the first hidden layer $\text{HL}_1$. The coefficients $A_i$ and $B_i$ are randomly sampled from a Gaussian probability distribution. The mapping of inputs \( x \) and \( t \) to random Fourier features is solely to avoid spectral bias and does not involve multiplying or adding any weights or biases.}
    \label{fig:rff}
\end{figure}

\subsubsection{Random weight factorization} \label{sec:RWF}

Random Weight Factorization is an initialization method for customizing the weights associated with each linear mapping in the PINN network. The method is based on a scale factor and a direction vector, which are used to multiply the network weights. This factorization allows the model to independently adjust the magnitudes and directions of its weight vectors, leading to improved training dynamics and model accuracy. The weights are then combined in a structured manner:

\begin{equation}
    W^{(l)}=\operatorname{diag}\left(\exp \left(s^{(l)}\right)\right) \cdot v^{(l)}, \quad l=1,2, \ldots, L+1
\end{equation}

where $W^{(l)}$ represents the weight matrix, $\text{diag}\left(\exp(s^{(l)})\right)$ represents a diagonal matrix containing scale factors, and $v^{(l)}$ is a matrix of direction vectors for layer $l$. For each weight matrix, we initialize the scale vector factors where $s^{(l)}$ is sampled from a multivariate normal distribution $\mathcal{N}(\mu, \sigma^2 I)$. This ensures the scale factors cover a broad spectrum of magnitudes, preventing zeros or small values, which is vital for maintaining training stability and effectiveness. Exponential parameterization is critical for averting performance decline to that of a conventional MLP and for avoiding instability during training. This strategic decomposition ensures that each neuron in the network can adapt its response with greater flexibility, leading to enhanced learning dynamics and better convergence properties.

\begin{figure}[h!]
    \centering
    \includegraphics[width=\linewidth]{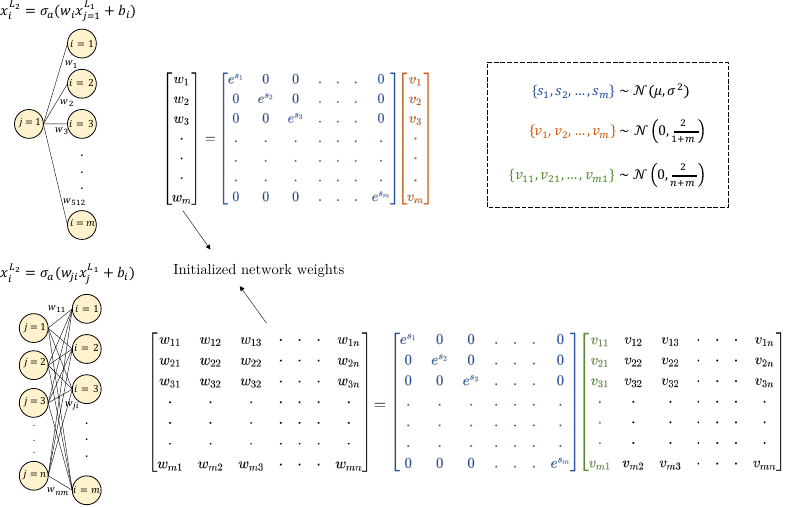}
    \caption{Schematic of two-layer perceptron and the matrix multiplication required for calculating network weights using the \textit{random weight factorization} technique. The top and bottom rows both show two-layer perceptrons: the top with one input feature and $m$ output features, and the bottom with $n$ input features and $m$ output features.}
    \label{fig:rwf2}
\end{figure}

\begin{figure}[h!]
    \centering
    \includegraphics[width=0.9\linewidth]{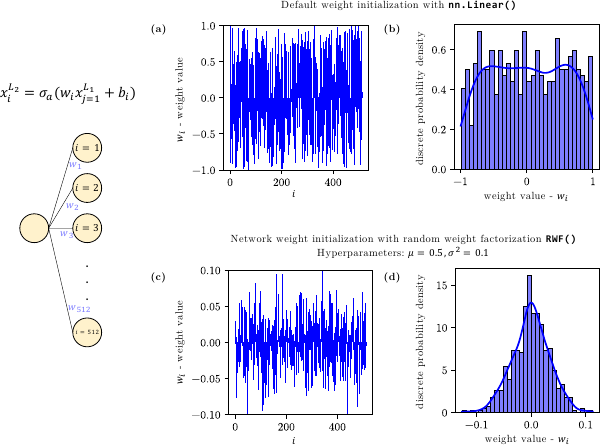}
    \caption{Visualization of weights during network initialization and their corresponding discrete probability distribution for a two-layer perceptron with one input feature and 512 output features. The top row corresponds to the default initialization method in the PyTorch library, while the bottom row corresponds to the \textit{random weight factorization} technique. } 
    \label{fig:rwf1}
\end{figure}

\subsubsection{Imposing strict periodic boundary conditions}

\begin{figure}[h!]
    \centering
    \includegraphics[width=0.95\linewidth]{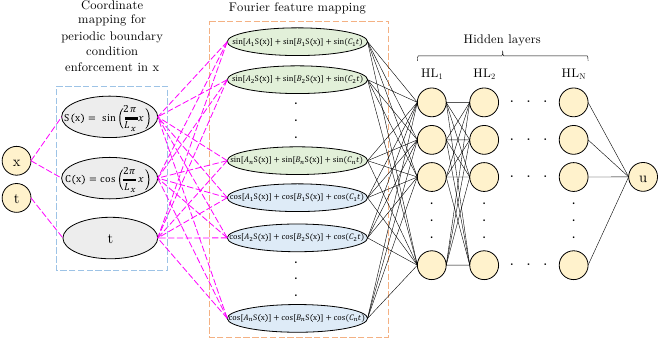}
    \caption{Schematic network model showing the imposition of strict periodic boundary condition in $x$ direction with a period length equal to $L_x$. The input $x$ is mapped to a Fourier space to enforce periodicity of length $L_x$; no weights or biases are involved in this mapping.}
    \label{fig:periodic}
\end{figure}

In the standard PINN formulation described in Sec.~\ref{sec:losses}, periodic boundary conditions are imposed weakly by adding a boundary penalty term to the loss function. This enforces equality of the solution and its spatial derivatives at the domain boundaries through optimization. An alternative approach is to impose periodicity directly within the network architecture. Instead of penalizing boundary mismatches, the input coordinates are mapped to a periodic representation so that the network output satisfies periodic boundary conditions by construction~\cite{dong2021}. This eliminates the need for an explicit boundary loss term. For periodic conditions in space and time, the input variables are transformed using Fourier components,
\begin{equation}
    w(x,t) =
    \left[
    \cos\!\left(\frac{2\pi}{P_x}x\right),
    \sin\!\left(\frac{2\pi}{P_x}x\right),
    \cos\!\left(\frac{2\pi}{P_t}t\right),
    \sin\!\left(\frac{2\pi}{P_t}t\right)
    \right],
\end{equation}
where $P_x$ denotes the spatial period and $P_t$ is a trainable parameter initialized to the temporal domain length. The neural network is then evaluated as $u_\theta(x,t)=\mathcal{N}_\theta(w(x,t))$.

Because sine and cosine functions are periodic, the resulting representation satisfies periodic boundary conditions identically, including equality of derivatives at the boundaries. As a result, the boundary loss term can be removed from the objective function. This also avoids unbalanced gradients between residual and boundary terms during optimization and improves training stability in problems where periodic consistency is essential. For Dirichlet boundary conditions, strict enforcement can be achieved by modifying the network output so that it satisfies the prescribed boundary values exactly, typically by multiplying the network output by a function that vanishes at the boundary and adding a term that encodes the boundary data \cite{berrone2023enforcing, du2021evolutional}.

By embedding boundary conditions directly into the architecture, PINNs can satisfy physical constraints exactly while simplifying the loss formulation and reducing the optimization burden.
\subsubsection{Loss balancing}  \label{sec:loss_balancing}
Efficient training of PINNs requires careful management of various loss components to prevent any single term from dominating the optimization process. This challenge is addressed through dynamic loss balancing, which adaptively adjusts the weights of different loss components during training based on their relative magnitudes and gradients. The general form of the composite loss function used in PINNs is expressed as:
\begin{equation}
    L(\theta)=\lambda_{i c} L_{i c}(\theta)+\lambda_{b c} L_{b c}(\theta)+\lambda_{p d e} L_{p d e}(\theta)
\end{equation}
where $\lambda_{i c}$,  $\lambda_{b c}$, and $\lambda_{p d e}$ are the respective weights that are adaptively adjusted to maintain balance in the 
training process. Initially, global weights are computed from the norm of the gradients of each weighted 
loss component:

\begin{equation}
    \begin{aligned}
    & \hat{\lambda}_{i c}=\frac{\left\|\nabla_\theta L_{i c}(\theta)\right\|+\left\|\nabla_\theta L_{b c}(\theta)\right\|+\left\|\nabla_\theta L_{p d e}(\theta)\right\|}{\left\|\nabla_\theta L_{i c}(\theta)\right\|} \\
    & \hat{\lambda}_{b c}=\frac{\left\|\nabla_\theta L_{i c}(\theta)\right\|+\left\|\nabla_\theta L_{b c}(\theta)\right\|+\left\|\nabla_\theta L_{p d e}(\theta)\right\|}{\left\|\nabla_\theta L_{b c}(\theta)\right\|} \\
    & \hat{\lambda}_{p d e}=\frac{\left\|\nabla_\theta L_{i c}(\theta)\right\|+\left\|\nabla_\theta L_{b c}(\theta)\right\|+\left\|\nabla_\theta L_{p d e}(\theta)\right\|}{\left\|\nabla_\theta L_{p d e}(\theta)\right\|}
    \end{aligned}
\end{equation}

Subsequently, these weights are updated using a moving average formula with $\alpha$ = 0.9:

\begin{equation}
    \lambda_{\text {new }}=\alpha \lambda_{\text {old }}+(1-\alpha) \hat{\lambda}_{\text {new }}, \quad \lambda=\left(\lambda_{i c}, \lambda_{b c}, \lambda_{\text {pde }}\right)
\end{equation}

\subsubsection{Optimizers}
Training physics-informed neural networks requires optimization methods that can handle nonconvex loss landscapes arising from the combined cost function that accounts for physics residual, boundary/initial conditions, and data misfit \cite{wang2021understanding}. In practice, for simple low-dimensional problems (for instance, problems involving one space and one time dimension), a two-stage strategy based on first-order and quasi-Newton methods is often effective for achieving extremely small errors \cite{wang2025discovery}.

In the initial stage, the Adam optimizer \cite{kingma2014adam} is employed. Let $\theta$ denote the vector of network parameters and let $g_k = \nabla_{\theta} \mathcal{L}(\theta_k)$ be the gradient of the PINN loss at iteration $k$. Adam updates biased first and second moment estimates,
\begin{equation}
m_k = \beta_1 m_{k-1} + (1-\beta_1) g_k, 
\qquad
v_k = \beta_2 v_{k-1} + (1-\beta_2) g_k^2,
\end{equation}
with bias-corrected forms
\begin{equation}
\hat m_k = \frac{m_k}{1-\beta_1^k}, 
\qquad
\hat v_k = \frac{v_k}{1-\beta_2^k}.
\end{equation}
The parameter update is then
\begin{equation}
\theta_{k+1} = \theta_k - \alpha \frac{\hat m_k}{\sqrt{\hat v_k} + \varepsilon},
\end{equation}
where $\alpha$ is the step size (learning rate) and $\varepsilon=10^{-8}$ ensures numerical stability. $[\beta_1, \beta_2]$ Unlike standard stochastic gradient descent (SGD), which applies the same learning rate uniformly to all parameters, Adam adaptively smoothens and rescales the gradient for each update. The first moment term $\hat m_k$ introduces a momentum effect that smooths noisy gradient directions, while the second moment term $\hat v_k$ normalizes the step size according to the magnitude of recent gradients. This per-parameter normalization reduces sensitivity to poorly scaled gradients and mitigates imbalance between loss components. In PINNs, where gradients arising from PDE residuals, boundary terms, and data terms can differ substantially in scale and conditioning, such adaptive scaling typically leads to faster and more stable reduction of the composite loss during early training compared to vanilla SGD.

For refinement, optimization can be switched to a second-order optimizer such as the L-BFGS method \cite{liu1989lbfgs}. L-BFGS constructs an approximation $H_k$ of the inverse Hessian using a limited history of parameter and gradient differences,
\begin{equation}
s_k = \theta_{k+1} - \theta_k, 
\qquad
y_k = g_{k+1} - g_k,
\end{equation}
and updates $H_k$ so that it satisfies the secant condition
\begin{equation}
H_{k+1} y_k = s_k,
\end{equation}
which enforces consistency with local curvature information. The search direction is then computed as
\begin{equation}
p_k = - H_k g_k.
\end{equation}
The parameters are updated via a line search,
\begin{equation}
\theta_{k+1} = \theta_k + \lambda_k p_k,
\end{equation}
with $\lambda_k$ chosen to satisfy standard descent conditions. This quasi-Newton phase improves local convergence and enforces closer satisfaction of the governing equations, complementing the exploratory phase driven by Adam.

\subsubsection{Curriculum training}

Curriculum training refers to a progressive optimization strategy in which the learning task is ordered from simpler to more complex configurations \cite{bengio2009curriculum}. In the context of fluid flow, a typical approach is to train a PINN first on a reduced form of the governing equations before introducing additional physical effects. 

For instance, consider the incompressible Navier-Stokes equations
\begin{equation}
\partial_t \mathbf{u} + (\mathbf{u}\cdot\nabla)\mathbf{u} + \nabla p - \nu \Delta \mathbf{u} = \mathbf{f},
\qquad
\nabla \cdot \mathbf{u} = 0,
\end{equation}
where $\mathbf{u}$ denotes the velocity field, $p$ the pressure, and $\nu$ the kinematic viscosity. A curriculum strategy may begin by neglecting the nonlinear convective term and training the network on the Stokes system
\begin{equation}
\partial_t \mathbf{u} + \nabla p - \nu \Delta \mathbf{u} = \mathbf{f},
\qquad
\nabla \cdot \mathbf{u} = 0,
\end{equation}
which is easier to optimize due to the absence of a nonlinear term in this reduced form. After convergence, the nonlinear term $(\mathbf{u}\cdot\nabla)\mathbf{u}$ is introduced and the training is continued using the full equations.

Within physics-informed neural networks, this approach can also be realized by progressively increasing the Reynolds number, or the number of collocation points, or by adjusting the loss weights associated with different residual terms during training. Such staged learning is shown to improve stability and convergence when solving nonlinear flow problems \cite{krishnapriyan2021failuremodes, wang2023expertsguide}.

Let the PINN loss be written as
\begin{equation}
\mathcal{L}(\theta; \lambda) = \lambda_{\mathrm{pde}} \, \mathcal{L}_{\mathrm{pde}}(\theta)
+ \lambda_{\mathrm{bc}} \, \mathcal{L}_{\mathrm{bc}}(\theta)
+ \lambda_{\mathrm{ic}} \, \mathcal{L}_{\mathrm{ic}}(\theta),
\end{equation}
where $\theta$ denotes the network parameters and $\lambda = (\lambda_{\mathrm{pde}}, \lambda_{\mathrm{bc}}, \lambda_{\mathrm{ic}})$ are weighting coefficients. In curriculum training, these weights are updated according to a predefined schedule
\begin{equation}
\lambda^{(k+1)} = \Phi\bigl(\lambda^{(k)}\bigr),
\end{equation}
so that the optimization initially emphasizes simpler constraints, such as boundary or initial conditions, and gradually increases the contribution of the full PDE residual.

An alternative formulation relies on adaptive sampling. Let $\Omega_k \subset \Omega$ denote the set of collocation points at training stage $k$, with $|\Omega_k| < |\Omega_{k+1}|$. The loss at stage $k$ is then
\begin{equation}
\mathcal{L}^{(k)}(\theta) = \frac{1}{|\Omega_k|} \sum_{x \in \Omega_k} \bigl\| \mathcal{R}(u_\theta(x)) \bigr\|^2,
\end{equation}
where $\mathcal{R}$ denotes the differential operator of the PDE. Increasing $|\Omega_k|$ refines the enforcement of the governing equations as training proceeds. This staged approach stabilizes optimization, reduces sensitivity to poor initializations, and improves convergence when solving complex multiphysics or multiscale problems \cite{wang2021understanding}.

\subsubsection{Learning rate scheduler}
Optimizing PINNs benefits significantly from a learning rate scheduler that dynamically adjusts the learning rate based on training progress. Typically, methods such as exponential decay or adaptive learning rate adjustments are employed. The scheduler reduces the learning rate as the training progresses, preventing overshooting in the later stages of training and helping to fine-tune the network parameters effectively.

\subsubsection{Temporal causality}
To address the essential aspect of respecting temporal causality in time-dependent problems, a structured training approach is utilized. This approach segments the temporal domain, training the network to solve the PDE across these segments sequentially. Such a method ensures that the model adheres to the natural progression of time, integrating the causality of physical processes into the learning algorithm. Consequently, the predictions for any given time step are based solely on the information from past states, 
maintaining the integrity of the temporal dynamics.

Initially, the temporal weights ${\omega}_{i=1}^M$ are set to 1. Then, they are updated for each segment $i$, the weight $\omega_i$ is updated based on the cumulative loss terms from the previous segments:
\begin{equation}
    \omega_i=\exp \left(-\epsilon \sum_{k=1}^{i-1} L_{p d e}^k(\theta)\right), \quad \text { for } i=2,3, \ldots, M.
\end{equation}
This equation implies that the weight for a given segment $i$ decreases exponentially with the sum of the PDE loss terms from all preceding segments. Here, $\epsilon$ is a scaling factor that controls the rate of exponential decay.

Finally, the weighted loss is calculated as:
\begin{equation}
    L_{p d e}(\theta)=\frac{1}{M} \sum_{i=1}^M \omega_i L_{p d e}^i(\theta).
\end{equation}
By incorporating dynamically updated weights $\omega_i$, this loss function ensures that segments with higher prior errors contribute more to the overall loss, guiding the model to improve its predictions in those segments.

\section{QPINNACLE: outline and key features}


This section provides an overview of the QPINNACLE repository's structure and key features. The codebase will be made publicly available upon publication. The implementation is primarily built using the PyTorch \cite{paszke2019pytorch} library and has been validated on a range of hardware platforms, including Apple silicon, x86 CPUs, and NVIDIA GPU servers equipped with A100, A6000, and L40s GPUs.

\subsection{Repository outline}

The repository is organized into six modules, each progressively more complex and introducing various convergence-enhancing techniques discussed in the previous section to address the challenges of training PINNs. Starting with a simple vanilla neural network in Module 1, we gradually incorporate enhancements to improve convergence. The impact of modifications to network architecture (RFF, activation functions), initialization (RWF), and training strategies (scheduler, Adam+LBFGS, loss balancing) is explored throughout each module. The contents and details of each module are detailed below:


\begin{enumerate}
    \item \textbf{Module-1: Vanilla PINNs}
    \begin{enumerate}
        \item 1D Diffusion equation
        \item 1D Unsteady diffusion equation
        \item Allen-Cahn equation
    \end{enumerate}

    \item \textbf{Module-2: Enhancing performance by altering network architecture and initialization}
    \begin{enumerate}
        \item Allen-Cahn equation with Random Fourier Features (RFF)
        \item Lid-driven cavity at Re=100 and 400 with Fourier Feature Features (RFF)
        \item Lid-driven cavity at Re=400 with RFF and RWF
        \item Allen-Cahn equation and Lid-driven cavity with periodic activation function
    \end{enumerate}

    \item \textbf{Module-3: Strict boundary conditions, loss balancing, and LBFGS optimizer switch}
    \begin{enumerate}
        \item Imposing strict periodic boundary conditions
        \begin{enumerate}
            \item Advection equation
            \item Allen-Cahn equation
        \end{enumerate}
        
        \item Adam to LBFGS optimizer switch
        \begin{enumerate}
            \item Allen-Cahn equation: loss history 
            \item Lid-driven cavity: loss history
        \end{enumerate}
        
        \item Lid-Driven cavity with `loss balancing' and `curriculum training'
        \begin{enumerate}
            \item Implementing loss Balancing
            \item Training lid-driven cavity at Re=1000 with curriculum training
        \end{enumerate}
    \end{enumerate}

    \item \textbf{Module-4: Shock-containing problems}
    \begin{enumerate}
        \item 1-D Burgers equation
        \item 1-D Euler equations: Sod and Lax shock Tube problem
    \end{enumerate}

    \item \textbf{Module-5: Multi-GPU training with Distributed Data Parallel (DDP) framework}
    \begin{enumerate}
        \item Multi-GPU implementation for Allen-Cahn equation
        \item Multi-GPU implementation for the lid-driven cavity problem
    \end{enumerate}
\end{enumerate}

\textcolor{blue}{\textit{Module-1:}} We begin with a step-by-step guide to building a basic vanilla PINNs neural network to solve the 1D steady-state diffusion equation. It includes defining the model, configuring optimization parameters, and training the network. Given the simplicity of the differential equation and boundary conditions in module 1(a) and 1(b), we demonstrate how easily the solution converges using a vanilla network. In section 1(c), we attempt to solve the Allen-Cahn equation, showing how the vanilla network struggles with this more complex system, underscoring the limitations of a basic vanilla PINNs network for training such differential equations.

\textcolor{blue}{\textit{Module-2:}} In this module, we incorporate Random Fourier Features (RFF) to overcome the training challenges encountered by the vanilla neural network in solving the Allen-Cahn equation in Module 1(c). After successfully applying RFF to the Allen-Cahn equation, we extend its use to the lid-driven cavity problem at Re=100 and 400. While RFF converges for Re=100, it struggles at Re=400, even with increased network size and training time, suggesting the need for further network enhancements. To address this, we include Random Weight Factorization (RWF), a network weight initialization technique, alongside RFF, resolving the issue. Additionally, in Module 2(d), we demonstrate the superior performance of the less commonly used periodic activation function in PINNs.

\textcolor{blue}{\textit{Module-3:}} In continuation of Module-2, we further enhance the training performance of PINNs through loss balancing, strict boundary conditions, and an Adam to LBFGS optimizer switch. Strict periodic boundary conditions are implemented to solve the advection and Allen-Cahn equations, effectively eliminating the boundary condition loss component from the overall loss calculation, thereby reducing the optimization burden on the network during training. Next, we demonstrate the rapid convergence achieved by switching from the Adam optimizer to LBFGS after partial convergence with Adam. Finally, we explore advanced training techniques, such as loss balancing and curriculum training, applied to the lid-driven cavity at Re=1000, highlighting how these strategies can address challenges in complex fluid dynamics simulations.

\textcolor{blue}{\textit{Module-4:}} This module builds on the techniques developed in the previous modules to tackle problems involving sharp discontinuities. Specifically, we solve the 1D Burgers' and the Euler equations. We emphasize that, beyond the techniques already covered, specialized treatment near discontinuities is necessary to ensure accurate solutions.

\textcolor{blue}{\textit{Module-5:}} In this module, apply the techniques from previous modules to blood flow in stenosis. We begin by demonstrating how to use sparse training data to accelerate convergence, addressing the challenges of limited data in medical simulations. Next, we investigate the impact of various activation functions on solution accuracy, comparing their effectiveness in computing the wall shear-stress profile. Finally, we perform an ablation study to assess the impact of RFF, RWF, and activation function on the solution convergence.

\textcolor{blue}{\textit{Module-6:}} The last module focuses on scaling PINNs training using the Distributed Data Parallel (DDP) framework for multi-GPU environments. We demonstrate the implementation of multi-GPU training for the Allen-Cahn equation and the lid-driven cavity problem, highlighting improvements in training efficiency and model performance.

The current structure and contents of this repository are subject to change with future updates. Users are encouraged to report any issues, bugs, or concerns regarding the clarity of the repository's contents. Since each notebook in this repository is tailored to specific problems, it does not serve as a general-purpose code for arbitrary applications. Users are encouraged to select and adapt the notebook that best fits their problem. Future updates will include a general-purpose PINNs code that can accept governing equations, geometry, hyperparameters, and training options directly.

\begin{itemize}
    \item \textit{Advanced convergence enhancing techniques:} The notebooks introduce techniques like RFF, RWF, and periodic activation functions for improved convergence, with guidance for advanced users.

    \item \textit{Optimal training strategies:} The repository employs training methods such as curriculum training, Adam+LBFGS, loss balancing, and learning rate scheduling for efficient convergence.

    \item \textit{Multi-GPU capability:} Supports multi-GPU training via Distributed Data Parallel (DDP) for handling large-scale problems.

    \item \textit{Simple code structure:} The code repository is intentionally designed to be beginner-friendly with easy code readability. We also prioritized the code structure to make it easier to add new modules and to enable faster, simpler future development. Accelerates the learning curve for the new students working in this domain. The simple code structure and limited memory allocation also make the code run efficiently.
\end{itemize}

The code repository is intentionally designed to be beginner-friendly and easy to read. The increasing levels of complexity from Module 1 to Module 6 allow for easy extension through the addition of new modules, as well as a faster learning curve for new students in the domain. The accompanying notebooks provide detailed explanations and provide easy experimentation with model parameters.


\section{Benchmark results of QPINNACLE}
We evaluate QPINNACLE on a set of benchmark PDEs chosen to represent distinct classes of numerical and optimization challenges encountered during PINN training. The linear advection example provides a controlled setting in which the exact solution corresponds to a pure translation, allowing errors to be interpreted directly in terms of phase shift and amplitude. It also enables the study of sensitivity to consistent coordinate scaling and the effects of enforcing periodicity. The Allen-Cahn equation describes stiff reaction-diffusion dynamics with thin moving interfaces. The inviscid Burgers equation exhibits gradient blow-up and shock formation, which challenge smooth neural approximations and can lead to overly diffused solutions unless sufficient resolution and appropriate boundary constraints are used. The steady lid-driven cavity problem evaluates coupled incompressible Navier-Stokes constraints over increasing Reynolds number, where boundary layers and secondary vortices yield multiscale spatial features and a more ill-conditioned optimization problem that can converge to incorrect solution branches. The stenosis benchmark addresses near-wall hemodynamics with sparse interior supervision. The Sod shock tube and the 2D Riemann problem test nonlinear conservation laws with discontinuities and wave interactions, focusing on whether the learned solution reproduces shock and contact structures without explicit shock-capturing operators. The 2D Maxwell pulse benchmark tests multidimensional wave propagation over longer time intervals in a periodic domain, where oscillatory solutions require accurate phase evolution and adequate representation of high-frequency components.

\subsection{Advection equation} \label{sec:advection}
The governing equation, together with the associated initial and boundary conditions for the linear advection problem, is given by
\begin{equation}
\begin{aligned}
    \text{PDE:} \quad & u_t + c u_x = 0, \quad && x \in [0, 2\pi], \; t \in [0, 1], \\
    \text{IC:} \quad & u(x, 0) = \sin(x), \quad && x \in [0, 2\pi], \\
    \text{BC:} \quad & u(0, t) = u(2\pi, t), \quad && t \in [0, 1], \\
    & u_x(0, t) = u_x(2\pi, t), \quad && t \in [0, 1].
\end{aligned}
\label{eqn:adv}
\end{equation}
The initial condition corresponds to a sinusoidal profile, and the PDE governs its periodic transport across the domain over the time interval \( T = 1 \), with wave speed \( c \). For moderate wave speeds, such as \( c = 10 \) and \( c = 40 \), the combination of Random Fourier Features (RFF) (Sec.~\ref{sec:RFF}) and loss balancing (Sec.~\ref{sec:loss_balancing}) is sufficient to obtain converged solutions. However, for larger wave speeds, such as \( c = 80 \), strict enforcement of periodic boundary conditions becomes necessary to maintain stability and accuracy. In the absence of such enforcement, the solution \( u(x,t) \) tends to decay toward zero.

The network's robustness was further evaluated across alternative spatial domains and collocation-point sampling methods. We consider the domain sizes of \( x \in [0,2] \) and \( x \in [-\pi,\pi] \). The corresponding initial conditions were chosen as \( u(x,0)=\sin(\pi x) \) and \( u(x,0)=-\sin(x) \), respectively, which are mathematically equivalent to the original problem (Eqn.~\ref{eqn:adv}). When the spatial domain is not expressed in terms of \( \pi \), consistency requires modifying the governing equation to \( u_t + \frac{c}{\pi} u_x = 0 \). If this rescaling is not applied, the network fails to learn the transported wave and instead converges to the trivial solution.

\begin{figure}[t!]
    \centering
    \includegraphics[width=\linewidth]{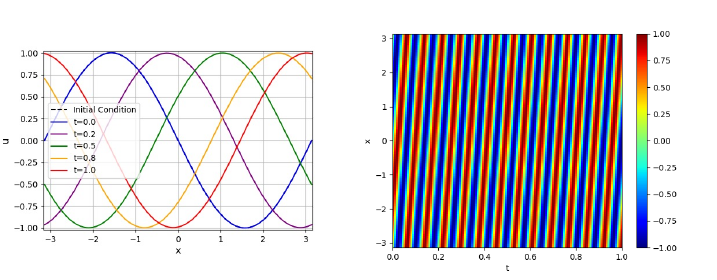}
    \caption{Temporal evolution of the solution for the advection/convection equation with the initial condition. The plot demonstrates the wave's propagation through the domain over time.}
    \label{fig:advection_diffusion}
\end{figure}

\begin{table}[h!]
\centering
\caption{Comparison of equation loss, initial condition (IC) loss, boundary condition (BC) loss, and relative $L^2$ error for different collocation point distributions. The analytical solution was used to compute the relative $L^2$ error.}
\begin{tabular}{lcccc}
\toprule
 & \textbf{Eqn. Loss} & \textbf{IC Loss} & \textbf{BC Loss} & \textbf{Rel. $L^2$ Error} \\
\midrule
Uniform, $x \in [0,2\pi]$ 
& $3.00 \times 10^{-4}$ 
& $6.20 \times 10^{-10}$ 
& $1.50 \times 10^{-13}$ 
& $1.15 \times 10^{-5}$ \\

LHS, $x \in [0,2\pi]$ 
& $3.22 \times 10^{-4}$ 
& $3.56 \times 10^{-10}$ 
& $1.58 \times 10^{-13}$ 
& $5.81 \times 10^{-7}$ \\

LHS, $x \in [0,2]$ 
& $2.96 \times 10^{-4}$ 
& $4.36 \times 10^{-10}$ 
& $1.02 \times 10^{-12}$ 
& $5.16 \times 10^{-7}$ \\

Wang et al.~\cite{wang2021understanding} 
& - 
& -
& -
& $5.84 \times 10^{-4}$ \\
\bottomrule
\end{tabular}
\label{tab:advection_results}
\end{table}

The effect of collocation point generation was investigated by comparing uniformly spaced grids with grids generated using the Latin Hypercube Sampling (LHS) method. Randomized resampling of collocation points at each epoch did not yield improved accuracy relative to fixed uniform grids. A quantitative comparison of the relative \( L_2 \) error for the different sampling strategies is provided in Table~\ref{tab:advection_results}. Among the tested configurations, the uniform grid achieved the lowest error. Using the present implementation, the resulting relative error is nearly two orders of magnitude smaller than that reported in the expert guide for the same number of collocation points and network architecture.

\subsection{Allen-Cahn equation}

The governing equation, initial condition, and boundary conditions for the Allen-Cahn problem are as follows:
\begin{equation}
\begin{aligned}
    \text{PDE:} \quad & u_t - 0.0001 u_{xx} + 5u^3 - 5u = 0, \quad && x \in [-1, 1], \, t \in [0, 1] \\
    \text{IC:} \quad & u(x, t=0) = x^2 \cos(\pi x), \quad && x \in [-1, 1] \\
    \text{BC:} \quad & u(x=-1, t) = u(x=1, t), \quad && t \in [0, 1] \\
    & u_x(x=-1, t) = u_x(x=1, t), \quad && t \in [0, 1]
\end{aligned}
\end{equation}

The Allen–Cahn equation (general form, \( u_t = \epsilon^2 u_{xx} - f(u) \)), is considered on a bounded spatial domain with appropriate initial and boundary conditions. The initial condition is prescribed as \( u(x,0) = x^2 \cos(\pi x) \), and the equation governs its temporal evolution, leading to the development and motion of diffuse interfaces, as shown in Fig.~\ref{fig:allen_cahn}. Since no closed-form analytical solution is available for this configuration, a high-resolution finite-difference method solution computed on a fine $1024^2$ grid is used as the reference. The results demonstrate the network’s ability to resolve the coupled diffusion and nonlinear reaction dynamics over time.

\begin{figure}[t!]
    \centering
    \includegraphics[width=\linewidth]{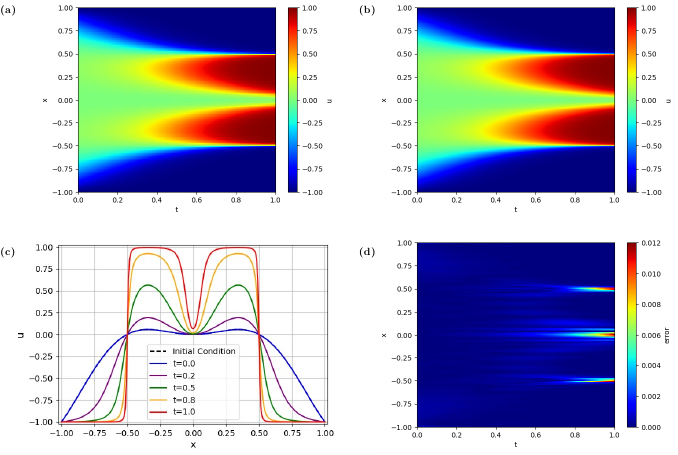}
    \caption{Temporal evolution of the solution for the Allen-Cahn equation. The plot demonstrates the phase separation and interface dynamics through the domain $x \in [-1, 1]$ over time.}
    \label{fig:allen_cahn}
\end{figure}

Similar to the exercise performed in Sec.~\ref{sec:advection}, the effect of collocation point generation was evaluated by comparing uniformly spaced grids with grids generated using the LHS. Table~\ref{table:resultsAC} presents a comparison of the relative \( L^2 \) error norms for the different sampling strategies. Both approaches yielded comparable accuracy when benchmarked against the numerical reference solution. The temporal causality mechanism was also assessed, but it did not produce a measurable reduction in error. Notably, the relative \( L^2 \) error reported by Wang et al.~\cite{wang2021understanding} is approximately two orders of magnitude smaller than the values obtained here, with discrepancies in the present results primarily concentrated near phase boundaries at later times. To further investigate resolution effects, an additional simulation with \(256 \times 256\) collocation points was conducted, yielding similar error levels.

\begin{table}[t!]
\centering
\begin{threeparttable}
\caption{Comparison of Eq. Loss, IC Loss, BC Loss, and Relative \( L^2 \) Error for different collocation point distributions for the advection equation.}
\begin{tabular}{lcccc}
\toprule
 & \textbf{Eq. Loss} & \textbf{IC Loss} & \textbf{BC Loss} & \textbf{Rel. \( L^2 \) Error} \\ 
\midrule
Uniform & \( 3.90 \times 10^{-7} \) & \( 2.86 \times 10^{-10} \) & \( 3.55 \times 10^{-11} \) & \( 2.58 \times 10^{-3} \) \\ 
1D LHS & \( 9.67 \times 10^{-7} \) & \( 3.29 \times 10^{-9} \) & \( 1.25 \times 10^{-10} \) & \( 3.97 \times 10^{-3} \) \\ 
2D LHS & \( 3.37 \times 10^{-7} \) & \( 2.17 \times 10^{-9} \) & \( 1.66 \times 10^{-10} \) & \( 1.19 \times 10^{-3} \) \\ 
Uniform, Temp. Causality & \( 1.26 \times 10^{-7} \) & \( 1.65 \times 10^{-10} \) & \( 3.67 \times 10^{-11} \) & \( 1.92 \times 10^{-3} \) \\ 
Uniform, 256 x 256 & \( 4.01 \times 10^{-7} \) & \( 6.99 \times 10^{-10} \) & \( 7.03 \times 10^{-11} \) & \( 2.61 \times 10^{-3} \) \\ 
Wang et al. (2023)\cite{wang2023expertsguide} & - & - & - & \( 5.37 \times 10^{-5} \) \\ 
\bottomrule
\label{table:resultsAC}
\end{tabular}
\end{threeparttable}

\end{table}

\subsection{Inviscid Burgers equation}
The one-dimensional inviscid Burgers equation is considered in the form
\begin{equation}
\begin{aligned}
    \text{PDE:} \quad & u_t + u u_x = 0, \quad && x \in [0,2], \; t \in [0,1], \\
    \text{IC:} \quad & u(x,0) = \sin(\pi x), \quad && x \in [0,2], \\
    \text{BC:} \quad & u(0,t) = u(2,t) = 0, \quad && t \in [0,1].
\end{aligned}
\end{equation}
This problem exhibits strong nonlinear wave steepening and the formation of a sharp gradient that evolves into a discontinuity, posing significant challenges for PINN optimization. To improve training stability, temporal causality-aware training~\cite{wang2024causal} was evaluated in combination with Random Fourier Features (Sec.~\ref{sec:RFF}), random weight factorization (Sec.~\ref{sec:RWF}), and loss balancing strategies (Sec.~\ref{sec:loss_balancing}). Strict boundary enforcement was applied in the spatial domain by constraining the network to satisfy identical values and gradients at both boundaries, thereby promoting periodic consistency while simultaneously adding a Dirichlet loss term to enforce the zero boundary values. In contrast, strict constraints were not imposed in time; instead, the temporal domain was rescaled to the interval \([-1,1]\). This normalization is consistent with the spatial treatment and aligns with the range of sine and cosine functions over a period of \(2\pi\).

As in the advection case, the spatial coordinate may be reformulated in terms of \( \pi \); however, defining the domain as \( [0,2\pi] \) or using other \( \pi \)-scaled coordinates resulted in slower convergence, demonstrating the sensitivity of PINNs to domain representation. On the domain \( [0,2] \), the network successfully resolved the steep gradient using a relatively coarse collocation grid. The temporal causality-aware training did not yield improved results at either domain size. The influence of collocation point generation was also evaluated by comparing uniform sampling with LHS. The LHS approach exhibited convergence difficulties, with the equation loss stagnating on a plateau. A comparison of the relative \( L^2 \) error norms for all tested configurations is presented in Table~\ref{tab:burgers}.

\begin{figure}[h!]
    \centering
    \includegraphics[width=\linewidth]{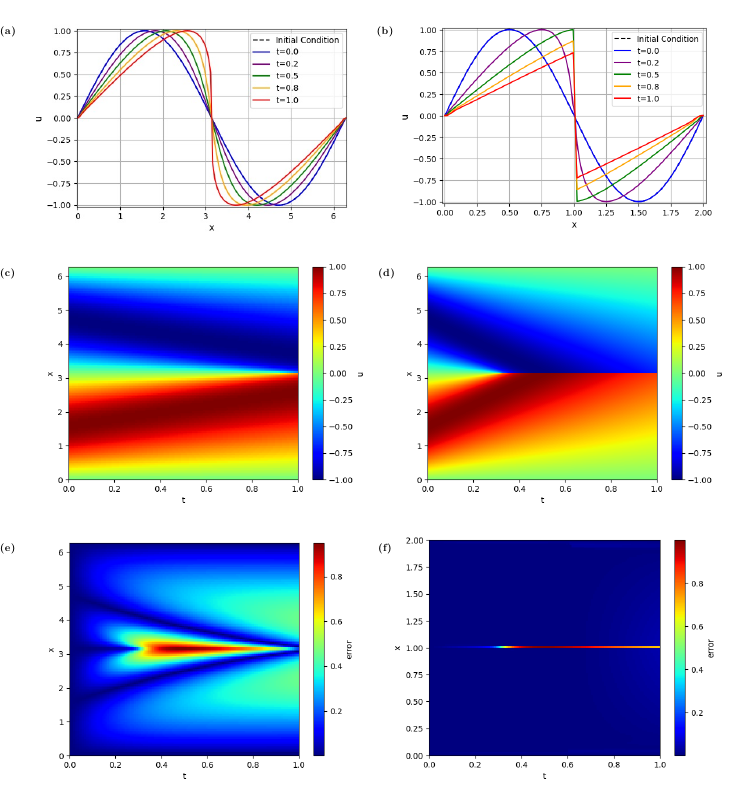}
    \caption{PINNs solution of Burger's equation studied using two domain sizes. 1st column: (a,c,e) $x \in [0,2\pi]$ and initial condition $u(x, t = 0) = \sin (\pi x)$, and second column (b,d,f) $x \in [0,2\pi]$ and initial condition $u(x, t = 0) = \sin (\pi x)$. 1st row: Solution at different time steps. 2nd row: Contour plot of solution evolution. 3rd row: The respective absolute error compared with a high-resolution Finite difference numerical solution.}
    \label{fig:burgers_2d}
\end{figure}





\begin{table}[t]
\centering
\caption{Comparison of equation loss, initial condition (IC) loss, boundary condition (BC) loss, and relative L\(^2\) error for different collocation point distributions. The FDM solution was used to calculate the relative L\(^2\) error.}
\begin{tabular}{lcccc}
\toprule
\textbf{Method} & \textbf{Eq. Loss} & \textbf{IC Loss} & \textbf{BC Loss} & \textbf{Rel. L\(^2\) Error} \\ \midrule
Uniform, $x \in [0, 2\pi]$ & $1.55 \times 10^{-7}$ & $1.36 \times 10^{-10}$ & $2.03 \times 10^{-11}$ & $4.22 \times 10^{-1}$ \\ 
Uniform, $x \in [0, 2]$ & $5.96 \times 10^{-7}$ & $4.10 \times 10^{-10}$ & $1.22 \times 10^{-11}$ & $1.05 \times 10^{-1}$ \\ 
1D LHS & - & - & - & - \\ 
2D LHS & $2.07 \times 10^{-1}$ & $2.08 \times 10^{-4}$ & $2.47 \times 10^{-11}$ & $3.10 \times 10^{-1}$ \\ 
Uniform, temporal causality & $1.17 \times 10^{-7}$ & $3.50 \times 10^{-10}$ & $3.14 \times 10^{-11}$ & $1.05 \times 10^{-1}$ \\ 
2D LHS, temporal causality & $7.27 \times 10^{-2}$ & $4.44 \times 10^{-5}$ & $6.31 \times 10^{-10}$ & $1.92 \times 10^{-1}$ \\ 
Uniform, 256 $\times$ 256 & $9.42 \times 10^{-7}$ & $2.10 \times 10^{-9}$ & $3.01 \times 10^{-12}$ & $1.34 \times 10^{-2}$ \\ 
Uniform, 256 $\times$ 256, temporal causality & $5.45 \times 10^{-7}$ & $1.47 \times 10^{-9}$ & $2.87 \times 10^{-11}$ & $7.29 \times 10^{-2}$ \\ \bottomrule
\end{tabular}
\label{tab:burgers}
\end{table}

\subsection{Lid-driven cavity}
We consider the two-dimensional incompressible lid-driven cavity flow governed by the nondimensional Navier-Stokes equations on the unit square $(x,y)\in[0,1]\times[0,1]$. 
\begin{align}
    u_x + v_y &= 0, \\
    u u_x + v u_y &= -p_x + \frac{1}{Re}\left(u_{xx} + u_{yy}\right), \\
    u v_x + v v_y &= -p_y + \frac{1}{Re}\left(v_{xx} + v_{yy}\right),
\end{align}
where $(u,v)$ denotes the velocity components and $p$ denotes pressure. The Reynolds number $Re$ controls the ratio of inertial to viscous effects. The steady formulation of the above governing equations is considered for the low Reynolds numbers considered in this study, $Re\in\{100,400,1000,3200\}$. The steady assumption eliminates temporal dependence in the governing equations. This reduces the training cost, but it can increase optimization difficulty as $Re$ increases. At higher $Re$, the cavity flow exhibits sharper gradients and stronger recirculation, and time-dependent dynamics may arise in the physical problem; nevertheless, the steady benchmark remains the standard reference for method comparison. No-slip boundary conditions are imposed on all walls, with a unit tangential lid velocity at the top boundary:
\begin{align}
    u(0,y) &= 0, \quad u(1,y)=0, \quad u(x,0)=0, \\
    u(x,1) &= 1, \\
    v(0,y) &= 0, \quad v(1,y)=0, \quad v(x,0)=0, \quad v(x,1)=0.
\end{align}
These constraints are enforced through penalty terms in the training objective.

\textbf{Re = [100,400].} For $Re=100$ and $Re=400$, the learned solutions match the Ghia et al.\ centerline profiles closely, as shown in Fig.~\ref{fig:compare_ldc_100_400}. Here, the profiles correspond to $u(x=0.5,y)$ and $v(x,y=0.5)$, which are standard diagnostics for the cavity benchmark. For $Re=1000$ and $Re=3200$, baseline training can converge to solutions that deviate from the benchmark profiles. To improve convergence at higher $Re$, we use curriculum training, increase the collocation resolution to $256\times256$, and sample collocation points using Latin hypercube sampling (LHS). Curriculum training progresses from lower to higher Reynolds numbers to provide a better initialization for the more challenging regimes. Increasing collocation density improves the spatial coverage of regions with strong gradients. LHS improves space-filling properties relative to purely grid-based or naive random sampling, thereby reducing sensitivity to point placement during optimization.

\textbf{Re = [1000, 3200].} For $Re=1000$, Fig.~\ref{fig:compare_ldc_1000} contrasts a baseline result and the improved result after applying these modifications. The improved model recovers the expected centerline profiles and yields a qualitatively consistent $u$-velocity field. For $Re=3200$, Fig.~\ref{fig:compare_ldc_3200} compares our centerline profiles with the DNS data of Ghia et al. \cite{ghia1982highre} and with the profiles reported by Wang et al.\ \cite{wang2021understanding}. The proposed training strategy yields closer agreement with the DNS data over most of the centerline, including near-wall regions where gradients are largest.

\begin{figure}[t!]
    \centering
    \includegraphics[width=\linewidth]{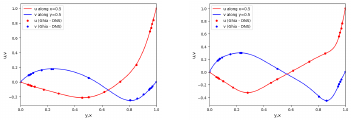}
    \caption{Lid-driven cavity (No curriculum training) centerline velocity profiles for $Re=100$ (left) and $Re=400$ (right). Solid lines denote the present predictions for $u(x=0.5,y)$ and $v(x,y=0.5)$; markers denote the DNS data of Ghia et al.\ \cite{ghia1982highre}.}
    \label{fig:compare_ldc_100_400}
\end{figure}

\begin{figure}[t!]
    \centering
    \includegraphics[width=\linewidth]{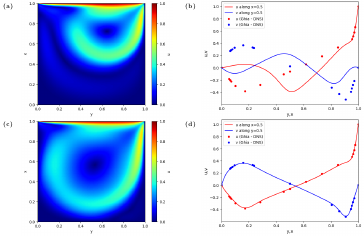}
    \caption{Lid-driven cavity results at $Re=1000$. Panels (a,b) show a baseline model: (a) contour of the horizontal velocity component $u$, and (b) centerline profiles compared with Ghia et al.\ \cite{ghia1982highre}. Panels (c,d) show the improved model obtained with curriculum training, $256\times256$ collocation points, and LHS: (c) contour of $u$, and (d) corresponding centerline profiles $u(x=0.5,y)$ and $v(x,y=0.5)$ compared with DNS.}
    \label{fig:compare_ldc_1000}
\end{figure}

\begin{figure}[t!]
    \centering
    \includegraphics[width=0.6\linewidth]{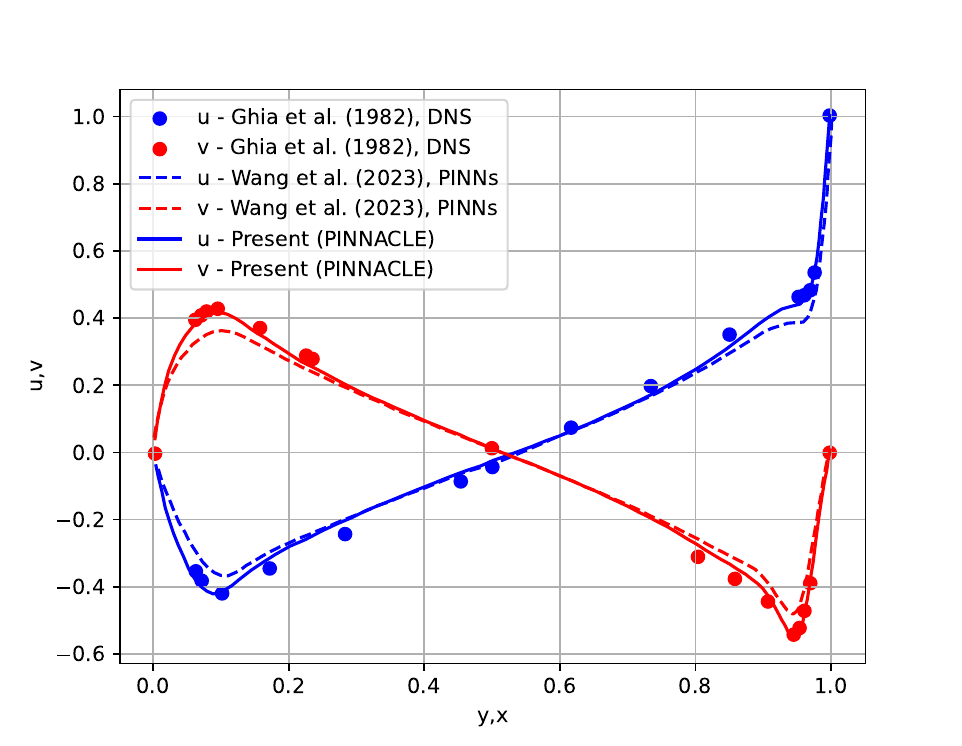}
    \caption{Lid-driven cavity centerline velocity profiles at $Re=3200$. Solid lines denote the present predictions for $u(x=0.5,y)$ and $v(x,y=0.5)$. DNS data from Ghia et al. \cite{ghia1982highre} are shown as markers. Dashed lines reproduce the profiles reported by Wang et al. \cite{wang2021understanding} for comparison.}
    \label{fig:compare_ldc_3200}
\end{figure}

\subsection{2D blood flow in a stenosis}
\label{sec:stenosis}

We next benchmark QPINNACLE on a steady, incompressible laminar flow through an idealized 2D stenosis, and validate the learned solution against a high-fidelity CFD reference. This benchmark is adapted from the near-wall hemodynamics study of Arzani \textit{et al.}~\cite{arzani2021nearwall}, and is representative of flows where clinically relevant wall quantities (in particular wall shear stress, WSS) must be inferred from sparse in-domain velocity observations using physics constraints~\cite{raissi2019pinns}.

The computational domain is a 2D channel of length $L=2$ and height $H=0.3$, with a symmetric stenosis centered at $x\approx 1$, stenosis length $L_s=0.3$, and a minimum throat opening of $H_t=0.075$ (Fig.~\ref{fig:stenosis_points}). We consider a steady flow at $\mathrm{Re}=150$ (based on $u_{\max}$ and $H$), with constant density $\rho=1$ and dynamic viscosity $\mu=10^{-3}$.
At the inlet ($x=0$), we prescribe a parabolic profile
\begin{equation}
u_{\mathrm{in}}(y) = 4u_{\max}\Big(\frac{y}{H}\Big)\Big(1-\frac{y}{H}\Big), \qquad v_{\mathrm{in}}(y)=0,
\end{equation}
with $u_{\max}=0.5$.
On solid walls, we enforce the no-slip condition $u=v=0$.
At the outlet ($x=L$) we impose homogeneous axial gradients,
\begin{equation}
\frac{\partial u}{\partial x}=\frac{\partial v}{\partial x}=\frac{\partial p}{\partial x}=0.
\end{equation}
The computational fluid dynamics (CFD) reference fields $(u_{\mathrm{CFD}},v_{\mathrm{CFD}},p_{\mathrm{CFD}})$ used for validation are taken from~\cite{arzani2021nearwall}.

\textit{PINN setup with sparse supervision.}
We approximate $(u,v,p)$ with a fully connected neural network $(u_\theta(x,y),v_\theta(x,y),p_\theta(x,y))$ trained by minimizing a composite objective
\begin{equation}
\mathcal{L} = \mathcal{L}_{\mathrm{PDE}} + \lambda_{\mathrm{bc}}\,\mathcal{L}_{\mathrm{BC}} + \lambda_{\mathrm{data}}\,\mathcal{L}_{\mathrm{data}},
\end{equation}
where $\mathcal{L}_{\mathrm{PDE}}$ is the mean-squared residual of the steady incompressible Navier-Stokes equations at interior collocation points, $\mathcal{L}_{\mathrm{BC}}$ enforces the wall and inlet/outlet constraints at boundary points, and $\mathcal{L}_{\mathrm{data}}$ penalizes mismatch to sparse interior velocity measurements (``supervision'' points). Spatial derivatives are computed via automatic differentiation, following the standard PINN formulation~\cite{raissi2019pinns}. To improve representation of sharp spatial variations, we embed $(x,y)$ using a Fourier feature mapping before the MLP~\cite{tancik2020fourierfeatures}, consistent with common PINN best practices~\cite{wang2023expertsguide}.

In this benchmark, we use:
\begin{itemize}
  \item Architecture: 5 hidden layers, 128 neurons per layer, outputs $(u,v,p)$.
  \item Input embedding: Fourier feature mapping with Gaussian bandwidth parameter $\sigma=10$~\cite{tancik2020fourierfeatures}.
  \item Optimizer: Adam~\cite{kingma2014adam}, learning rate $10^{-3}$, trained for $1.5\times 10^4$ epochs.
  \item Collocation and boundary sets: $N_{\Omega}=4\times 10^4$ interior points, $N_{\Gamma_w}=1500$ wall points, and $N_{\Gamma_{\mathrm{in}}}=N_{\Gamma_{\mathrm{out}}}=200$ on the inlet and outlet.
  \item Sparse supervision: $N_{\mathrm{data}}=5$ interior velocity samples $(u_{\mathrm{CFD}},v_{\mathrm{CFD}})$ placed in the stenosis and near-wall region (Fig.~\ref{fig:stenosis_points}).
  \item Loss weights: $\lambda_{\mathrm{bc}}=20$ and $\lambda_{\mathrm{data}}=1$.
\end{itemize}

\begin{figure}[t]
  \centering
  \includegraphics[width=\linewidth]{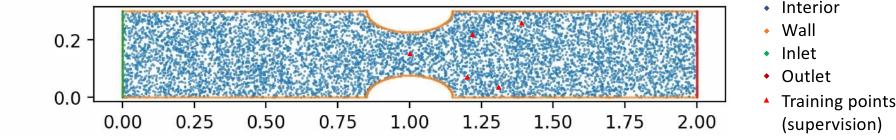}
  \caption{Collocation points sampled for the 2D stenosis benchmark.
  Interior collocation points (blue) enforce the PDE residual, boundary points enforce wall no-slip (orange) and inlet/outlet conditions (green/red), and five sparse interior velocity measurements (red triangles) provide supervision.
  This configuration matches the benchmark protocol used for near-wall hemodynamics reconstruction from sparse data~\cite{arzani2021nearwall}.}
  \label{fig:stenosis_points}
\end{figure}

To quantify the error beyond pointwise velocity agreement, we focus on WSS along the bottom wall downstream of the stenosis, a key near-wall diagnostic in~\cite{arzani2021nearwall}. For a Newtonian fluid, the streamwise WSS on a horizontal wall is
\begin{equation}
\tau_w(x) = \mu\left.\frac{\partial u_\theta}{\partial y}\right|_{y=0},
\end{equation}
computed directly from the learned velocity via automatic differentiation. We compare the normalized profile $\tau_w/\tau_{w,\max}$ against CFD on the interval $1.15 \le x \le 2$.
For a discrete set of wall evaluation points $\{x_i\}_{i=1}^N$, we report the mean absolute error (MAE) and mean relative error (MRE),
\begin{equation}
\mathrm{MAE}=\frac{1}{N}\sum_{i=1}^N \left| \hat{\tau}_i-\tau_i \right|,
\qquad
\mathrm{MRE}=\frac{1}{N}\sum_{i=1}^N \frac{\left| \hat{\tau}_i-\tau_i \right|}{\left|\tau_i\right|},
\end{equation}
where $\hat{\tau}_i$ is the PINN prediction and $\tau_i$ is the CFD reference (both using the same normalization).

Figure~\ref{fig:stenosis_benchmark} shows representative PINN predictions of the streamwise velocity field $u(x,y)$ and the corresponding absolute error relative to CFD, together with a quantitative WSS comparison.
Both activation choices recover the global flow features and the downstream recovery of the velocity field, while the largest discrepancies occur in the post-stenotic region, where the shear layer and recirculation dynamics are most demanding for sparse-data learning.
For the normalized bottom-wall WSS at $1.5\times 10^4$ epochs, we obtain MAE $\approx 3.96\times 10^{-2}$ and MRE $\approx 32.1\%$ (tanh), and MAE $\approx 3.58\times 10^{-2}$ and MRE $\approx 28.5\%$ (swish). Overall, the stenosis benchmark confirms that QPINNACLE can reproduce the CFD reference to a useful level using physics constraints and only five interior velocity measurements, with WSS trends captured at moderate relative error, consistent with the practical difficulty of near-wall inference emphasized in~\cite{arzani2021nearwall}.

\begin{figure}[t]
  \centering
  \includegraphics[width=\linewidth]{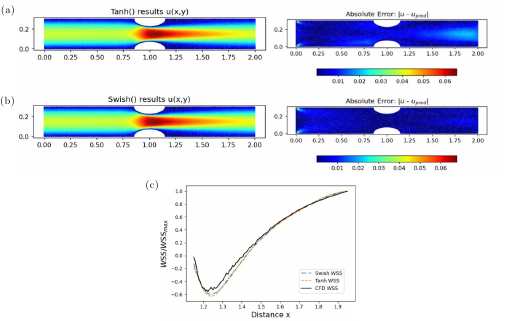}
  \caption{Stenosis benchmark validation against CFD.
  (a,b) Predicted streamwise velocity $u(x,y)$ and pointwise absolute error $\left|u_{\mathrm{CFD}}-u_{\theta}\right|$ over the full domain (shown for two activation functions to illustrate robustness of the benchmark outcome).
  (c) Normalized bottom-wall WSS profile $\tau_w/\tau_{w,\max}$ on $1.15 \le x \le 2$, comparing PINN predictions with the CFD reference.
  WSS is computed as $\tau_w(x)=\mu\,\partial u/\partial y|_{y=0}$ using automatic differentiation.}
  \label{fig:stenosis_benchmark}
\end{figure}

\subsection{Sod shock tube problem}
\label{sec:sod}

As a representative benchmark for nonlinear hyperbolic systems with discontinuities, we consider the one-dimensional Sod shock tube problem governed by the compressible Euler equations. This test is widely used to assess a method's ability to capture discontinuities without numerical oscillations and excessive dissipation \cite{sod1978shock,toro2009riemann}.

The conservative form of the one-dimensional Euler equations is
\begin{equation}
\frac{\partial}{\partial t}
\begin{pmatrix}
\rho \\
\rho u \\
\rho E
\end{pmatrix}
+
\frac{\partial}{\partial x}
\begin{pmatrix}
\rho u \\
\rho u^2 + p \\
u(\rho E+p)
\end{pmatrix}
= 0,
\end{equation}
where $\rho$ is the density, $u$ the velocity, $p$ the pressure, and $E=\frac{p}{(\gamma-1)\rho}+ \frac{u^2}{2}$ the total energy. An ideal gas equation of state is assumed. The initial condition consists of a Riemann problem with a discontinuity at $x=0.5$,
\begin{equation}
(\rho,u,p) =
\begin{cases}
(1.0,\,0.0,\,1.0), & x < 0.5, \\
(0.125,\,0.0,\,0.1), & x \ge 0.5,
\end{cases}
\end{equation}
and zero-gradient Neumann boundary conditions are enforced on all solution variables on either side of the computational domain.

The solution $(\rho,u,p)(x,t)$ is approximated using a fully connected neural network trained by minimizing the Euler residuals at interior collocation points.
A $128 \times 128$ set of space-time collocation points is used.
To mitigate spectral bias and improve representation of steep gradients, random Fourier feature embeddings are applied at the input layer.
Training employs gradient-based optimization with adaptive loss balancing, followed by a switch from Adam to L-BFGS for refinement, consistent with the overall QPINNACLE benchmark protocol \cite{raissi2019pinns,wang2023expertsguide}.

Figure~\ref{fig:sod_results} compares the PINN predictions against a reference finite-volume CFD solution at a fixed time.
The density, velocity, and pressure profiles closely match the reference solution, capturing the shock location, contact discontinuity, and rarefaction fan without spurious oscillations.
Minor deviations are confined to narrow regions around discontinuities, which is expected given the smooth neural representation and the absence of explicit shock-capturing terms.
Overall, this benchmark demonstrates that the PINN framework can recover the correct weak solution structure of the Euler equations under a standard shock tube configuration.

\begin{figure}[t]
  \centering
  \includegraphics[width=\linewidth]{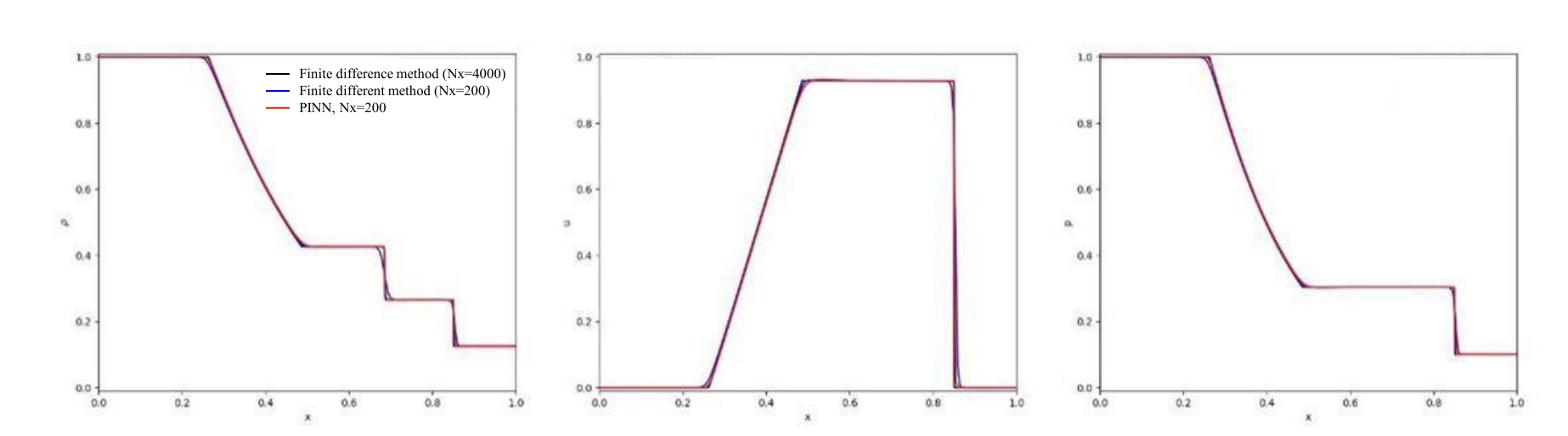}
  \caption{Sod shock tube benchmark.
  Comparison of PINN predictions with a finite-volume CFD reference for density $\rho$, velocity $u$, and pressure $p$ at a fixed time.
  The PINN solution reproduces the shock, contact discontinuity, and rarefaction fan with good agreement.
  Small discrepancies are localized near discontinuities, reflecting the smooth function approximation inherent to neural representations.}
  \label{fig:sod_results}
\end{figure}

\subsection{2D Riemann problem}
\label{sec:riemann2d_cfg6}

To further assess the capability of QPINNACLE on nonlinear hyperbolic systems with interacting discontinuities, we consider the two-dimensional compressible Euler equations under the sixth configuration of the classical 2D Riemann problem. This benchmark is known to generate complex shock interactions, slip lines, and vortex-type structures, and is commonly used to evaluate multidimensional solvers for conservation laws \cite{schulz1993classification,kurganov2001semidiscrete,toro2009riemann}.

The governing equations are the two-dimensional compressible Euler equations written in conservative form,
\begin{equation}
\frac{\partial}{\partial t}
\begin{pmatrix}
\rho \\
\rho u \\
\rho v \\
\rho E
\end{pmatrix}
+
\frac{\partial}{\partial x}
\begin{pmatrix}
\rho u \\
\rho u^2 + p \\
\rho u v \\
u(rho E+p)
\end{pmatrix}
+
\frac{\partial}{\partial y}
\begin{pmatrix}
\rho v \\
\rho u v \\
\rho v^2 + p \\
v(\rho E+p)
\end{pmatrix}
= 0,
\end{equation}
where the variable names follow from Sec.~\ref{sec:sod} with $(u,v)$ as $(x,y)$ velocity components, and $E=\frac{p}{(\gamma-1)\rho}+ \frac{u^2+v^2}{2}$ is the total energy. An ideal gas equation of state is assumed.

The computational domain is the unit square $(x,y)\in[0,1]\times[0,1]$, divided into four quadrants with piecewise constant initial states,
\begin{equation}
(\rho,u,v,p)=
\begin{cases}
(1.0,\,0.75,\,-0.5,\,1.0), & x>0.5,\ y>0.5,\\
(2.0,\,0.75,\,0.5,\,1.0),  & x<0.5,\ y>0.5,\\
(1.0,\,-0.75,\,0.5,\,1.0), & x<0.5,\ y<0.5,\\
(3.0,\,-0.75,\,-0.5,\,1.0),& x>0.5,\ y<0.5.
\end{cases}
\end{equation}
This configuration corresponds to configuration~6 in the classification of Schulz-Rinne \textit{et al.}~\cite{schulz1993classification}. Periodic boundary conditions are enforced on all domain boundaries, allowing the wave interactions to evolve without external forcing.

The solution fields $(\rho,u,v,p)(x,y,t)$ are approximated using a fully connected neural network trained by minimizing the Euler residuals at interior collocation points.
A uniform set of $64\times64\times64$ space-time collocation points is used.
Random Fourier feature embeddings are applied at the input layer to mitigate spectral bias, together with random weight factorization and adaptive loss balancing.
Strict periodic boundary conditions are enforced directly through the network input mapping.

Figure~\ref{fig:riemann2d_cfg6} compares the PINN predictions of density and $u$-velocity against a reference finite-volume CFD solution obtained with an MP5 scheme \cite{suresh1997accurate}.
The PINN solution reproduces the main multidimensional features of configuration~6, including the curved shock structures and the shear-driven roll-up near the quadrant interfaces.
As expected for a smooth neural representation, the largest discrepancies are localized near sharp discontinuities, while the overall wave topology and interaction patterns remain consistent with the CFD reference.
This benchmark demonstrates that QPINNACLE can recover the essential dynamics of a challenging two-dimensional Riemann problem without explicit shock-capturing terms.

\begin{figure}[t]
  \centering
  \includegraphics[width=0.75\linewidth]{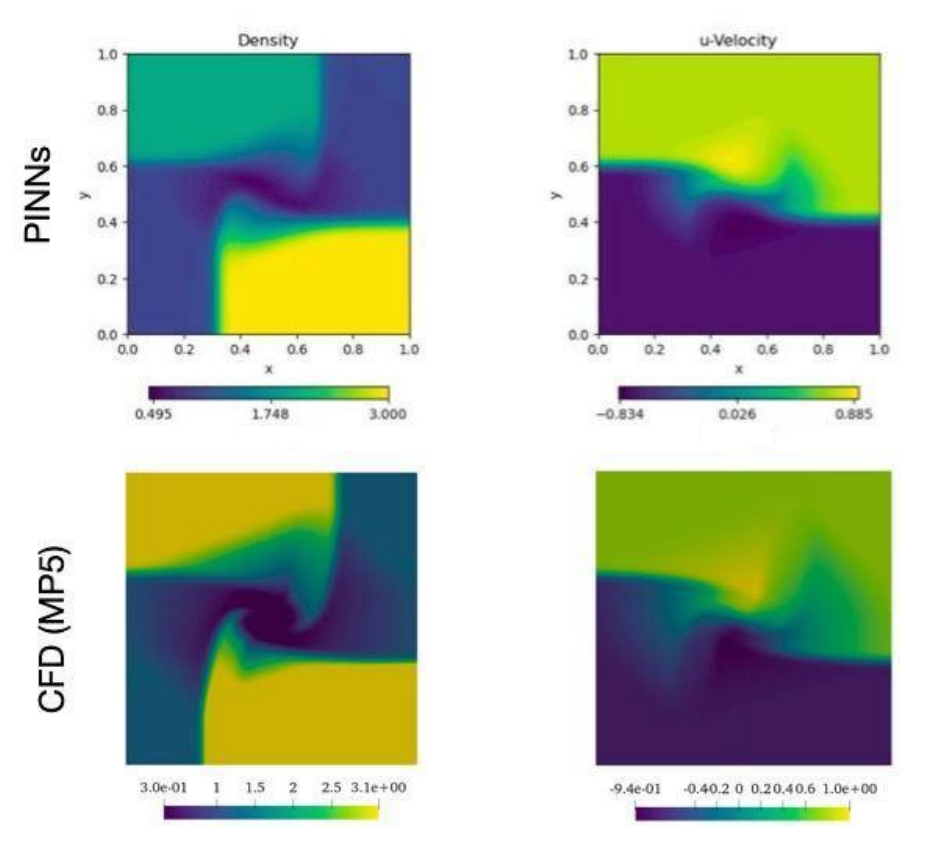}
  \caption{Two-dimensional Riemann problem, configuration~6.
  Comparison of PINN and CFD (computed using the MP5 reconstruction \cite{suresh1997accurate} and HLLC Riemann solver \cite{toro2009riemann}) solutions for density $\rho$ and $u$-velocity at a fixed time.
  The PINN results capture the dominant shock interactions and shear structures observed in the reference solution, with notable deviations confined to narrow regions near discontinuities.}
  \label{fig:riemann2d_cfg6}
\end{figure}

\subsection{Maxwell’s equations: 2-D Gaussian pulse in a periodic box}
\label{sec:maxwell_pulse_box}

We next demonstrate the capability of QPINNACLE on an unsteady two-dimensional electromagnetic wave propagation problem governed by Maxwell’s equations. Specifically, we consider the evolution of a two-dimensional Gaussian electric pulse in a spatially periodic domain, following the benchmark problem presented in Section 4.2 of \cite{wang2023expertsguide}. This test case is particularly well-suited for assessing training stability and solution quality in PINNs, as it involves coupled hyperbolic dynamics, multidimensional wave propagation, and long-time integration.

In two spatial dimensions and in the absence of sources, Maxwell’s equations for a transverse electric configuration can be written as
\begin{align}
\frac{\partial E_z}{\partial t} &=
\frac{1}{\varepsilon}
\left(
\frac{\partial H_y}{\partial x}
-
\frac{\partial H_x}{\partial y}
\right), \label{eq:maxwell_2d_ez}\\
\frac{\partial H_x}{\partial t} &=
-\frac{1}{\mu}
\frac{\partial E_z}{\partial y}, \label{eq:maxwell_2d_hx}\\
\frac{\partial H_y}{\partial t} &=
\frac{1}{\mu}
\frac{\partial E_z}{\partial x}. \label{eq:maxwell_2d_hy}
\end{align}
Throughout this study, free-space conditions are assumed, with normalized material parameters $\varepsilon = \mu = 1$.

The computational domain is defined as $(x,y)\in[-1,1]\times[-1,1]$, with the temporal interval $t\in[0,1.5]$. Periodicity is enforced in both spatial directions directly through the network architecture, thereby eliminating the need for an explicit boundary loss term. The initial conditions consist of a centered Gaussian pulse in the electric field and zero magnetic fields,
\begin{align}
E_z(x,y,0) &= \exp\!\left(-25(x^2+y^2)\right), \label{eq:ic_ez}\\
H_x(x,y,0) &= 0, \\
H_y(x,y,0) &= 0.
\end{align}
The PINN is trained by minimizing a composite loss comprising the initial condition loss and the physics residual corresponding to Eqs.~\eqref{eq:maxwell_2d_ez}-\eqref{eq:maxwell_2d_hy}.

To quantify the impact of architectural and training choices, an ablation study is conducted by selectively removing Random Fourier Features (RFF), strict spatio-temporal periodicity, and temporal causality enforcement. The relative $L_2$ error with respect to a high-resolution Padé reference solution is reported in Table~\ref{tab:maxwell_ablation}. The full configuration consistently yields the lowest error across network sizes, while removing individual components leads to systematic degradation in accuracy. In particular, the absence of RFF results in a marked increase in error, highlighting their importance for representing oscillatory electromagnetic fields.

\begin{table}[h!]
\centering
\caption{Ablation study for the 2-D Gaussian pulse in a periodic box.}
\label{tab:maxwell_ablation}
\begin{tabular}{cccccc}
\hline
RFF & Periodicity $(x,y)$ & Periodicity $(t)$ & Causality & $L_2$ error (64) & $L_2$ error (128) \\
\hline
$\checkmark$ & $\checkmark$ & $\checkmark$ & $\checkmark$ & 0.08471 & 0.04218 \\
$\checkmark$ & $\checkmark$ & $\checkmark$ & $\times$     & 0.07853 & 0.04272 \\
$\checkmark$ & $\checkmark$ & $\times$     & $\times$     & 0.11841 & 0.04895 \\
$\times$     & $\checkmark$ & $\checkmark$ & $\checkmark$ & 0.19232 & 0.09558 \\
$\checkmark$ & $\times$     & $\times$     & $\times$     & 0.51355 & 0.49885 \\
\hline
\end{tabular}
\end{table}

\begin{figure}[t]
\centering
\includegraphics[width=\linewidth]{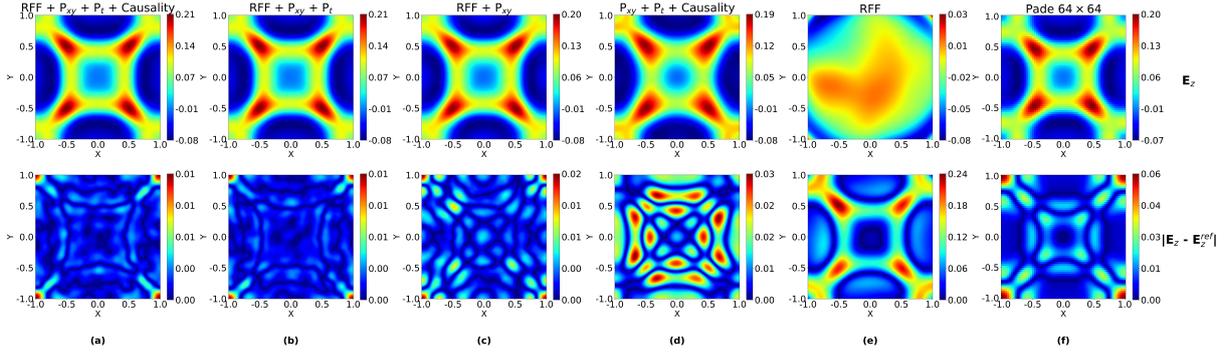}
\caption{Electric field $E_z$ (top row) and absolute error (bottom row) at $t=1.5$ for the 2-D Gaussian pulse in a periodic box, computed using different PINN configurations and a Padé reference solution.}
\label{fig:maxwell_pulse_box}
\end{figure}

Figure~\ref{fig:maxwell_pulse_box} shows the spatial distribution of the electric field $E_z$ at $t=1.5$ together with the corresponding absolute error for different PINN configurations, ordered from best to worst performance. The full model and the variant without causality enforcement produce visually indistinguishable field solutions with low and spatially uniform error. As additional components are removed, the solution quality deteriorates, with pronounced wave-front distortion and loss of symmetry. For reference, the error obtained using a low-resolution Padé scheme is also shown, demonstrating that poorly configured PINNs can perform comparably to under-resolved classical solvers.

This test case highlights the sensitivity of PINN training for unsteady Maxwell problems to architectural choices. While strict periodicity and temporal causality improve robustness, Random Fourier Features play a dominant role in accurately capturing the spatial structure of propagating electromagnetic waves.

\subsection{QPINNs - Quantum physics-informed neural networks}
\label{subsec:qpinns}

\subsubsection{Solving Maxwell's equations using QPINNs}
We assess how replacing a classical component of the PINN with a parametrized quantum circuit (PQC) affects training dynamics and solution quality for the unsteady two-dimensional Maxwell system introduced in Subsec.~\ref{sec:maxwell_pulse_box}. The classical PINN serves as the baseline.

The QPINN retains the standard physics-informed training procedure (PDE residual and initial condition losses), but replaces the penultimate classical layer with a PQC head. The classical trunk maps $(x,y,t)$ to a latent vector, which is embedded into the PQC via angle encoding. Expectation values of Pauli-$Z$ measurements are then mapped through a final linear layer to the outputs $(E_z, H_x, H_y)$.

A key quantum-specific design choice is the mapping of latent activations $a\in[-1,1]$ to rotation angles. Under single-qubit $R_x(\theta)$ embedding with Pauli-$Z$ measurement, the model output is trigonometric in $\theta$, and the scaling controls both frequency representation and measurement distribution.

We consider the following scalings:
\begin{equation}
\begin{aligned}
    scale_{none}(a) &= a, \quad 
    scale_{\pi}(a) = a\pi, \quad 
    scale_{bias}(a) = \tfrac{a+1}{2}\pi, \\
    scale_{asin}(a) &= \arcsin(a)+\tfrac{\pi}{2}, \quad 
    scale_{acos}(a) = \arccos(a)
\end{aligned}
\end{equation}
These mappings differ in range and in the distribution of angles over the Bloch sphere.

We evaluate the impact of PQC architecture by comparing several ansatz families, including mid-depth entangling layers (Basic and Strongly Entangling), cross-mesh constructions, and a no-entanglement baseline. All ans\"atze share a common structure: angle embedding, $L$ layers of single-qubit rotations, and a fixed entangling pattern, followed by Pauli-$Z$ measurements. Differences in performance, therefore, reflect the interaction structure induced by entanglement rather than changes in the surrounding classical network.

\paragraph{Training instability and energy regularization.}
A key distinction between QPINN and classical PINN behavior is the emergence of a failure mode in vacuum simulations. Many QPINN runs converge to a trivial solution that matches the initial condition but decays to near-zero fields for $t>0$. We refer to this behavior as a black-hole (BH) loss landscape.

To mitigate this effect, we introduce an additional penalty based on the Poynting theorem, enforcing global energy conservation. This term is critical for stable QPINN training. With it, collapse is avoided and physically meaningful solutions are recovered. Without it, convergence often fails despite low residual loss. In contrast, adding the same term is detrimental to the optimization of classical PINNs.

\paragraph{Ablation study.}
We perform a systematic ablation over three components: angle scaling, ansatz choice, and inclusion of the energy conservation term. Each configuration is evaluated across multiple runs, yielding a total of 300 experiments (Fig.~\ref{fig:qpinn_results}).

Three main observations emerge:
\begin{enumerate}
\item \textbf{Input scaling:} Poor scaling significantly degrades performance. For example, $scale_{\pi}$ yields over 60\% higher relative $L_2$ error compared to $scale_{asin}$, indicating that encoding geometry strongly influences model behavior.
\item \textbf{Ansatz choice:} Mid-depth entangling architectures provide the best accuracy, while no-entanglement and certain connectivity patterns underperform.
\item \textbf{Energy regularization:} With the energy constraint, several QPINN configurations outperform the classical baseline while using fewer parameters (approximately 19\% reduction). The best configuration improves the relative $L_2$ error by up to $\sim 19\%$.
\end{enumerate}

\begin{figure}
    \centering
    \includegraphics[width=\linewidth]{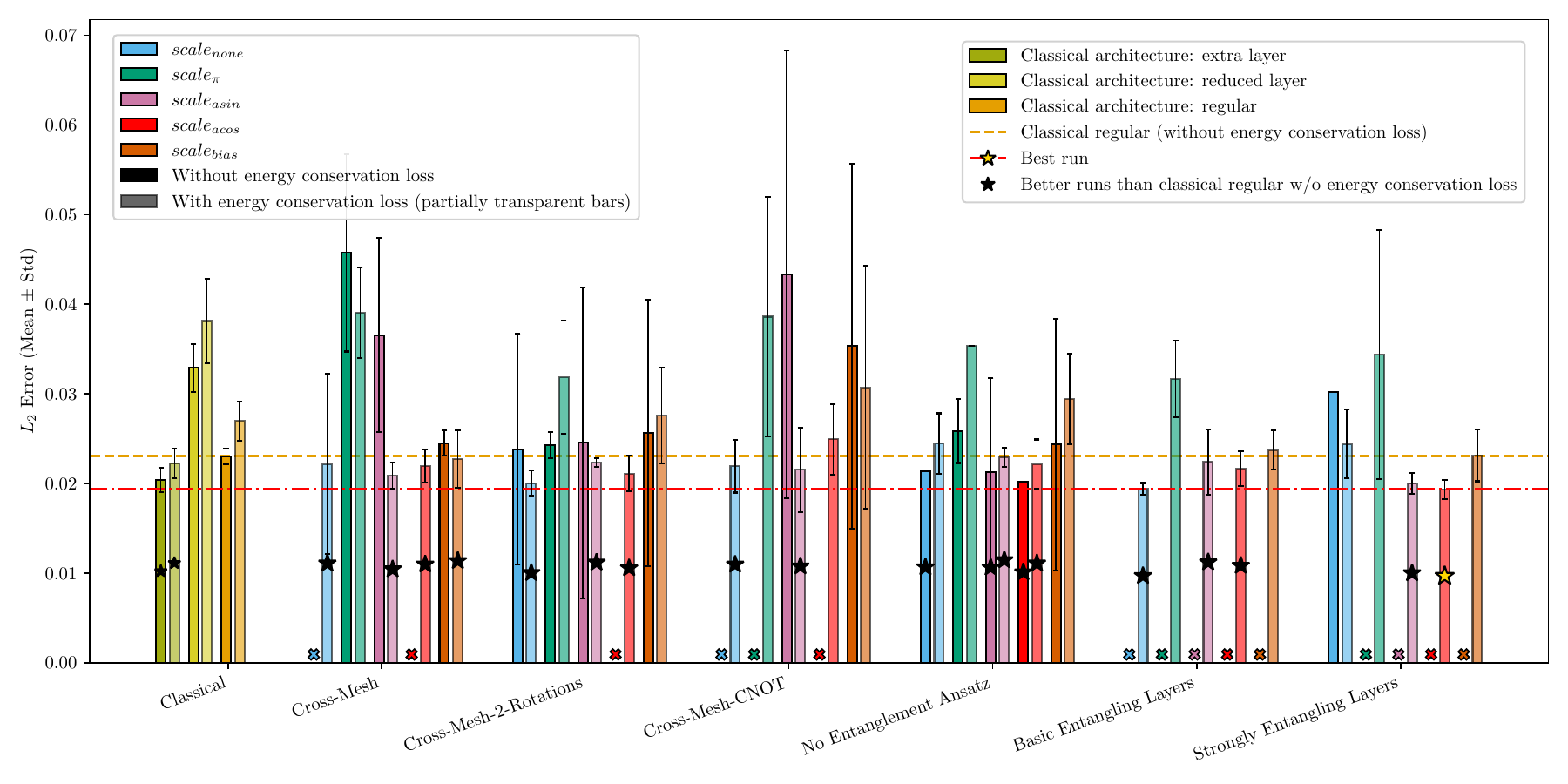}
    \caption{$L_2$ errors for all combinations of ansatz, angle scaling, and energy conservation loss.}
    \label{fig:qpinn_results}
\end{figure}

These observations suggest that QPINN performance on unsteady 2D Maxwell dynamics is shaped primarily by whether global conservation is enforced to mitigate BH. It is also affected by the angle-encoding scaling, which determines the effective harmonic basis observed in Pauli-$Z$ measurements, as well as by the ansatz-induced interaction structure. This empirical behavior comes with a substantial computational overhead. Each evaluation of the PDE residual requires repeated circuit executions due to parameter-shift differentiation, which quickly dominates the training cost. We formalize this scaling precisely in the next subsection.

\subsubsection{Circuit-evaluation complexity}
The computational burden observed in the previous experiments is driven by the cost of parameter-shift differentiation. We now formalize how this cost scales with the embedding dimension, PDE order, and number of trainable parameters.

\begin{theorem}[Parameter-shift circuit-evaluation complexity for QPINN PDE losses]
Fix $\mathbf{x}\in\mathbb{R}^d$ and consider a scalar quantum model output
\[
f(\mathbf{x};\theta)=\langle O\rangle_{U(\boldsymbol{\phi}(\mathbf{x}),\theta)},
\qquad \boldsymbol{\phi}(\mathbf{x})\in\mathbb{R}^{Q},\quad \theta\in\mathbb{R}^{P},
\]
where $\boldsymbol{\phi}$ is an angle-embedding map (one gate parameter per qubit) and $\theta$ are trainable ansatz
parameters.

Assume the following.
\begin{enumerate}
\item[(A1)] Shift rule: Each circuit parameter (each component of $\boldsymbol{\phi}$ and $\theta$) admits the
standard parameter-shift rule for first derivatives. Higher-order mixed partials are computed by iterated
parameter-shift, and circuit evaluations are not reused across distinct mixed partials.
\item[(A2)] PDE order: A PDE residual $\mathcal{R}(\mathbf{x};\theta)$ depends on coordinate derivatives of $f$
up to total order $K$.
\item[(A3)] Reduction to $\boldsymbol{\phi}$-partials: There exist finite index sets
$\mathcal{A}_k\subseteq\{\alpha\in\mathbb{N}_0^{Q}:\ |\alpha|=k\}$, $k=1,\dots,K$, such that for fixed $\mathbf{x}$
\[
\mathcal{R}(\mathbf{x};\theta)=\Psi\!\left(\mathbf{x},\theta,\{\partial_{\phi}^{\alpha} f(\mathbf{x};\theta)\}_{\alpha\in\mathcal{A}_1\cup\cdots\cup\mathcal{A}_K}\right),
\]
for some (classical) function $\Psi$. Define $S_k:=|\mathcal{A}_k|$.
\end{enumerate}
Then, the number of circuit evaluations required to compute the pointwise loss value
$L(\mathbf{x};\theta):=\ell(\mathcal{R}(\mathbf{x};\theta))$ is
\begin{equation}
N_{\mathrm{loss}}(\mathbf{x})
\;=\;
1+\sum_{k=1}^{K} S_k\,2^{k}.
\label{eq:Nloss}
\end{equation}
Moreover, if $\nabla_\theta L(\mathbf{x};\theta)$ is computed by parameter-shift by evaluating the full loss at
$\theta_p\pm s$ for each $p\in\{1,\dots,P\}$, then the number of circuit evaluations required for one update step at
$\mathbf{x}$ (loss plus all $P$ trainable gradients) is
\begin{equation}
N_{\mathrm{step}}(\mathbf{x})
\;=\;
(1+2P)\,N_{\mathrm{loss}}(\mathbf{x}).
\label{eq:Nstep}
\end{equation}
Finally, for each $k\le K$,
\begin{equation}
0\le S_k \le \binom{Q+k-1}{k}.
\label{eq:Sk_bound}
\end{equation}
\end{theorem}

\begin{proof}
The forward value $f(\mathbf{x};\theta)$ requires one circuit evaluation, yielding the leading $1$ in
\eqref{eq:Nloss}. Fix $k\in\{1,\dots,K\}$ and $\alpha\in\mathcal{A}_k$ with $|\alpha|=k$. Under (A1),
each differentiation via parameter-shift doubles the number of required circuit evaluations; hence
$\partial_\phi^\alpha f(\mathbf{x};\theta)$ requires $2^k$ circuit evaluations. Since $L(\mathbf{x};\theta)$ depends on
the collection $\{\partial_\phi^\alpha f(\mathbf{x};\theta)\}_{\alpha\in\mathcal{A}_1\cup\cdots\cup\mathcal{A}_K}$
by (A3), the loss value is obtained with
$1+\sum_{k=1}^K |\mathcal{A}_k|\,2^k = 1+\sum_{k=1}^K S_k\,2^k$ circuit evaluations, proving \eqref{eq:Nloss}.
(Any coordinate derivatives required by the PDE are incorporated into $\Psi$ through classical chain-rule coefficients
involving derivatives of $\boldsymbol{\phi}(\mathbf{x})$, and thus do not introduce additional circuit evaluations.)

For the update step, parameter-shift computes each component $\partial L/\partial\theta_p$ from two loss evaluations at $\theta_p\pm s$, each costing $N_{\mathrm{loss}}(\mathbf{x})$ circuit evaluations. Summing over $p=1,\dots,P$ and including the unshifted loss yields \eqref{eq:Nstep}. The bound \eqref{eq:Sk_bound} follows because the number of distinct mixed partial derivatives of total order $k$ in $Q$ variables is the number of multi-indices $\alpha\in\mathbb{N}_0^Q$ with $|\alpha|=k$, which equals the multiset coefficient $\binom{Q+k-1}{k}$.
\end{proof}

\paragraph{Example (2D unsteady Maxwell, first-order PDE; $Q=7$, $P=84$).}
Revisiting the 2D unsteady Maxwell setup from Section~\ref{subsec:qpinns}, and using the same PQC architecture as in the considered ans\"atze, the residual involves only first-order coordinate derivatives (in $t,x,y$), hence
$K=1$. If the embedding angles $\{\phi_j(x,y,t)\}_{j=1}^Q$ depend on the coordinates, then all required coordinate
derivatives of $f$ can be expressed using the embedding-angle gradient $\{\partial f/\partial\phi_j\}_{j=1}^{Q}$ via
the chain rule. Consequently one may take $\mathcal{A}_1=\{e_1,\dots,e_Q\}$ (canonical basis multi-indices), so
$S_1=Q$ and \eqref{eq:Nloss} gives
\[
N_{\mathrm{loss}}=1+S_1\,2^1 = 1+2Q.
\]
With $Q=7$, $N_{\mathrm{loss}}=15$. Using \eqref{eq:Nstep} with $P=84$,
\[
N_{\mathrm{step}}=(1+2P)\,N_{\mathrm{loss}}=(1+168)\cdot 15 = 2535.
\]

This result highlights a fundamental limitation of QPINNs. The exponential dependence on derivative order and the linear dependence on the number of trainable parameters lead to a rapid increase in computational cost, even for moderately sized problems. In practice, this scaling restricts QPINNs to low-dimensional settings or requires careful architectural and training design.


\section{GPU acceleration}

This section reports QPINNACLE's performance on single- and multi-GPU systems. We quantify GPU versus CPU speedup, multi-GPU scaling behavior, VRAM usage, and the distribution of wall-clock time across major computational components. Results are shown for NVIDIA A100, A6000, and L40S GPUs.

\subsection{GPU vs. CPU speedup comparison}

All experiments show a clear reduction in wall-clock time when moving from CPU to GPU execution. GPU execution enables larger batch sizes and higher collocation point counts within fixed memory limits, which is not feasible on CPUs for the tested configurations. Detailed CPU-versus-GPU timing data are provided in the supplementary material.

\subsection{Multi-GPU performance}\label{subsec:gpu_perf}

QPINNACLE was evaluated on multi-GPU systems using the lid-driven cavity problem as a representative benchmark. Experiments were conducted on clusters comprising eight NVIDIA L40S GPUs, two A100 GPUs, and two A6000 GPUs. All multi-GPU runs employed Distributed Data Parallel (DDP) training.

\subsubsection{Wall-clock time and VRAM occupancy}

Figure~\ref{fig:gpu_scaling} reports the wall-clock time required for 100 training epochs and the per-GPU VRAM usage for a fixed problem size of 40{,}000 collocation points. Scaling from one to two GPUs yields near-linear speedup across all hardware configurations, with this behavior maintained up to four GPUs. Beyond this point, performance begins to saturate, and a slight increase in wall-clock time is observed on the L40S system. This degradation is attributed to the communication and synchronization overhead associated with gradient aggregation.

In contrast, memory usage scales favorably. The per-GPU VRAM footprint decreases as the number of devices increases, with diminishing returns beyond four GPUs. On the L40S system, scaling to eight GPUs reduces per-device memory usage to approximately one-eighth of the single-GPU requirement, thereby enabling substantially larger collocation sets before memory exhaustion.

\begin{figure}[t!]
    \centering
    \begin{subfigure}{0.49\linewidth}
        \centering
        \includegraphics[width=\linewidth]{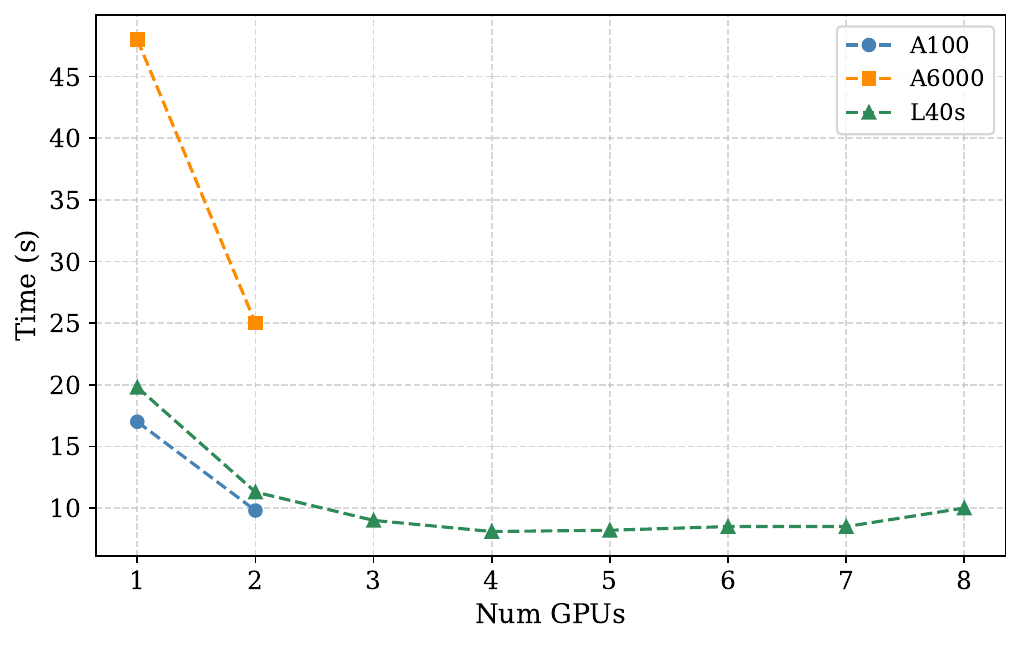}
        \caption{Wall-clock time per trial.}
    \end{subfigure}\hfill
    \begin{subfigure}{0.49\linewidth}
        \centering
        \includegraphics[width=\linewidth]{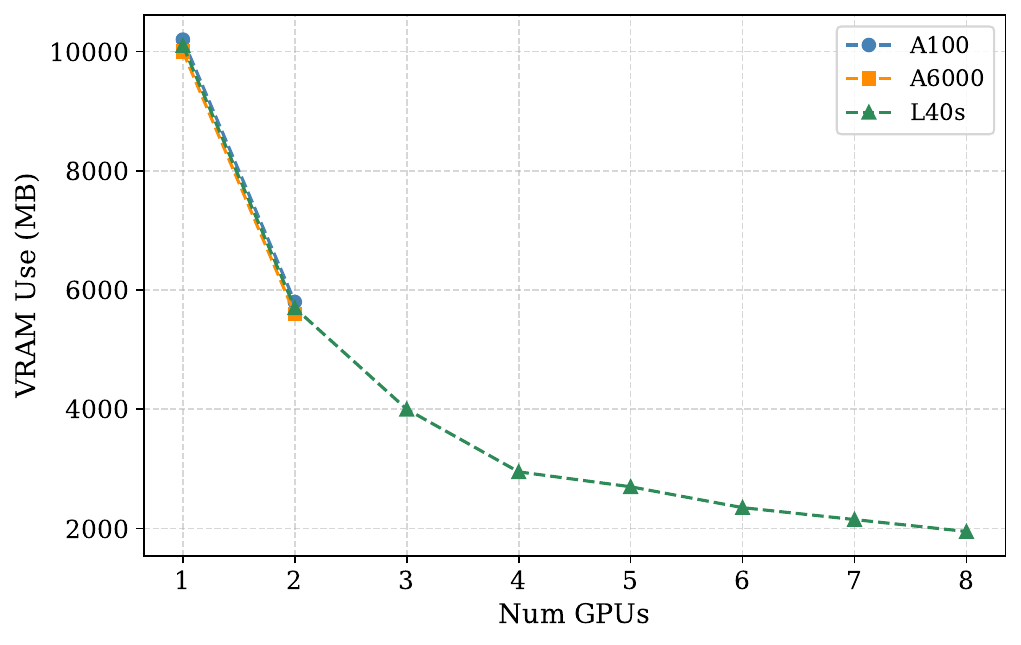}
        \caption{Per-GPU VRAM usage.}
    \end{subfigure}
    \caption{GPU scaling results for 100 epochs with 40{,}000 collocation points across A100, A6000, and L40s hardware.}
    \label{fig:gpu_scaling}
\end{figure}

\subsubsection{Memory capacity and runtime breakdown}

Figure~\ref{fig:capacity_profile} shows the maximum number of collocation points supported on the L40S cluster prior to memory exhaustion, together with a per-epoch runtime breakdown. The maximum achievable collocation count scales approximately linearly with the number of GPUs, enabling higher-resolution steady simulations and unsteady cases that exceed single-device memory limits.

The runtime profile was measured on the eight-GPU L40s system with 40{,}000 collocation points. Equation loss evaluation accounts for approximately 15\% of total runtime ($\sim$0.2 seconds per epoch), while boundary condition evaluation contributes around 6\% ($\sim$0.1 seconds). Inter-device synchronization represents roughly 2\% of the total runtime. Although modest in absolute terms, this overhead becomes a limiting factor for strong scaling, explaining the observed deviation from ideal speedup at higher GPU counts.

\begin{figure}[t!]
    \centering
    \includegraphics[width=\linewidth]{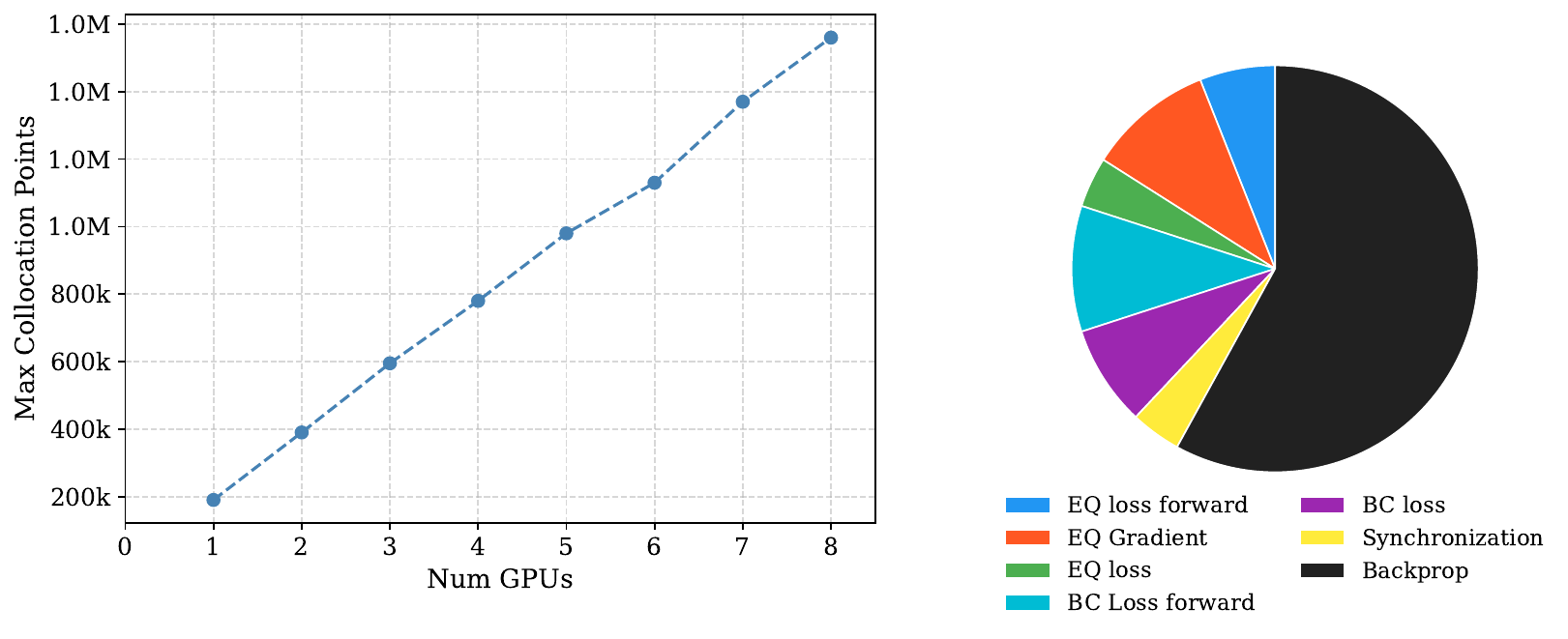}
    \caption{Memory capacity scaling and per-epoch runtime breakdown on the eight-GPU L40S system.}
    \label{fig:capacity_profile}
\end{figure}

\subsubsection{Limitations of Distributed Data Parallel}
DDP introduces two primary limitations in the present setting. First, communication overhead associated with gradient synchronization limits strong scaling efficiency as the number of GPUs increases, as reflected in Figure~\ref{fig:gpu_scaling}. While memory scaling remains favorable, runtime improvements diminish beyond moderate device counts. Second, DDP execution requires a multi-process launch configuration, which is not compatible with standard Jupyter notebook environments. Consequently, multi-GPU training must be executed using standalone Python scripts.

\textit{TorQ} evaluates parameterized quantum circuits (PQCs) as fully differentiable PyTorch programs on a single GPU, as shown in the flow chart in Figure~\ref{fig:torq_flowchart}.
For a batch of inputs $x \in \mathbb{R}^{B\times n}$, the simulator represents the quantum state as a complex batched
statevector $\psi \in \mathbb{C}^{B\times 2^n}$ and executes the circuit using dense linear algebra on CUDA.
Angle embedding is computed by generating per-qubit rotation vectors (for the default $\lvert 0\ldots0\rangle$
initialization) and forming their batched tensor product via a single \texttt{torch.einsum} contraction to obtain
$\psi(x)$ directly on GPU.
Each variational layer constructs a dense $2^n\times 2^n$ operator $U_\ell$ by (i) building a dense rotation wall using
Kronecker products of single-qubit gates (implemented via \texttt{einsum}) and (ii) composing it with a dense entangler
matrix $Ent$ (e.g., a precomputed CNOT ladder) using chained \texttt{torch.matmul}.
State evolution is then performed as one dense GEMM (General Matrix Multiplication) per layer, $\psi \leftarrow U_\ell\psi$.
Finally, analytic expectation values $\langle Z\rangle$ (Pauli-Z) are computed without sampling by reshaping probabilities and
contracting with per-qubit $Z$ axes using \texttt{einsum}; gradients are obtained by standard \texttt{torch.autograd}
backpropagation through the same CUDA tensor operations.

\textit{Comparison against PennyLane.}
To evaluate \textit{TorQ}, benchmark tests were conducted against several PennyLane's simulators. \textit{TorQ} was found to be over 50$\times$ faster than using PennyLane's \texttt{default.qubit} simulator with Torch integration (see Figure~\ref{fig:torq_benchmark}). Qiskit runs were slower than PennyLane's. \textit{TorQ} 
also had significantly better memory usage: the largest grid of collocation points that could be executed without memory overflow was $87^3$ on \textit{TorQ}, versus $43^3$ on PennyLane's \texttt{default.qubit}. 

\textit{Limitations and scaling.}
The current implementation explicitly materializes dense operators, so memory scales as
$\mathcal{O}(B\,2^n)$ for the batched statevector and $\mathcal{O}(4^n)$ for each dense $U_\ell$ (and similarly for
precomputed entanglers such as $Ent$).
Dense matrix multiplications and scales dominate runtime as $\mathcal{O}(B\,4^n\,L)$ for $L$ layers, which limits
practical simulations to moderate qubit counts and motivates choosing $n$ and $B$ to fit GPU memory.
\textit{TorQ} targets ideal statevector simulation with analytic (shot-free) expectation values; hardware noise models and shot-based sampling are not part of the current execution path.
Moreover, execution is single-GPU (no distributed multi-GPU parallelism, discussed in Subsec.~\ref{subsec:gpu_perf}), and backpropagation through long circuits can
increase activation/graph memory pressure, potentially requiring smaller batches or gradient checkpointing in future work.

\begin{figure}[t!]
\centering
\includegraphics[width=1.0\linewidth]{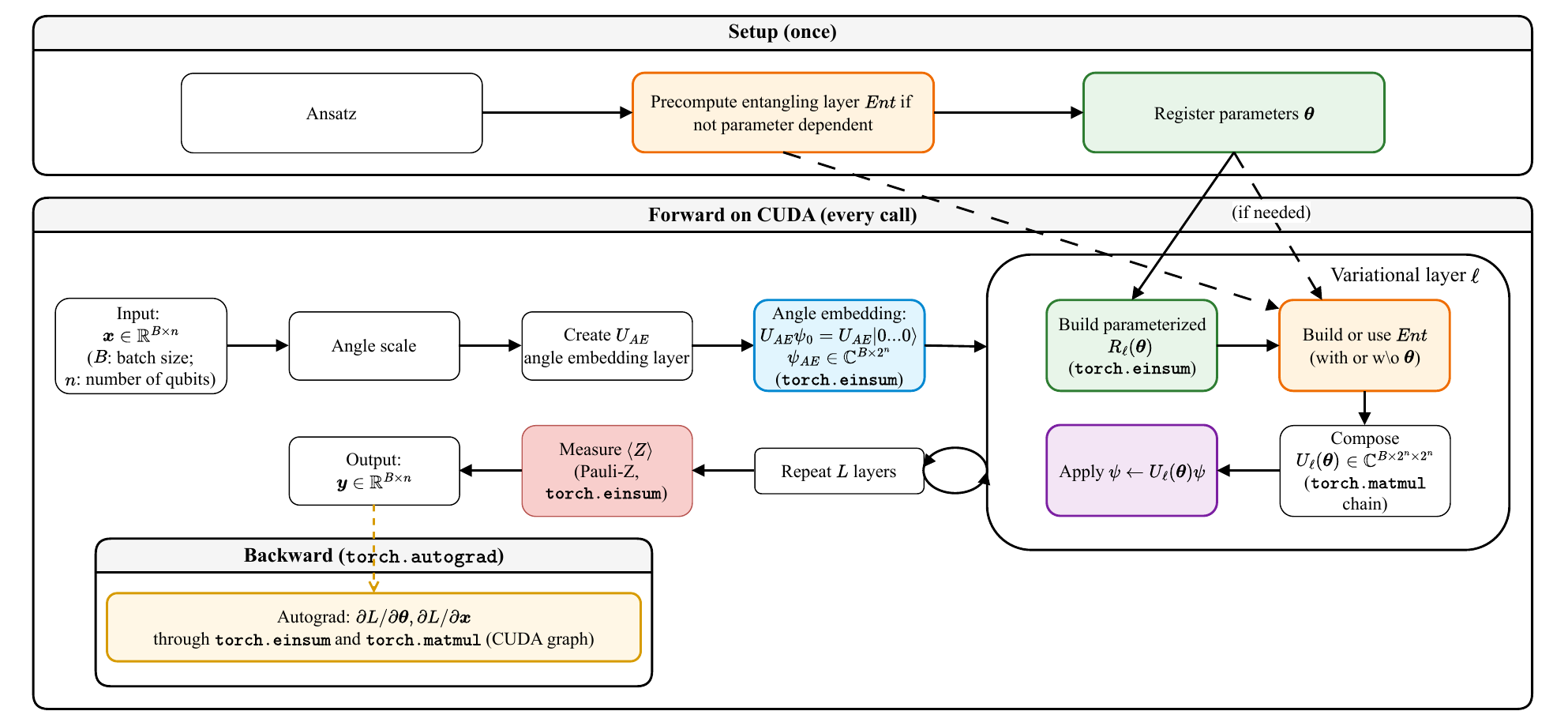}
\caption{\textit{TorQ} GPU execution pipeline.
Setup precomputes dense entanglers $Ent$ if they are not parameter-dependent, and registers trainable parameters $\theta$ on CUDA.
In the forward pass, inputs $x$ are angle-scaled and embedded by a batched Kronecker product implemented with \texttt{torch.einsum} to form the statevector $\psi \in \mathbb{C}^{B\times 2^n}$.
Each layer builds a dense rotation wall $R_\ell$ (via \texttt{einsum}), composes a dense operator $U_\ell$ with the precomputed entangler $Ent$ (via \texttt{torch.matmul}), and updates the state using a dense GEMM $\psi \leftarrow U_\ell\psi$.
Expectation values $\langle Z\rangle$ are computed analytically via \texttt{einsum}, and gradients propagate through the same CUDA computation graph via \texttt{torch.autograd}.}
\label{fig:torq_flowchart}
\end{figure}

\begin{figure}[t!]
\centering
\includegraphics[width=0.8\linewidth]{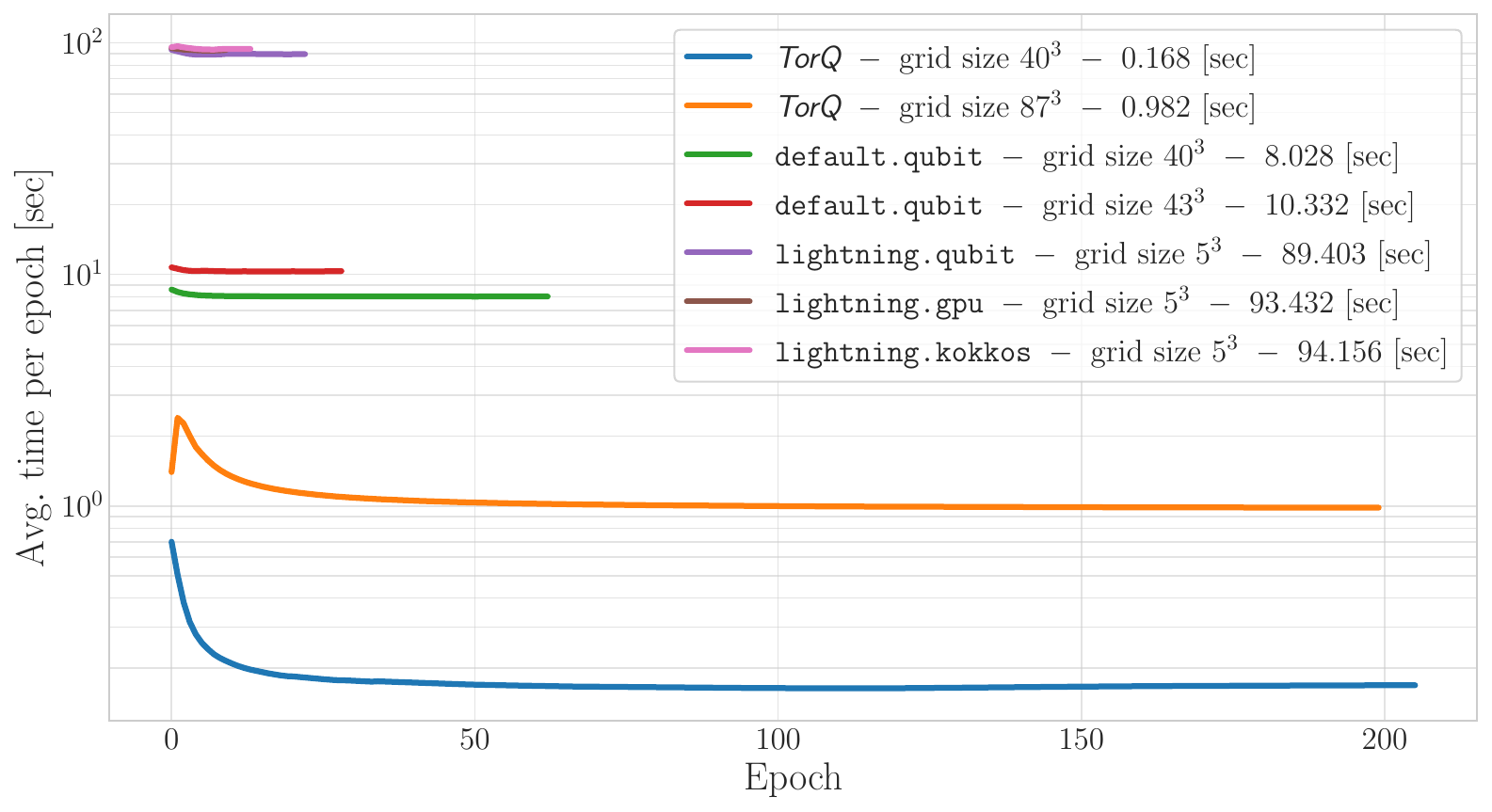}
\caption{Average running time per epoch comparison between \textit{TorQ} and several PennyLane simulators. The results show significant advantages for \textit{TorQ} across both runtime and memory usage. The value in the legend is the average epoch runtime at the end of the run.}
\label{fig:torq_benchmark}
\end{figure}

\section{Summary}

This work presented QPINNACLE, an open-source framework for physics-informed neural networks featuring convergence-enhancing techniques, quantum libraries, and multi-GPU execution. The benchmark suite spans linear advection, hyperbolic conservation laws, incompressible flows, and electromagnetic wave propagation, which can also be executed using a quantum simulator. The results demonstrate that successful training requires careful coordination of representation, constraint enforcement, loss weighting, and curriculum schedules. No single configuration was effective across all problems, and each benchmark exposed distinct numerical and optimization challenges.

A dominant limitation observed throughout the study is the high computational cost relative to classical solvers. Key findings from the benchmark results include:
\begin{itemize}
    \item The Fourier feature mappings and random-weight factorization techniques were crucial for capturing high-frequency components and improving training stability. They significantly enhanced the network's ability to model complex PDEs. 
    \item Dynamic loss balancing effectively managed the optimization landscape, preventing domination by specific loss components and ensuring balanced training.
    \item Enforcing boundary conditions directly within the network architecture proved essential, especially for problems with periodic conditions, simplifying the training process and improving computational efficiency.
    \item Curriculum training was vital for progressively exposing the network to increasingly challenging flow regimes, particularly in high-Re scenarios, thereby improving convergence and solution accuracy.
    \item The combination of ADAM and L-BFGS optimizers did not present any changes in simulations where the training error had already decreased and nearly converged.
\end{itemize}

Future work could focus on refining temporal causality techniques, exploring alternative collocation point generation methods, and extending the application of these advanced techniques to more complex and higher-dimensional PDEs. Despite encouraging results on tested low-dimensional benchmark problems, the observed computational cost and solution behavior indicate that PINNs remain fundamentally limited for problems with long time horizons and multiscale dynamics. For PINNs to serve as a general-purpose replacement for classical numerical methods, training cost would need to avoid exponential growth with problem complexity, convergence behavior would need to be reliable, and solution error would need to admit deterministic bounds comparable to those of established numerical solvers. 

Also, the QPINNs demonstrated several improvements, including higher accuracy, faster convergence, and fewer learnable parameters. On the other hand, there is still a major caveat: the number of circuit evaluations, especially for relevant cases, demands a prohibitive number of runs for foreseeable future quantum hardware and would require more elaborate methods of differentiation.

\subsection*{Funding}
ZC was supported by the Israel Science Foundation (ISF), Grants n. 939/23 and 2691/23, German-Israeli Project Cooperation (DIP) n. 2032991, Ollendorff-Minerva Center of the Technion n. 86160946, and the Helen Diller Quantum Center at the Technion. 2033613. 

\subsection*{Competing interests}
The authors declare that they have no known competing financial interests
or personal relationships that could have appeared to influence the work
reported in this paper.

\subsection*{Availability of data and material}
The benchmark problem definitions and the numerical results (training
histories, error metrics, runtime measurements) reported in this manuscript
are available within the QPINNACLE repository
(see \emph{Code availability} below). No proprietary or restricted datasets
were used in this study.

\subsection*{Code availability}
The complete source code of the QPINNACLE framework, including the
Jupyter notebooks, training scripts, and configuration files required to
reproduce all results in this manuscript, is openly available at
\texttt{https://github.com/zivchen9993/QPINNACLE.git} under the MIT open-source license.

\subsection*{Authors' contributions}
\textbf{Z.\ Chen}, \textbf{H.\ Chandravamsi}, and \textbf{S.\ Pisnoy} 
contributed equally to this work.
\textbf{S.\ Pisnoy} and \textbf{H.\ Chandravamsi} designed the classical
PINN framework, implemented the training pipeline and benchmark cases,
and performed the multi-GPU scaling experiments.
\textbf{Z.\ Chen} designed and implemented the hybrid quantum-classical
architecture and derived the parameter-shift complexity analysis.
\textbf{A.\ Goldgewert} contributed to
multi-GPU parallelization, high-performance computing and code review.
\textbf{G.\ Shaviner} and \textbf{B.\ Shrager} contributed to
additional bench marking the code and code review.
\textbf{S.\ H.\ Frankel} supervised the project, contributed to the conceptualization and methodology, and acquired funding. All authors discussed the results, contributed to writing the manuscript, and approved the final version.

\section*{Github repositories}
The full code repository for this work will be released upon publication. The GitHub repositories for the quantum components are listed below.
\begin{itemize}
    \item Quantum and classical solvers (toggle between quantum and classical simulations): 
    \url{https://github.com/zivchen9993/QPINNACLE.git}.
    \item The \textit{TorQ - Tensor Operations for Research of Quantum systems} library used to simulate the quantum circuits and integration to the classical parts of the network is available as a git repository at \url{https://github.com/zivchen9993/TorQ.git}, and can be installed with \textsc{pip}: \texttt{pip install torq-quantum}.
    \item The benchmarks against PennyLane were obtained using an extension to the said library, \textit{TorQ-bench}, and can be found in the git repository at \url{https://github.com/zivchen9993/TorQ-bench.git}, and can be installed with \textsc{pip}: \texttt{pip install torq-bench}.
\end{itemize}

\begin{appendix}
    \section{Code snippets}
    \subsection{Multi layer perceptron} \label{MLP_basic}
    The following class defines a simple multi-layer perceptron (MLP) with two inputs, four hidden layers (each containing 256 units), and one output feature. The network uses the hyperbolic tangent (tanh) activation function. The network can be instantiated by executing \texttt{model = simpleMLP()}; this initializes the network with random weights. To perform a forward pass through the network, you can call \texttt{model.forward(x, t)} or, more simply, \texttt{model(x, t)}.
\begin{lstlisting}[language=Python]
class simpleMLP(nn.Module):
    def __init__(self, in_dim=2, HL_dim=256, out_dim=1, activation=nn.Tanh()):
        super().__init__()
        self.u = nn.Sequential( nn.Linear(in_dim, HL_dim), activation,
                                nn.Linear(HL_dim, HL_dim), activation,
                                nn.Linear(HL_dim, HL_dim), activation,
                                nn.Linear(HL_dim, HL_dim), activation,
                                nn.Linear(HL_dim, out_dim)) 
                                
    def forward(self, x, t):
        return self.u.(torch.cat((x, t), 1))   \end{lstlisting}


    \subsection{Random Fourier features}
    Random Fourier Features (RFF), as described in section \ref{sec:RFF}, help mitigate the neural network’s spectral bias towards low-frequency features of the solution. The \texttt{RFF} class shown below takes four inputs: \texttt{in\_dim} (input layer size), \texttt{out\_dim} (size of the first hidden layer), \texttt{mean\_RFF} (mean of the Gaussian distribution), and \texttt{std\_RFF} (standard deviation of the Gaussian). The inputs are mapped to a higher-dimensional space, with the output being a summation of Fourier modes determined by a set of random values (\texttt{B}) initialized through a Gaussian distribution.
\begin{lstlisting}[language=Python]
class RFF(nn.Module):
    def __init__(self, in_dim, out_dim, mean_RFF, std_RFF):
        super().__init__()
        self.B = nn.Parameter(torch.normal(mean=mean_RFF, std=std_RFF, size=(in_dim, out_dim), requires_grad=False))
    
    def forward(self, x):
        x_map = torch.matmul(x, self.B)
        return torch.cat([torch.cos(x_map), torch.sin(x_map)], -1) \end{lstlisting}

    The code above for \texttt{RFF} class can be easily integrated into a standard MLP model by adding \texttt{RFF()} as the first element in the \texttt{nn.Sequential} within the model constructor. After that, the first linear layer following the RFF layer should be modified to \texttt{nn.Linear(2*out\_dim, out\_dim)} to account for the expanded feature dimensions from the RFF mapping.

    \subsection{Random weight factorization}
    Instead of relying on the random weights initialized by a default MLP, such as the one shown in Sec. \ref{MLP_basic}, a strategic way to improve performance is by initializing the weights through the Random Weight Factorization (RWF) method described in Sec. \ref{sec:RWF}. The weights corresponding to the RWF initialization can be computed using the following class:
\begin{lstlisting}[language=Python]
class RWF(nn.Module):
    def __init__(self, in_dim, out_dim, mean_RWF, std_RWF):
        super().__init__()
        self.V = nn.Parameter(torch.Tensor(out_dim, in_dim))
        self.s = nn.Parameter(torch.Tensor(out_dim, 1))
        
        init.xavier_normal_(self.V)
        init.normal_(self.s, mean=mean_RWF, std=std_RWF)
    
    def forward(self, x):
        W =torch.matmul(torch.diag_embed(torch.exp(self.s.squeeze(-1))), self.V)
        return F.linear(x, W)    \end{lstlisting}

The above defined \texttt{RWF} class can be used within a neural network class very easily, by simply replacing the \texttt{nn.Linear(dim\_in, dim\_out)} definitions with \texttt{RWF(dim\_in, dim\_out, mean\_RWF, std\_RWF)}.

    \subsection{Strict periodic boundary conditions}
    Strict boundary conditions eliminate the need to define boundary-condition losses, thereby reducing the computational burden during PINN training. The following network class demonstrates a \texttt{forward} function that incorporates a periodic boundary condition along the \(x\)-direction. This condition ensures periodicity with a length equal to \texttt{xf - xi}, where \texttt{xf} and \texttt{xi} represent the final and initial \(x\)-coordinates of the 1D periodic domain, respectively.
\begin{lstlisting}[language=Python]
class MLP_with_periodicBC(nn.Module):
    def __init__(self, in_dim=2, HL_dim=256, out_dim=1, activation=nn.Tanh()):
        super().__init__()
        self.u = nn.Sequential(nn.Linear(2*HL_dim, HL_dim), activation,
                               nn.Linear(HL_dim,   HL_dim), activation,
                               nn.Linear(HL_dim,   HL_dim), activation,
                               nn.Linear(HL_dim,  out_dim))
        
    def forward(self, x, t, config):
        # for periodic BC:
        x_sin = torch.sin(2*torch.pi/(config.xf-config.xi) * x)
        x_cos = torch.cos(2*torch.pi/(config.xf-config.xi) * x)
        return self.u.(torch.cat((x_sin, x_cos, t), 1))    \end{lstlisting}
    
    \subsection{Periodic activation function}
    
    The following class defines a periodic activation function which takes input \texttt{x} and transforms to \texttt{sin(w0*x)}. Where \texttt{w0} is a hyperparameter that controls the degree of high-frequency details to be resolved in the solution. The activation function can be used by simply specifying \texttt{SineActivation(w0)} in the place of \texttt{nn.tanh()}. 
    
\begin{lstlisting}[language=Python]
class SineActivation(nn.Module):
    # periodic activation function: sin(w0*x)
    def __init__(self, w0):
        super().__init__()
        self.w0 = w0

    def forward(self,x):
        return torch.sin(self.w0*x)    \end{lstlisting}
    
    \subsection{Coefficient computation function for loss balancing}
    The following Python function shows how to update the coefficients \(\lambda_{pde}\), \(\lambda_{ic}\), and \(\lambda_{bc}\) based on the network's current loss values, following the procedure outlined in Sec. \ref{sec:loss_balancing}.
    
\begin{lstlisting}[language=Python]
def update_global_weights(model, eq_loss, ic_loss, bc_loss, lambda_eq, lambda_ic, lambda_bc, epoch):
    losses = [eq_loss, ic_loss, bc_loss]
    grads_concat = []

    # ******** compute grads of loss w.r.t. model params ********
    for loss in losses:
        grads = grad(loss, model.parameters(), create_graph=True)                           # compute gradients loss w.r.t. model params
        concatenated_grads = torch.cat([g.view(-1) for g in grads if g is not None]) # concat all d(loss)/d(theta)
        grads_concat.append(concatenated_grads)

    # ******** compute L2 norms of grads and compute weights ********
    norms = torch.tensor([torch.linalg.norm(g.detach(), ord=2) for g in grads_concat])  # L2 norm of gradients
    sum_norms = torch.sum(norms)   # sum of L2 norms
    weights_updated = sum_norms/(norms + 1e-9)
    weights_updated = weights_updated.tolist()

    # ******** update weights ********
    if epoch % 1000 == 0 and epoch > 0:
        lambda_eq = 0.9*lambda_eq + 0.1*weights_updated[0]
        lambda_ic = 0.9*lambda_ic + 0.1*weights_updated[1]
        lambda_bc = 0.9*lambda_bc + 0.1*weights_updated[2]
        
    return lambda_eq, lambda_ic, lambda_bc      \end{lstlisting}

    \subsection{\textit{TorQ} library}

    In this subsection, selected excerpts of the implementation are provided to illustrate the core components of the \textit{TorQ} library. The focus is on the elements that define the user interface and the construction of variational quantum circuits, while auxiliary details are omitted for clarity.
    
    \subsubsection{QLayer}\label{subsubsec:qlayer}
    The quantum neural network, \textit{Qlayer}, used in \textit{TorQ} serves as the user's connection node to the \textit{TorQ} library and defines its API. Parts of the class that weren't discussed in this work were removed from the code snippet for better readability; the full class appears in the \textit{TorQ} repository~\cite{torq}.

\begin{lstlisting}[language=Python]
class QLayer(nn.Module):
    def __init__(self, n_qubits=3, n_layers=1, ansatz_name="strongly_entangling", config=None,
                 basis_angle_embedding='X'):
        super().__init__()
        self.n_qubits = n_qubits
        self.n_layers = n_layers
        self.config = config
        self.ansatz_name = ansatz_name
        self.basis_angle_embedding = basis_angle_embedding

        self.angle_scaling_method = getattr(self.config, 'angle_scaling_method', 'none')
        self.angle_scaling = getattr(self.config, 'angle_scaling', 1.0)

        # === ansatz selection + parameter tensor shape ===
        # build ansatz object (holds any precomputes)
        # NOTE: device of params is not known yet; we'll move precomputes lazily on first forward if needed
        self.ansatz = make_ansatz((ansatz_name, n_qubits, n_layers, device=None)

        # sigma_Z observable
        self.observable = tq.sigma_Z_like(x=self.params)  # keeps dtype/device consistent via likeable
        self._optional_backend = maybe_create_pennylane_backend(self)  # for benchmarking and testing against pennylane; if it is None, the regular forward will be used. For benchmarking, the TorQ-bench library should be used.

    def forward(self, x):
        if not torch.isfinite(self.params).all():
            raise ValueError(f"QLayer.params has NaN: {self.params}")
        if self._optional_backend is not None:
            return self._optional_backend.forward(x)
        state = tq.angle_embedding(
            x,
            angle_scaling_method=self.angle_scaling_method,
            angle_scaling=self.angle_scaling,
            basis=self.basis_angle_embedding,
        ).squeeze(-1)  # [B,2**n]

        angles_reparametrize = None
        if self.reparametrize_sin_cos:
            angles_reparametrize = torch.atan2(torch.sin(self.params), torch.cos(self.params))

        for layer in range(self.n_layers):
            if self.reparametrize_sin_cos:
                w = angles_reparametrize[layer]
            else:
                w = self.params[layer]
            U = self.ansatz.layer_op(layer, w)  # [2**n, 2**n]
            state = tq.apply_matrix(state, U)

        return tq.measure(state, self.observable)
\end{lstlisting}

    \subsubsection{Ansatz building example}\label{subsubsec:ansatz}
    The ansatz building procedure is demonstrated by the Strongly Entangling ansatz.

\begin{lstlisting}[language=Python]
class StronglyEntangling(BaseAnsatz):
    param_shape = ("per_layer", (None, 3))

    def __init__(self, n_qubits, n_layers, device=None):
        super().__init__(n_qubits, n_layers, device)
        dummy = torch.empty(1, device=device)
        # one ladder per "r" value
        self.cnots = [
            aq.get_cnot_ladder(n_qubits, r=r, x=dummy)
            for r in range(n_layers)
        ]

    def layer_op(self, layer_idx, weights):
        return aq.basic_or_strongly_single_layer(self.n, weights, self.cnots[layer_idx]).to(weights.device)
\end{lstlisting}

    \section{A short tutorial on implementing the Distributed Data Parallel approach for multi-GPU PINNs}
    Distributed Data Parallel (DDP) works by distributing the forward and backward pass computations across multiple GPUs. Each GPU maintains an independent copy of the model in its own memory, isolated from other GPUs. However, the model parameters are to be kept constant across the GPUs at all instances. This is achieved through \textit{gradient synchronization}, which will be discussed soon. 
    
    To fully leverage Distributed Data Parallel (DDP), the dataset (such as spatio-temporal collocation points or network inputs) is partitioned into smaller chunks (mini-batches), allowing each GPU to process its portion independently. This approach reduces net computational time by performing forward and backward passes on the local mini-batch, and it scales the effective dataset size linearly with the number of GPUs. DDP ensures model parameters stay synchronized across all GPUs throughout training by following a defined communication protocol. A schematic of the DDP approach and the various steps 1-7 are shown in figure \ref{fig:ddp}. 

    \begin{figure}[h!]
        \centering
        \includegraphics[width=\linewidth]{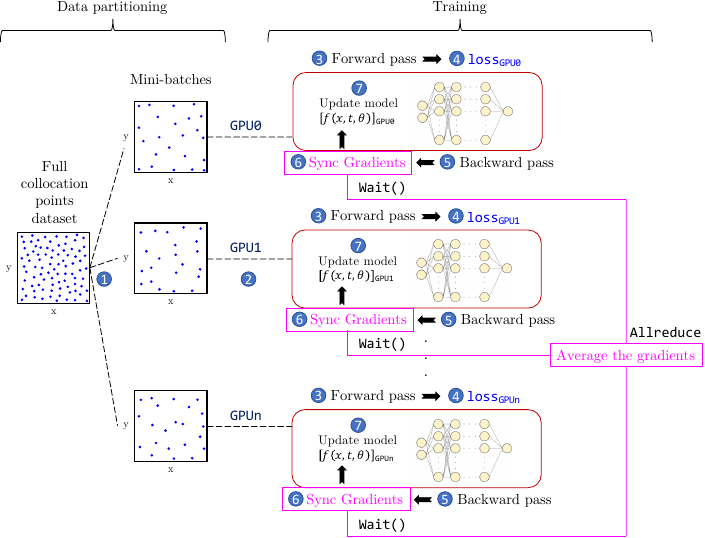}
        \caption{Overview of the Distributed Data Parallel (DDP) approach. This schematic illustrates how model parameters and gradients are synchronized across multiple GPUs to achieve parallel training. Each GPU trains a copy of the model on its local collocation points, and gradients are aggregated and averaged across nodes to update the model parameters.}
        \label{fig:ddp}
    \end{figure}
    
    A summary of key steps as depicted in Figure \ref{fig:ddp} can be enumerated as follows:

    \begin{enumerate}
        \item \textbf{Data Partitioning:} Split the dataset across multiple GPUs. Each GPU processes a different subset of the data (mini-batch).
        
        \item \textbf{Model Initialization:} Clone the model onto each GPU. All GPUs start with identical models.
        
        \item \textbf{Forward Pass (on each GPU):} Each GPU processes its own mini-batch and computes the output from its local model.
        
        \item \textbf{Compute Loss (on each GPU):} Calculate the model loss based on the local mini-batch for each GPU.
        
        \item \textbf{Backward Pass (Gradient Computation):} Each GPU calculates gradients for its mini-batch using backpropagation.
        
        \item \textbf{Gradient Synchronization:} All GPUs communicate and aggregate their gradients by averaging them across GPUs.
        
        \item \textbf{Model Update:} Each GPU updates its local model parameters based on the synchronized gradients.
        
        \item \textbf{Repeat:} Repeat steps 3-7 for all mini-batches across all GPUs until training is complete.
    \end{enumerate}

    Figure \ref{fig:ddp} illustrates the same procedure.

    \subsection{\texorpdfstring{From Single-GPU $\rightarrow$ Multi-GPU DDP PINNs}{From Single-GPU -> Multi-GPU DDP PINNs}}
    Fundamentally, implementing DDP is simple and requires only minor modifications to a single GPU code. One thing to keep in mind is that DDP doesn’t work with Jupyter notebooks (".ipynb" files aren’t supported). For a full multi-GPU example, check out the lid-driven cavity case in Module 6 of the QPINNACLE repository. The extra libraries to be imported to work with DDP are: 

    \begin{lstlisting}[language=Python]
import torch.distributed as dist
import torch.multiprocessing as mp
from torch.utils.data.distributed import DistributedSampler
from torch.nn.parallel import DistributedDataParallel as DDP
from torch.distributed import init_process_group, destroy_process_group \end{lstlisting}
    
    Starting with the single-GPU code available in Module 6 of QPINNACLE, \texttt{ldc\_singleGPU.py}, the following is a step-by-step procedure for adapting it to use DDP.

    \subsubsection*{Step 1/4: Preparing the mini-batches}
    The collocation points ($\mathbf{x},t$) are sliced into mini-batches of roughly equal size based on the number of GPUs available, i.e., \texttt{world\_size}. This is done as follows using a few lines of code: 
    
    \begin{lstlisting}[language=Python]
# obtain the full set of collocation points
x_points, y_points, t_points, x_bc, y_bc = Generate_Collocation_Points()

# obtain the size of data set
num_colloc = x_points.view(-1,1).size(0)                        

# determine mini-batch length of interior collocation points
slice_size = config.num_colloc//world_size

# starting index of the sliced mini-batch on the corresponding rank
slice_start = rank*slice_size 

# send to trainer a clone of initialized model and corresponding inputs
trainer = Trainer(model,x_points[slice_start:slice_start+slice_size],
                        y_points[slice_start:slice_start+slice_size],
                        t_points[slice_start:slice_start+slice_size],
                        x_bc,y_bc, world_size, optimizer, scheduler, rank)  \end{lstlisting}
    
    The first step is to determine the size of each mini-batch by dividing the total number of collocation points (\texttt{num\_colloc}) by the total number of GPUs (\texttt{world\_size}). Next, to find the starting index of the mini-batch corresponding to the local rank, multiply the GPU ID (rank) by the mini-batch size. To then provide this chunk to the Trainer class (which will be discussed in the next step), the points, which are essentially a list, are sliced from \texttt{slice\_start} to \texttt{slice\_start + slice\_size}.
    
    \textit{Note:} For better training robustness, the collocation points corresponding to the domain boundaries (\texttt{x\_bc}, \texttt{y\_bc}) are not divided into multiple batches. The same (\texttt{x\_bc}, \texttt{y\_bc}) dataset is used on all GPUs during training. Since the boundary collocation points are often significantly fewer than the interior collocation points, this does not create a significant computational overhead.
    

    \subsubsection*{Step 2/4: Creating the trainer class}
    We need to add a new function called \texttt{ddp\_setup}, to initialize the Distributed Data Parallel (DDP) environment, set up the process group, and specify the GPU to be used for each process. The function is a standard and can be written as follows:
    
    \begin{lstlisting}[language=Python, label={lst:ddp_setup}]
def ddp_setup(rank, world_size):
    """
    Args:
        rank: Unique identifier of each process
        world_size: Total number of processes
    """
    # Set environment variables for DDP communication
    os.environ["MASTER_ADDR"] = "localhost"
    os.environ["MASTER_PORT"] = "12355"
    os.environ["NCCL_TIMEOUT"] = "300"

    # Initialize the process group for DDP
    init_process_group(backend="nccl", rank=rank, world_size=world_size)

    # Set the device to be used by the current process
    torch.cuda.set_device(rank) \end{lstlisting}
    
    This function takes two arguments: 1) \texttt{rank} - the current GPU, and 2) \texttt{world\_size} - the total number of GPUs in the server. Next, a new class called \texttt{Trainer} is created to wrap the training loop into a single callable unit, like such: 
    
    \begin{lstlisting}[language=Python]
class Trainer:
    def __init__(self, model: torch.nn.Module, x, y, x_bc, y_bc, t, 
                 world_size, optimizer: torch.optim.Optimizer, 
                 scheduler, gpu_id: int) -> None:
        self.gpu_id         = gpu_id
        self.model          = model.to(gpu_id)
        self.x              = x
        self.y              = y
        self.t              = t
        self.x_bc           = x_bc
        self.y_bc           = y_bc
        self.scheduler      = scheduler
        self.optimizer      = optimizer
        self.save_every     = save_every
        self.optim_lbfgs    = optim_lbfgs
        self.model          = DDP(model, device_ids=[gpu_id])
        
    def train(self, max_epochs:int): # Training loop
        dist.barrier()
        for epoch in range(max_epochs):
            # training goes here
            dist.barrier()
        dist.barrier() \end{lstlisting}
    
    The training loop from the single-GPU code is moved into this class at \texttt{trainer.train()}. The parameters that are passed into the trainer class are: 
    \begin{enumerate}
        \item the model itself,
        \item the training data, which has already been broken up as described in step 1/4,
        \item the world size, which is just how many GPUs are in the server,
        \item the optimizer,
        \item the scheduler (if one is in use),
        \item the \texttt{gpu\_id}  (\texttt{= rank}), which is an identifier, the class can use to know on which GPU it is running. 
    \end{enumerate}
    
    In the \texttt{\_\_init\_\_()} constructor above, the model is wrapped with DDP on line 16 by creating a DDP object that accepts the original PyTorch model and the GPU ID. This allows the model to communicate with other copies of itself on other GPUs. Additionally, in the \texttt{train} function, the \texttt{dist.barrier()} command is called before, after, and during the training loop to force the GPUs to wait for each other.
    
Within the training loop, two modifications are needed to ensure that all data is pushed to the correct GPU.
\begin{enumerate}
    \item \texttt{.cuda() / .to(device)}
    \begin{enumerate}
        \item In a normal single-GPU PyTorch setup, to push a variable to the GPU, all that needs to be done is to call \texttt{.cuda()}:
            \begin{lstlisting}[language=Python]
x.cuda() \end{lstlisting}
        However, this does not work in a DDP setup because there are multiple CUDA devices. Therefore, a small change must be made:
        \begin{lstlisting}[language=Python]
x.to(self.gpu_id) \end{lstlisting}
        This pushes the variable in question to the current GPU instead of the first available GPU. This change must be applied at every point where the program uses the \texttt{.cuda()} or \texttt{.to(device)} command.
    \end{enumerate}

    \item Accessing model functions
    \begin{enumerate}
        \item Generally, accessing functions within the model is just like accessing functions of any object in Python:
        \begin{lstlisting}[language=Python]
model.compute_loss(x, y) \end{lstlisting}
        \item However, since the model has been wrapped by DDP, accessing functions and parameters this way no longer works. Instead, use:
        \begin{lstlisting}[language=Python]
model.module.compute_loss(x, y) \end{lstlisting}
    \end{enumerate}
\end{enumerate}

    \subsubsection*{Step 3/4: Creating the main function and argument parser}

    Two functions must be created to run the code: the main function and the argument parser. \textit{\textit{Main Function:}} This function contains the entire computation sequence necessary for training the PINNs network. By calling this function, all the steps are executed in sequence to train the network. It begins with preprocessing tasks, including generating and slicing the training points, and then uses the \texttt{Trainer} class to carry out the training process. Finally, the function instantiates the trained network and plots the results. The function follows this general structure:
    \begin{lstlisting}[language=Python]
def main(rank, world_size, total_epochs):
    ddp_setup(rank, world_size)     # initializes DDP
    torch.cuda.set_device(rank)     # tells DDP which GPU is the current GPU
    torch.cuda.empty_cache()        # clears memory

    # preprocessing steps
    x_points, y_points, t_points, x_bc, y_bc = Generate_Collocation_Points()
    # <slice data here - as explained in step 1/4>
    model = QPINNACLE()
    optimizer = torch.optim.Adam(model.parameters(), lr=0.001)
    scheduler = ExponentialLR(optimizer, gamma=config.lr_decay)

    # initialize the trainer
    trainer = Trainer(model, x_sliced, y_sliced, x_bc, y_bc, 
                      world_size, optimizer,scheduler, rank)

    # train the model 
    trainer.train(total_epochs) 

    # <instance the trained network>
    # <plot the results>  \end{lstlisting}
    
    \textit{\textit{Argument parser:}} The other piece of code that must be written is a function that tells Python to run the main function. Here is its form: \begin{lstlisting}[language=Python]
if __name__ == "__main__":
    import argparse
    parser = argparse.ArgumentParser(description='simple distributed training job')
    parser.add_argument('total_epochs', type=int, help='Total epochs to train the model') #Controls the number of epochs
    args = parser.parse_args()
    
    world_size = torch.cuda.device_count()
    mp.spawn(main, args=(world_size, args.total_epochs), nprocs=world_size) \end{lstlisting}
    
    This condition contains a few elements. Mainly, it adds an argument for the number of epochs for training, counts the number of GPUs in the system, and pushes those numbers to the main function using \texttt{mp.spawn}, PyTorch's method for creating multiple processes.

    \subsubsection*{Step 4/4: The training loop}
    One advantage of using DDP is that the training function remains mostly the same, with a few key differences to maintain data synchronization across GPUs. First, the \texttt{dist.barrier()} command. This command forces the GPUs to update in the proper order. It is placed at the beginning and end of the training loop, as such: 
    
    \begin{lstlisting}[language=Python]
def train(self, max_epochs:int): #Training loop
    dist.barrier()
    for epoch in range(max_epochs):
        #training goes here
        dist.barrier()
    dist.barrier() \end{lstlisting}
    
    Second, \texttt{all\_reduce}. This command forces data synchronization across all GPUs in the system and reduces the inability to converge due to incorrect gradient values. Changing the code is very simple: 
    
    \begin{lstlisting}[language=Python]
reduced_loss = loss.clone()
dist.all_reduce(reduced_loss, op=dist.ReduceOp.SUM)
reduced_loss /= self.world_size \end{lstlisting}
    
    Then, perform back-propagation on the reduced values: 
    
    \begin{lstlisting}[language=Python]
self.optimizer.zero_grad()
reduced_loss.backward()
L_eq[0], L_ic[0], L_total[0] = reduced_eq_loss.item(), reduced_ic_loss.item(), reduced_loss.item() \end{lstlisting}
    
    One thing to ensure is that each value is pushed to the correct GPU. This can be accomplished simply by passing the GPU ID as an argument to each loss function.

\end{appendix}

\bibliographystyle{unsrt}
\bibliography{ML_bibilography}

@article{sitzmann2020implicit,
  title={Implicit neural representations with periodic activation functions},
  author={Sitzmann, Vincent and Martel, Julien and Bergman, Alexander and Lindell, David and Wetzstein, Gordon},
  journal={Advances in neural information processing systems},
  volume={33},
  pages={7462--7473},
  year={2020}
}

@article{karniadakis2021piml,
title        = {Physics-informed machine learning},
author       = {Karniadakis, George Em and Kevrekidis, Ioannis G and Lu, Lu and Perdikaris, Paris and Wang, Sifan and Yang, Liu},
journal      = {Nature Reviews Physics},
volume       = {3},
pages        = {422--440},
year         = {2021},
doi          = {10.1038/s42254-021-00314-5}
}

@article{lu2021deepxde,
title        = {DeepXDE: A deep learning library for solving differential equations},
author       = {Lu, Lu and Meng, Xuhui and Mao, Zhiping and Karniadakis, George Em},
journal      = {SIAM Review},
volume       = {63},
number       = {1},
pages        = {208--228},
year         = {2021},
doi          = {10.1137/19M1274067}
}

@article{wang2022ntk,
title        = {When and why {PINNs} fail to train: A neural tangent kernel perspective},
author       = {Wang, Sifan and Yu, Xinling and Perdikaris, Paris},
journal      = {Journal of Computational Physics},
volume       = {449},
pages        = {110768},
year         = {2022},
doi          = {10.1016/j.jcp.2021.110768}
}

@article{wang2024causal,
title        = {Respecting causality for training physics-informed neural networks},
author       = {Wang, Sifan and Sankaran, Shyam and Perdikaris, Paris},
journal      = {Computer Methods in Applied Mechanics and Engineering},
volume       = {421},
pages        = {116813},
year         = {2024},
doi          = {10.1016/j.cma.2024.116813}
}

@misc{wang2023expertsguide,
title        = {An Expert's Guide to Training Physics-informed Neural Networks},
author       = {Wang, Sifan and Sankaran, Shyam and Wang, Hanwen and Perdikaris, Paris},
year         = {2023},
eprint       = {2308.08468},
archivePrefix= {arXiv},
primaryClass = {cs.LG},
doi          = {10.48550/arXiv.2308.08468}
}

@inproceedings{krishnapriyan2021failuremodes,
title        = {Characterizing possible failure modes in physics-informed neural networks},
author       = {Krishnapriyan, Aditi S and Gholami, Amir and Zhe, Shandian and Kirby, Robert M and Mahoney, Michael W},
booktitle    = {Advances in Neural Information Processing Systems},
volume       = {34},
pages        = {26548--26560},
year         = {2021}
}

@inproceedings{rahaman2019spectralbias,
title        = {On the spectral bias of neural networks},
author       = {Rahaman, Nasim and Baratin, Aristide and Arpit, Devansh and Draxler, Felix and Lin, Min and Hamprecht, Fred and Bengio, Yoshua and Courville, Aaron},
booktitle    = {Proceedings of the 36th International Conference on Machine Learning},
series       = {Proceedings of Machine Learning Research},
volume       = {97},
pages        = {5301--5310},
year         = {2019},
publisher    = {PMLR}
}

@misc{tancik2020fourierfeatures,
title        = {Fourier Features Let Networks Learn High Frequency Functions in Low Dimensional Domains},
author       = {Tancik, Matthew and Srinivasan, Pratul P. and Mildenhall, Ben and Fridovich-Keil, Sara and Raghavan, Nithin and Singhal, Utkarsh and Ramamoorthi, Ravi and Barron, Jonathan T. and Ng, Ren},
year         = {2020},
eprint       = {2006.10739},
archivePrefix= {arXiv},
primaryClass = {cs.CV}
}

@article{salimans2016weightnorm,
  title={Weight normalization: A simple reparameterization to accelerate training of deep neural networks},
  author={Salimans, Tim and Kingma, Durk P},
  journal={Advances in neural information processing systems},
  volume={29},
  year={2016}
}

@inproceedings{bengio2009curriculum,
  title={Curriculum learning},
  author={Bengio, Yoshua and Louradour, J{\'e}r{\^o}me and Collobert, Ronan and Weston, Jason},
  booktitle={Proceedings of the 26th annual international conference on machine learning},
  pages={41--48},
  year={2009}
}

@article{wu2023sampling,
title        = {A comprehensive study of non-adaptive and residual-based adaptive sampling for physics-informed neural networks},
author       = {Wu, Chenxi and Zhu, Min and Tan, Qinyang and Kartha, Yadhu and Lu, Lu},
journal      = {Computer Methods in Applied Mechanics and Engineering},
volume       = {403},
pages        = {115671},
year         = {2023},
doi          = {10.1016/j.cma.2022.115671}
}

@article{jagtap2020xpinns,
title        = {Extended physics-informed neural networks ({XPINNs}): A generalized space-time domain decomposition based deep learning framework for nonlinear partial differential equations},
author       = {Jagtap, Ameya D. and Karniadakis, George Em},
journal      = {Communications in Computational Physics},
volume       = {28},
number       = {5},
pages        = {2002--2041},
year         = {2020},
doi          = {10.4208/cicp.OA-2020-0164}
}

@misc{kingma2014adam,
title        = {Adam: A method for stochastic optimization},
author       = {Kingma, Diederik P. and Ba, Jimmy},
year         = {2014},
eprint       = {1412.6980},
archivePrefix= {arXiv},
primaryClass = {cs.LG}
}

@article{liu1989lbfgs,
title        = {On the limited memory {BFGS} method for large scale optimization},
author       = {Liu, Dong C. and Nocedal, Jorge},
journal      = {Mathematical Programming},
volume       = {45},
number       = {1},
pages        = {503--528},
year         = {1989},
doi          = {10.1007/BF01589116}
}

@misc{singh2021nysnewton,
title        = {{Nys-Newton}: Nystr{"o}m-Approximated Curvature for Stochastic Optimization},
author       = {Singh, Dinesh and Tankaria, Hardik and Yamada, Makoto},
year         = {2021},
eprint       = {2110.08577},
archivePrefix= {arXiv},
primaryClass = {cs.LG}
}

@inproceedings{rathore2024pinnlosslandscape,
title        = {Challenges in Training {PINNs}: A Loss Landscape Perspective},
author       = {Rathore, Pratik and Lei, Weimu and Frangella, Zachary and Lu, Lu and Udell, Madeleine},
booktitle    = {Proceedings of the 41st International Conference on Machine Learning},
series       = {Proceedings of Machine Learning Research},
volume       = {235},
pages        = {42159--42191},
year         = {2024},
publisher    = {PMLR}
}

@misc{micikevicius2018mixedprecision,
title        = {Mixed Precision Training},
author       = {Micikevicius, Paulius and Narang, Sharan and Alben, Jonah and Diamos, Gregory and Elsen, Erich and Garcia, David and Ginsburg, Boris and Houston, Michael and Kuchaiev, Oleksii and Venkatesh, Ganesh and Wu, Hao},
year         = {2018},
eprint       = {1710.03740},
archivePrefix= {arXiv},
primaryClass = {cs.LG}
}

@article{ghia1982highre,
title        = {High-{Re} solutions for incompressible flow using the Navier-Stokes equations and a multigrid method},
author       = {Ghia, U. and Ghia, K. N. and Shin, C. T.},
journal      = {Journal of Computational Physics},
volume       = {48},
number       = {3},
pages        = {387--411},
year         = {1982},
doi          = {10.1016/0021-9991(82)90058-4}
}

@article{wang2025simulating,
  title={Simulating three-dimensional turbulence with physics-informed neural networks},
  author={Wang, Sifan and Sankaran, Shyam and Fan, Xiantao and Stinis, Panos and Perdikaris, Paris},
  journal={arXiv preprint arXiv:2507.08972},
  year={2025}
}

@article{shaviner2025pinns,
  title={PINNs for solving unsteady Maxwell’s equations: convergence issues and comparative assessment with compact schemes},
  author={Shaviner, Gal G and Chandravamsi, Hemanth and Pisnoy, Shimon and Chen, Ziv and Frankel, Steven H},
  journal={Neural Computing and Applications},
  volume={37},
  number={29},
  pages={24103--24122},
  year={2025},
  publisher={Springer}
}

@article{chandravamsi2025spectral,
  title={Spectral Bottleneck in Sinusoidal Representation Networks: Noise is All You Need},
  author={Chandravamsi, Hemanth and Shenoy, Dhanush V and Zinn, Itay and Chen, Ziv and Pisnoy, Shimon and Frankel, Steven H},
  journal={arXiv preprint arXiv:2509.09719},
  year={2025}
}

@article{chen2025quantum,
  title={Quantum Physics-Informed Neural Networks for Maxwell's Equations: Circuit Design," Black Hole" Barren Plateaus Mitigation, and GPU Acceleration},
  author={Chen, Ziv and Shaviner, Gal G and Chandravamsi, Hemanth and Pisnoy, Shimon and Frankel, Steven H and Pereg, Uzi},
  journal={arXiv preprint arXiv:2506.23246},
  year={2025}
}

@article{arzani2021nearwall,
  title   = {Uncovering near-wall blood flow from sparse data with physics-informed neural networks},
  author  = {Arzani, Amirhossein and Wang, Jian Xun and D'Souza, Roshan M.},
  journal = {Physics of Fluids},
  volume  = {33},
  number  = {7},
  pages   = {071905},
  year    = {2021},
  doi     = {10.1063/5.0055600}
}

@article{raissi2019pinns,
  title   = {Physics-informed neural networks: A deep learning framework for solving forward and inverse problems involving nonlinear partial differential equations},
  author  = {Raissi, Maziar and Perdikaris, Paris and Karniadakis, George E.},
  journal = {Journal of Computational Physics},
  volume  = {378},
  pages   = {686--707},
  year    = {2019},
  doi     = {10.1016/j.jcp.2018.10.045}
}

@article{sod1978shock,
  title   = {A survey of several finite difference methods for systems of nonlinear hyperbolic conservation laws},
  author  = {Sod, Gary A.},
  journal = {Journal of Computational Physics},
  volume  = {27},
  number  = {1},
  pages   = {1--31},
  year    = {1978},
  doi     = {10.1016/0021-9991(78)90023-2}
}

@book{toro2009riemann,
  title     = {Riemann Solvers and Numerical Methods for Fluid Dynamics},
  author    = {Toro, Eleuterio F.},
  publisher = {Springer},
  edition   = {3rd},
  year      = {2009},
  doi       = {10.1007/b79761}
}

@inproceedings{rahimi2007random,
  title={Random features for large-scale kernel machines},
  author={Rahimi, Ali and Recht, Benjamin},
  booktitle={Advances in neural information processing systems},
  pages={1177--1184},
  year={2007}
}

@article{wang2022random,
  title={Random weight factorization improves the training of continuous neural representations},
  author={Wang, Sifan and Wang, Hanwen and Seidman, Jacob H and Perdikaris, Paris},
  journal={arXiv preprint arXiv:2210.01274},
  year={2022}
}

@article{dong2021,
  author    = {Dong, S. and Ni, N.},
  title     = {A method for representing periodic functions and enforcing exactly periodic boundary conditions with deep neural networks},
  journal   = {Journal of Computational Physics},
  volume    = {435},
  pages     = {110242},
  year      = {2021},
  doi       = {10.1016/j.jcp.2021.110242},
  url       = {https://doi.org/10.1016/j.jcp.2021.110242}
}

@article{xiang2022self,
  title={Self-adaptive loss balanced physics-informed neural networks},
  author={Xiang, Zixue and Peng, Wei and Liu, Xu and Yao, Wen},
  journal={Neurocomputing},
  volume={496},
  pages={11--34},
  year={2022},
  publisher={Elsevier}
}

@article{wang2021understanding,
  title={Understanding and mitigating gradient flow pathologies in physics-informed neural networks},
  author={Wang, Sifan and Teng, Yujun and Perdikaris, Paris},
  journal={SIAM Journal on Scientific Computing},
  volume={43},
  number={5},
  pages={A3055--A3081},
  year={2021},
  publisher={SIAM}
}

@article{gu2023nysnewton,
  author  = {Shaobo Gu and Xiao Wang and Ziyan Wang and Jinchao Xu},
  title   = {NysNewton: Fast Neural Network Training with Nyström Approximation},
  journal = {arXiv preprint},
  year    = {2023},
  eprint  = {2307.10697},
  archivePrefix = {arXiv},
  primaryClass = {cs.LG}
}

@inproceedings{liu2024finer,
  title={Finer: Flexible spectral-bias tuning in implicit neural representation by variable-periodic activation functions},
  author={Liu, Zhen and Zhu, Hao and Zhang, Qi and Fu, Jingde and Deng, Weibing and Ma, Zhan and Guo, Yanwen and Cao, Xun},
  booktitle={Proceedings of the IEEE/CVF Conference on Computer Vision and Pattern Recognition},
  pages={2713--2722},
  year={2024}
}

@book{nocedal2006numerical,
  title={Numerical optimization},
  author={Nocedal, Jorge and Wright, Stephen J},
  year={2006},
  publisher={Springer}
}

@article{wang2024respecting,
  title={Respecting causality for training physics-informed neural networks},
  author={Wang, Sifan and Sankaran, Shyam and Perdikaris, Paris},
  journal={Computer Methods in Applied Mechanics and Engineering},
  volume={421},
  pages={116762},
  year={2024},
  publisher={Elsevier},
  doi={10.1016/j.cma.2024.116762},
  note={Preprint originally released in 2022 as arXiv:2203.07404}
}

@article{penwarden2023unified,
  title={A unified scalable framework for causal sweeping strategies in physics-informed neural networks},
  author={Penwarden, Michael and Zhe, Shandian and Narayan, Akil and Kirby, Robert M.},
  journal={Journal of Computational Physics},
  volume={493},
  pages={112461},
  year={2023},
  publisher={Elsevier},
  doi={10.1016/j.jcp.2023.112461}
}

@article{wang2025discovery,
  title={Discovery of unstable singularities},
  author={Wang, Yongji and Bennani, Mehdi and Martens, James and Racani{\`e}re, S{\'e}bastien and Blackwell, Sam and Matthews, Alex and Nikolov, Stanislav and Cao-Labora, Gonzalo and Park, Daniel S and Arjovsky, Martin and others},
  journal={arXiv preprint arXiv:2509.14185},
  year={2025}
}

@article{sun2020surrogate,
  title={Surrogate modeling for fluid flows based on physics-constrained deep learning without simulation data},
  author={Sun, Luning and Gao, Han and Pan, Shaowu and Wang, Jian-Xun},
  journal={Computer Methods in Applied Mechanics and Engineering},
  volume={361},
  pages={112732},
  year={2020},
  publisher={Elsevier}
}

@article{wang2025physics,
  title={Physics-guided deep learning for dynamical systems: A survey},
  author={Wang, Rui and Yu, Rose},
  journal={ACM Computing Surveys},
  volume={58},
  number={5},
  pages={1--31},
  year={2025},
  publisher={ACM New York, NY}
}

@article{cai2021physics,
  title={Physics-informed neural networks for heat transfer problems},
  author={Cai, Shengze and Wang, Zhicheng and Wang, Sifan and Perdikaris, Paris and Karniadakis, George Em},
  journal={Journal of Heat Transfer},
  volume={143},
  number={6},
  pages={060801},
  year={2021},
  publisher={American Society of Mechanical Engineers}
}

@article{du2023state,
  title={State estimation in minimal turbulent channel flow: A comparative study of 4DVar and PINN},
  author={Du, Yifan and Wang, Mengze and Zaki, Tamer A},
  journal={International Journal of Heat and Fluid Flow},
  volume={99},
  pages={109073},
  year={2023},
  publisher={Elsevier}
}

@article{lagaris1998artificial,
  title={Artificial neural networks for solving ordinary and partial differential equations},
  author={Lagaris, Isaac E and Likas, Aristidis and Fotiadis, Dimitrios I},
  journal={IEEE transactions on neural networks},
  volume={9},
  number={5},
  pages={987--1000},
  year={1998},
  publisher={IEEE}
}

@article{wang2024piratenets,
  title={Piratenets: Physics-informed deep learning with residual adaptive networks},
  author={Wang, Sifan and Li, Bowen and Chen, Yuhan and Perdikaris, Paris},
  journal={Journal of Machine Learning Research},
  volume={25},
  number={402},
  pages={1--51},
  year={2024}
}

@article{penwarden2023causal,
  title   = {A unified scalable framework for causal sweeping strategies for physics-informed neural networks and their temporal decompositions},
  author  = {Penwarden, Michael and Jagtap, Ameya D. and Zhe, Shandian and Karniadakis, George Em and Kirby, Robert M.},
  journal = {Journal of Computational Physics},
  volume  = {481},
  pages   = {112001},
  year    = {2023},
  publisher = {Elsevier}
}

@article{roy2024adaptive,
  title={Adaptive interface-PINNs (AdaI-PINNs): An efficient physics-informed neural networks framework for interface problems},
  author={Roy, Sumanta and Annavarapu, Chandrasekhar and Roy, Pratanu and Sarma, Antareep Kumar},
  journal={arXiv preprint arXiv:2406.04626},
  year={2024}
}

@article{berrone2023enforcing,
  title={Enforcing Dirichlet boundary conditions in physics-informed neural networks and variational physics-informed neural networks},
  author={Berrone, Stefano and Canuto, Claudio and Pintore, Moreno and Sukumar, Natarajan},
  journal={Heliyon},
  volume={9},
  number={8},
  year={2023},
  publisher={Elsevier}
}

@article{lu2021physics,
  title={Physics-informed neural networks with hard constraints for inverse design},
  author={Lu, Lu and Pestourie, Raphael and Yao, Wenjie and Wang, Zhicheng and Verdugo, Francesc and Johnson, Steven G},
  journal={SIAM Journal on Scientific Computing},
  volume={43},
  number={6},
  pages={B1105--B1132},
  year={2021},
  publisher={SIAM}
}

@article{mcclenny2023self,
  title={Self-adaptive physics-informed neural networks},
  author={McClenny, Levi D and Braga-Neto, Ulisses M},
  journal={Journal of Computational Physics},
  volume={474},
  pages={111722},
  year={2023},
  publisher={Elsevier}
}

@inproceedings{chen2018gradnorm,
  title={Gradnorm: Gradient normalization for adaptive loss balancing in deep multitask networks},
  author={Chen, Zhao and Badrinarayanan, Vijay and Lee, Chen-Yu and Rabinovich, Andrew},
  booktitle={International conference on machine learning},
  pages={794--803},
  year={2018},
  organization={PMLR}
}

@article{mao2023physics,
  title={Physics-informed neural networks with residual/gradient-based adaptive sampling methods for solving partial differential equations with sharp solutions},
  author={Mao, Zhiping and Meng, Xuhui},
  journal={Applied Mathematics and Mechanics},
  volume={44},
  number={7},
  pages={1069--1084},
  year={2023},
  publisher={Springer}
}

@article{bonfanti2024generalization,
  title={On the generalization of pinns outside the training domain and the hyperparameters influencing it},
  author={Bonfanti, Andrea and Santana, Roberto and Ellero, Marco and Gholami, Babak},
  journal={Neural Computing and Applications},
  volume={36},
  number={36},
  pages={22677--22696},
  year={2024},
  publisher={Springer}
}

@article{toscano2025pinns,
  title={From pinns to pikans: Recent advances in physics-informed machine learning},
  author={Toscano, Juan Diego and Oommen, Vivek and Varghese, Alan John and Zou, Zongren and Ahmadi Daryakenari, Nazanin and Wu, Chenxi and Karniadakis, George Em},
  journal={Machine Learning for Computational Science and Engineering},
  volume={1},
  number={1},
  pages={1--43},
  year={2025},
  publisher={Springer}
}

@article{norambuena2024physics,
  title={Physics-informed neural networks for quantum control},
  author={Norambuena, Ariel and Mattheakis, Marios and Gonz{\'a}lez, Francisco J and Coto, Ra{\'u}l},
  journal={Physical Review Letters},
  volume={132},
  number={1},
  pages={010801},
  year={2024},
  publisher={APS}
}

@article{trahan2024quantum,
  title={Quantum physics-informed neural networks},
  author={Trahan, Corey and Loveland, Mark and Dent, Samuel},
  journal={Entropy},
  volume={26},
  number={8},
  pages={649},
  year={2024},
  publisher={MDPI}
}

@article{mitarai2018quantum,
  title={Quantum circuit learning},
  author={Mitarai, Kosuke and Negoro, Makoto and Kitagawa, Masahiro and Fujii, Keisuke},
  journal={Physical Review A},
  volume={98},
  number={3},
  pages={032309},
  year={2018},
  publisher={APS}
}

@article{schuld2019evaluating,
  title={Evaluating analytic gradients on quantum hardware},
  author={Schuld, Maria and Bergholm, Ville and Gogolin, Christian and Izaac, Josh and Killoran, Nathan},
  journal={Physical Review A},
  volume={99},
  number={3},
  pages={032331},
  year={2019},
  publisher={APS}
}

@article{suresh1997accurate,
  title={Accurate monotonicity-preserving schemes with Runge--Kutta time stepping},
  author={Suresh, Ambady and Huynh, Hung T},
  journal={Journal of Computational Physics},
  volume={136},
  number={1},
  pages={83--99},
  year={1997},
  publisher={Elsevier}
}

@article{schulz1993classification,
  title={Classification of the Riemann problem for two-dimensional gas dynamics},
  author={Schulz-Rinne, Carsten W},
  journal={SIAM journal on mathematical analysis},
  volume={24},
  number={1},
  pages={76--88},
  year={1993},
  publisher={SIAM}
}

@article{kurganov2001semidiscrete,
  title={Semidiscrete central-upwind schemes for hyperbolic conservation laws and Hamilton--Jacobi equations},
  author={Kurganov, Alexander and Noelle, Sebastian and Petrova, Guergana},
  journal={SIAM Journal on Scientific Computing},
  volume={23},
  number={3},
  pages={707--740},
  year={2001},
  publisher={SIAM}
}

@misc{torq,
    key = {ZZZ},
    note = {The quantum simulation library we created, TorQ - Tensor Operations for Research of Quantum systems, is available at \url{https://github.com/zivchen9993/TorQ.git}},
}

@article{lloyd2013quantum,
  title={Quantum algorithms for supervised and unsupervised machine learning},
  author={Lloyd, Seth and Mohseni, Masoud and Rebentrost, Patrick},
  journal={arXiv preprint arXiv:1307.0411},
  year={2013}
}

@book{nielsen2010quantum,
  title={Quantum computation and quantum information},
  author={Nielsen, Michael A and Chuang, Isaac L},
  year={2010},
  publisher={Cambridge university press}
}

@article{james2001measurement,
  title={Measurement of qubits},
  author={James, Daniel FV and Kwiat, Paul G and Munro, William J and White, Andrew G},
  journal={Physical Review A},
  volume={64},
  number={5},
  pages={052312},
  year={2001},
  publisher={APS}
}

@article{micikevicius2017mixed,
  title={Mixed precision training},
  author={Micikevicius, Paulius and Narang, Sharan and Alben, Jonah and Diamos, Gregory and Elsen, Erich and Garcia, David and Ginsburg, Boris and Houston, Michael and Kuchaiev, Oleksii and Venkatesh, Ganesh and others},
  journal={arXiv preprint arXiv:1710.03740},
  year={2017}
}

@article{xue2022novel,
  title={A novel automatic mixed precision approach for physics informed training},
  author={Xue, Jinze and Subramaniam, Akshay and Hoemmen, Mark},
  journal={Advances in Neural Information Processing Systems},
  year={2022}
}

@article{shoeybi2019megatron,
  title={Megatron-lm: Training multi-billion parameter language models using model parallelism},
  author={Shoeybi, Mohammad and Patwary, Mostofa and Puri, Raul and LeGresley, Patrick and Casper, Jared and Catanzaro, Bryan},
  journal={arXiv preprint arXiv:1909.08053},
  year={2019}
}

@inproceedings{narayanan2021efficient,
  title={Efficient large-scale language model training on gpu clusters using megatron-lm},
  author={Narayanan, Deepak and Shoeybi, Mohammad and Casper, Jared and LeGresley, Patrick and Patwary, Mostofa and Korthikanti, Vijay and Vainbrand, Dmitri and Kashinkunti, Prethvi and Bernauer, Julie and Catanzaro, Bryan and others},
  booktitle={Proceedings of the international conference for high performance computing, networking, storage and analysis},
  pages={1--15},
  year={2021}
}

@article{huang2019gpipe,
  title={Gpipe: Efficient training of giant neural networks using pipeline parallelism},
  author={Huang, Yanping and Cheng, Youlong and Bapna, Ankur and Firat, Orhan and Chen, Dehao and Chen, Mia and Lee, HyoukJoong and Ngiam, Jiquan and Le, Quoc V and Wu, Yonghui and others},
  journal={Advances in neural information processing systems},
  volume={32},
  year={2019}
}

@article{fedus2022switch,
  title={Switch transformers: Scaling to trillion parameter models with simple and efficient sparsity},
  author={Fedus, William and Zoph, Barret and Shazeer, Noam},
  journal={Journal of Machine Learning Research},
  volume={23},
  number={120},
  pages={1--39},
  year={2022}
}

@article{lepikhin2020gshard,
  title={Gshard: Scaling giant models with conditional computation and automatic sharding},
  author={Lepikhin, Dmitry and Lee, HyoukJoong and Xu, Yuanzhong and Chen, Dehao and Firat, Orhan and Huang, Yanping and Krikun, Maxim and Shazeer, Noam and Chen, Zhifeng},
  journal={arXiv preprint arXiv:2006.16668},
  year={2020}
}

@incollection{zaki2025data,
  title={Data assimilation and flow estimation},
  author={Zaki, Tamer A and Wang, Mengze},
  booktitle={Data Driven Analysis and Modeling of Turbulent Flows},
  pages={129--181},
  year={2025},
  publisher={Elsevier}
}

@article{du2021evolutional,
  title={Evolutional deep neural network},
  author={Du, Yifan and Zaki, Tamer A},
  journal={Physical Review E},
  volume={104},
  number={4},
  pages={045303},
  year={2021},
  publisher={APS}
}

@article{paszke2019pytorch,
  title={Pytorch: An imperative style, high-performance deep learning library},
  author={Paszke, Adam and Gross, Sam and Massa, Francisco and Lerer, Adam and Bradbury, James and Chanan, Gregory and Killeen, Trevor and Lin, Zeming and Gimelshein, Natalia and Antiga, Luca and others},
  journal={Advances in neural information processing systems},
  volume={32},
  year={2019}
}

@inproceedings{hennigh2021nvidia,
  title={NVIDIA SimNet™: An AI-accelerated multi-physics simulation framework},
  author={Hennigh, Oliver and Narasimhan, Susheela and Nabian, Mohammad Amin and Subramaniam, Akshay and Tangsali, Kaustubh and Fang, Zhiwei and Rietmann, Max and Byeon, Wonmin and Choudhry, Sanjay},
  booktitle={International conference on computational science},
  pages={447--461},
  year={2021},
  organization={Springer}
}

@article{haghighat2021sciann,
  title={SciANN: A Keras/TensorFlow wrapper for scientific computations and physics-informed deep learning using artificial neural networks},
  author={Haghighat, Ehsan and Juanes, Ruben},
  journal={Computer Methods in Applied Mechanics and Engineering},
  volume={373},
  pages={113552},
  year={2021},
  publisher={Elsevier}
}

@article{hao2024pinnacle,
  title={Pinnacle: A comprehensive benchmark of physics-informed neural networks for solving pdes},
  author={Hao, Zhongkai and Yao, Jiachen and Su, Chang and Su, Hang and Wang, Ziao and Lu, Fanzhi and Xia, Zeyu and Zhang, Yichi and Liu, Songming and Lu, Lu and others},
  journal={Advances in Neural Information Processing Systems},
  volume={37},
  pages={76721--76774},
  year={2024}
}

\end{document}